\documentclass[11pt]{article}

\usepackage[margin=1in]{geometry}
\usepackage{amsmath}
\usepackage{amssymb}
\usepackage{amsthm}
\usepackage{thmtools}
\usepackage{mathtools}
\usepackage{natbib}
\usepackage[colorlinks=true,citecolor=blue,linkcolor=blue,urlcolor=blue]{hyperref}
\usepackage[compress]{cleveref}
\usepackage{booktabs}

\crefname{assumption}{Assumption}{Assumptions}
\crefname{convention}{Convention}{Conventions}
\Crefname{assumption}{Assumption}{Assumptions}

\theoremstyle{plain}
\newtheorem{theorem}{Theorem}[section]
\newtheorem{lemma}[theorem]{Lemma}
\newtheorem{proposition}[theorem]{Proposition}
\newtheorem{corollary}[theorem]{Corollary}

\theoremstyle{definition}
\newtheorem{definition}[theorem]{Definition}
\newtheorem{example}[theorem]{Example}
\newtheorem{assumption}{Assumption}
\newtheorem{convention}{Convention}

\theoremstyle{remark}
\newtheorem{remark}[theorem]{Remark}
\newenvironment{sketch}[1][Proof sketch]{\begin{proof}[#1]\itshape}{\end{proof}}

\title{Differential Privacy of Gradient Descent on Perturbed Objectives}
\author{
  Austin Watkins\\
  Johns Hopkins University\\
  Baltimore, MD 21218\\
  \texttt{awatki29@jhu.edu}
  \and
  Raman Arora\\
  Johns Hopkins University\\
  Baltimore, MD 21218\\
  \texttt{arora@cs.jhu.edu}
}
\date{}

\begin{document}

\maketitle

\begin{abstract}
Objective perturbation adds a random linear term to a regularized empirical risk and releases the exact perturbed minimizer. We study the finite computation obtained by releasing the $N$-th iterate of deterministic gradient descent on $w\mapsto F(w;S)+\langle z,w\rangle$, where $z\sim\mathcal N(0,\sigma^2I_d)$ is drawn once before optimization. For strongly convex and smooth objectives with Lipschitz Hessian, we prove an explicit condition under which the map $z\mapsto w_N$ is a $C^1$-diffeomorphism on the bounded domains used in the privacy argument, with a quantitative lower bound on the smallest singular value of its Jacobian. This permits a direct change-of-variables analysis of the finite iterate. For generalized linear models, the resulting privacy-profile bound has no explicit ambient-dimension factor once the iteration condition holds, and its finite-iteration correction decreases geometrically. By letting the free truncation parameter grow slowly with $N$, we recover the corresponding exact-minimizer certificate in the limit. We also bound the expected excess empirical risk by $d\sigma^2/(2\mu)$ plus a geometrically decreasing optimization term, and transfer the result to population risk without an additional multiplicative condition-number factor in the leading statistical terms.
\end{abstract}

\section{Introduction}\label{sec:intro}

Objective perturbation is one of the earliest mechanisms for differentially private empirical risk minimization (ERM): draw a random vector $z$, add the linear term $\langle z,w\rangle$ to a regularized empirical objective, and release its minimizer \citep{chaudhuri2008privacy,chaudhuri2011differentially}. Its appeal is structural as well as algorithmic. Randomness is injected once, before optimization, and for generalized linear models (GLMs) the privacy analysis can avoid an explicit dependence on the ambient dimension. Modern analyses have sharpened the privacy accounting and computational treatment of the exact objective-perturbation mechanism \citep{redberg2023improving}.

The classical privacy proof, however, is tied to the exact perturbed minimizer. If
\[
    w^*(z) \in \arg\min_w \bigl\{F(w;S)+\langle z,w\rangle\bigr\},
\]
then the first-order condition $\nabla F(w^*(z);S)+z=0$ recovers the perturbation from the released point and enables a change-of-variables argument. A numerical optimizer instead produces a finite iterate. Existing approaches to this gap typically certify an approximate solution by adding a second randomization at release, as in Approximate Minima Perturbation (AMP) \citep{iyengar2019towards} and Approximate Objective Perturbation \citep{bassily2019private}, or inject noise throughout the optimization and analyze the resulting noisy iteration \citep{noisy-sgd,feldman2018privacy,chourasia2021differential,altschuler2022privacy}. Quadratic objective perturbation has recently provided a different perturbation mechanism, including guarantees for inexact solves under its own assumptions \citep{cortild2026quadratic}.

Can we instead certify a useful finite gradient-descent iterate using only the original linear objective perturbation, without per-step noise or an additional output perturbation? We seek a privacy guarantee that approaches the corresponding exact-minimizer guarantee as optimization converges. We answer this question for gradient descent on strongly convex, smooth GLM objectives with Lipschitz Hessian. Draw $z\sim\mathcal N(0,\sigma^2I_d)$ once, run
\[
    w_{k+1}=w_k-\eta\bigl(\nabla F(w_k;S)+z\bigr),
\]
from the fixed initialization $w_0=0$, and release $w_N$.

To analyze the privacy of this release, we study the finite-time map
\[
    \mathcal M_S:z\longmapsto w_N
\]
instead of the exact-minimizer map $z\mapsto w^*(z)$. Differentiating with respect to the perturbation gives sums of ordered products of curvature matrices evaluated along the optimization trajectory. These matrices need not commute, so controlling the Jacobian requires accounting for their ordered products. We compare the Jacobian with a frozen-Hessian recursion and obtain an explicit condition under which $\mathcal M_S$ is a $C^1$-diffeomorphism on the bounded convex domains used in the privacy proof, with a uniform lower bound on the smallest singular value of its Jacobian. We develop this comparison in \Cref{subsec:results-diffeo}.

Once this inverse is available, the classical change-of-variables argument can be applied at finite time. At a shared output, the privacy loss separates into a Gaussian density ratio between the two perturbation preimages and a ratio of Jacobian determinants. For GLMs, the gradient difference between neighboring datasets lies in the span of the two differing feature vectors, and Hessian perturbations are sums of controlled rank-one matrices. These two properties remove an explicit ambient-dimension factor from the privacy bound. The finite-iteration corrections vanish as $N\to\infty$; we may let the free truncation radius grow slowly with $N$ while these corrections still vanish, giving the corresponding exact-minimizer certificate in the limit (\Cref{rem:oracle-recovery,rem:asymptotic}).

\paragraph{Contributions.}
\begin{itemize}
    \item \textbf{A composition-free privacy certificate for the released iterate} (\Cref{thm:dp}). We derive an explicit upper bound on the privacy profile of $S\mapsto \mathcal M_S(Z)$. For every $\epsilon\ge \epsilon_0$, the release is $(\epsilon,\delta(\epsilon))$-differentially private, with
    \[
        \delta(\epsilon) \leq e^{-u^2/2} + \sqrt{\tfrac{2}{\pi}}\,\frac{G}{\sigma}\, e^{-\sigma^2(\epsilon - \epsilon_0)^2/(2G^2)},
        \qquad
        \epsilon_0 = \frac{G^2}{2\sigma^2} + \log\!\Bigl( 1 + \frac{\Delta_H}{2n\mu} \Bigr) + c_N,
    \]
    where $u$ controls the truncation probability, $G$ is the neighboring-dataset gradient sensitivity, $\Delta_H$ is the Hessian scale defined in \Cref{subsec:notation}, and $c_N\ge0$ contains finite-iteration corrections that decay geometrically. Once the explicit threshold \eqref{eq:N-dagger} is satisfied, the displayed privacy budget has no explicit dependence on the ambient dimension. When we solve the sufficient condition for $N$, $d$ appears only through $d_\star=\min\{d,n+1\}$ inside a logarithm.

    \item \textbf{Quantitative invertibility of the finite-time perturbation map} (\Cref{cor:diffeo}). We compare the Jacobian recursion for $\mathcal M_S$ with a frozen-Hessian recursion that can be solved in closed form. From this comparison we obtain an explicit negative-definiteness margin for the symmetric part of the true Jacobian. We then use that margin to prove global injectivity and a uniform lower bound on the smallest singular value over the truncation domains used in the privacy proof.

    \item \textbf{Utility with an explicit finite-optimization correction} (\Cref{thm:unregularized,thm:empirical-risk}). The expected excess empirical risk satisfies the standard objective-perturbation term $d\sigma^2/(2\mu)$ plus an optimization correction $\mathcal E_{\mathrm{opt}}(N)$ that decays geometrically. For GLMs, a uniform-convergence argument transfers this to population risk without an additional multiplicative condition-number factor in the leading statistical terms. In particular, for every comparator $\bar w$, with probability at least $1-\delta_S$ over the sample,
    \[
        \mathbb{E}_z\bigl[ \mathcal{R}_0(w_N) \bigr] - \mathcal{R}_0(\bar w)
        \leq \| \bar w \| \Bigl( \sqrt{2}\,\sigma\sqrt{d}
        + \bigl( 32 + 4\sqrt{2}\,\sqrt{\log(1/\delta_S)} \bigr) \frac{L_\ell}{\sqrt{n}} \Bigr)
        + 2\,\mathcal{E}_{\mathrm{opt}}(N).
    \]
    The high-probability statement here is over the sample; the displayed risk is averaged over the perturbation.
\end{itemize}

Our result gives a finite-trajectory change-of-variables certificate for ordinary deterministic gradient descent after a single objective perturbation. This complements analyses of the exact minimizer \citep{chaudhuri2011differentially,kifer2012private,redberg2023improving}, approximate solves followed by output noise \citep{iyengar2019towards,bassily2019private}, quadratic objective perturbation with inexact solves \citep{cortild2026quadratic}, and noisy iterative algorithms \citep{noisy-sgd,feldman2018privacy,chourasia2021differential,altschuler2022privacy}.

\Cref{sec:setup} presents the notation, assumptions, and mechanism; \Cref{sec:results} states the main results with proof sketches; and \Cref{sec:discussion} compares the guarantee with objective perturbation and composition-free analyses of noisy iterations and discusses computational cost. Complete proofs are given in \Cref{sec:diffeo,sec:glm,sec:privacy,sec:utility}.

\section{Related work}\label{sec:related}

We relate our results to four lines of work: objective perturbation, private optimization by gradient perturbation, composition-free privacy analyses of noisy iterations, and invertibility results for gradient methods.

\paragraph{Objective perturbation.}\label{subsec:rw-objpert}

\citet{chaudhuri2008privacy} introduced objective perturbation for logistic regression, and \citet{chaudhuri2011differentially} developed the regularized-ERM analysis for GLMs. The mechanism adds a random linear term $\langle z,w\rangle$ to the empirical objective and releases the exact perturbed minimizer.
\citet{kifer2012private} extended the mechanism to approximate differential privacy by using Gaussian noise and nondifferentiable regularizers via a limit argument and smooth approximation.
Regarding the dimensional utility cost, \citet{jain2014near} showed that unconstrained GLMs with strongly convex regularization admit excess risk with no explicit dependence on $d$, controlled instead by the norms of the risk minimizer and the data.

These exact-minimizer guarantees leave a numerical-certification question when optimization is stopped at finite accuracy. Approximate Minima Perturbation (AMP), introduced by \citet{iyengar2019towards}, addresses this by obtaining an approximate minimizer, measured by its gradient norm, on the perturbed (and additionally regularized) objective and then adding Gaussian output noise to certify the released point.
Similar to the approximation scheme above, \citet{bassily2019private} show that
stochastic variance-reduced gradient descent (SVRG) on the perturbed objective followed by the Gaussian mechanism attains the optimal population loss in $O(n \log n)$ gradient evaluations.
If the domain is discrete, \citet{neel2020oracle} show that only boundedness of the loss is sufficient for privacy in the constrained setting via a modified algorithm that normalizes candidates onto a sphere in one higher dimension. They also provide a continuous-domain algorithm requiring boundedness and Lipschitz continuity, with a $d^{5/4}$ factor in its stated utility bound.

\citet{redberg2023improving} provide the modern privacy accounting baseline for objective perturbation, developing tighter analyses through privacy profiles and R\'enyi differential privacy and treating both exact objective perturbation and AMP. They emphasize that the perturbation and released minimizer are dependent random variables and show that, without GLM structure, objective perturbation can incur dimension-dependent privacy terms. \citet[v2]{cortild2026quadratic} study random quadratic objective perturbation under smoothness and Hessian-rank assumptions, together with a bounded-diameter condition on the per-sample gradients at a publicly known, data-independent point. Interpolation is one special case of this condition. For inexact solves, they add Gaussian output noise calibrated to the optimization error.

\paragraph{Gradient perturbation and private optimization.}\label{subsec:rw-gradient}

\citet{noisy-sgd} established optimal rates for constrained differentially private empirical risk minimization (DP-ERM) via noisy stochastic gradient descent (DP-SGD) run for $T = n^2$ steps and accounted per step through amplification by sampling and strong composition: $\widetilde{\Theta}(\sqrt{d}/(n\epsilon))$ excess empirical risk for Lipschitz convex losses under $(\epsilon,\delta)$-differential privacy, and $\widetilde{\Theta}(d/(n^2\epsilon^2))$ under strong convexity.\footnote{Here and in \Cref{subsec:glm-role}, the rates are stated after accounting for the normalizing factor of $n$ that gives the average of the per-sample losses, as in \eqref{eq:empirical-risk}, instead of \citeauthor{noisy-sgd}'s sum.} 
In follow-up work, \citet{bassily2019private} characterize the optimal private stochastic convex optimization rate, up to logarithmic factors, for convex Lipschitz losses; in the standard Euclidean setting the rate has the form $O(1/\sqrt{n} + \sqrt{d}/(\epsilon n))$.
\citet{song2021evading} sharpened the picture for gradient perturbation with rank-based rates and the first lower bound for unconstrained ERM.

\paragraph{Composition-free privacy for iterative methods.}\label{subsec:rw-compfree}

A separate line of work keeps the per-step noise of DP-SGD but only releases the final iterate.
Privacy amplification by iteration \citep{feldman2018privacy} provides utility and bounded R\'enyi differential privacy for noisy SGD on smooth convex Lipschitz losses with a random stopping time.
\citet{altschuler2022privacy} gave a tight analysis of the same method, showing that the privacy loss need not grow with the iteration count for projected noisy SGD on smooth convex Lipschitz losses over a bounded domain.
\citet{bok2024shifted} recast amplification by iteration in $f$-differential privacy, which yields an exact characterization for strongly convex losses and covers constrained and unconstrained, full-batch, cyclic, and stochastic variants alike.
Under a quantified logarithmic Sobolev inequality, a tight last-iterate analysis was provided by \citet{chourasia2021differential}, which shows the R\'enyi privacy loss saturates exponentially in the iteration count for smooth and strongly convex loss functions.
\citet{ye2022hidden} carry the hidden-state analysis to mini-batch gradient descent, in shuffle-and-partition and sample-without-replacement schemes, folding in amplification by subsampling, with bounds that again converge exponentially fast.
\citet{chien2025convergent} significantly extend the generality of hidden-state noisy-SGD by showing that over a bounded domain, a H\"older-continuous gradient alone suffices for a convergent R\'enyi bound, with a strictly better bound in the smooth strongly convex case.
\citet{brown2024private} show that for linear regression an explicit law of the iterates can be obtained; privacy itself is still accounted per step.

\paragraph{Diffeomorphisms in optimization.}\label{subsec:rw-diffeo}

\citet{daskalakis2018limit} prove that the gradient descent--ascent update is a local diffeomorphism whenever the step size satisfies $\alpha < 1/L$, and likewise its optimistic variant; this is what admits the Center-Stable Manifold Theorem, from which they conclude that the dynamics avoid unstable critical points from almost every initialization.

For a fixed objective, \citet{bolte2025bilevel} show that the $k$-step gradient-descent map from the initialization $w_0$ to the final iterate is a global diffeomorphism when the step size satisfies $\alpha<1/L$, under their stated regularity assumptions. Our change-of-variables argument uses a different map: we fix the initialization and vary the perturbation $z$. Differentiation with respect to $w_0$ gives a product of the single-step factors $I-\alpha\nabla^2g$, whereas differentiation with respect to $z$ gives the affine recursion \eqref{eq:jacobian-recursion}, which adds $-\eta I$ at every step. We therefore need quantitative invertibility for the map $z\mapsto w_N$ rather than for the initialization map.

\section{Setup}\label{sec:setup}

\subsection{Notation and assumptions}\label{subsec:notation}

We equip $\mathbb{R}^d$ with the Euclidean inner product $\langle \cdot, \cdot \rangle$ and the induced norm $\| \cdot \|$. For a matrix $A$, the symbol $\| A \|$ denotes the spectral (operator) norm, $\| A \|_*$ the nuclear (trace) norm, and $\sigma_{\min}(A)$ the smallest singular value. We write $A^{\mathrm{sym}} := \tfrac{1}{2}(A + A^\top)$ for the symmetric part of $A$; it satisfies $\langle A v, v \rangle = \langle A^{\mathrm{sym}} v, v \rangle$ for every $v \in \mathbb{R}^d$ and $\| A^{\mathrm{sym}} \| \leq \| A \|$. For symmetric matrices we write $A \preceq B$ to mean that $B - A$ is positive semidefinite, and $I$ denotes the identity. 

We work with regularized generalized linear models. Each sample is a pair $(x_i, y_i)$ consisting of a feature vector $x_i \in \mathbb{R}^d$ and a label $y_i$; a dataset collects $n$ samples into a tuple $S = ((x_1,y_1), \ldots, (x_n,y_n))$, and two datasets are \emph{neighboring} if one is obtained from the other by replacing at most one sample. We write neighboring datasets as $S$ and $S' = ((x_1,y_1), \ldots, (x_{n-1},y_{n-1}),$ $(x_n', y_n'))$, placing the differing sample last, abbreviate $F_S(\cdot) := F(\cdot\,; S)$, and take the regularized empirical risk to be
\begin{equation}\label{eq:empirical-risk}
F_S(w) = \frac{1}{n}\sum_{i=1}^{n} \ell\bigl(\langle x_i, w \rangle, y_i\bigr) + \frac{\mu}{2}\,\| w \|^2,
\end{equation}
where the loss $\ell$ is twice continuously differentiable in its first argument. Let 
\begin{equation}\label{eq:data-span}
    \mathcal{C} := \mathrm{span}\{ x_1, \ldots, x_{n-1}, x_n, x_n' \}, \qquad \dim \mathcal{C} \leq d_\star := \min\{ d,\, n+1 \},
\end{equation}
carry every feature either dataset holds; it is the same subspace for $S$ and $S'$. The guarantee is stated with the upper bound $d_\star$, which is fixed by $d$ and $n$ alone. We impose the following structural conditions.

\begin{assumption}[Bounded features]\label{ass:bounded-features}
$\| x_i \| \leq \zeta$ for every feature vector appearing in $S$ or $S'$.
\end{assumption}

\begin{assumption}[Bounded loss derivatives]\label{ass:bounded-loss}
$| \ell'(t, y) | \leq L_1$ and $0 \leq \ell''(t, y) \leq L_2$ for all $t, y$, so in particular each per-sample loss is convex in its first argument.
\end{assumption}

\begin{assumption}[Lipschitz second derivative of the loss]\label{ass:lipschitz-third}
$| \ell''(s, y) - \ell''(t, y) | \leq L_3\,| s - t |$ for all $s, t, y$.
\end{assumption}

The logistic loss satisfies \Cref{ass:bounded-loss,ass:lipschitz-third}.

\Cref{ass:bounded-features,ass:bounded-loss} make every $F_S$ in \eqref{eq:empirical-risk} $\mu$-strongly convex and $L$-smooth for any $L \geq \mu + L_2\zeta^2$, uniformly in $S$ (\Cref{lem:glm-smoothness}); \Cref{ass:bounded-features,ass:lipschitz-third} make the Hessian $M$-Lipschitz for any $M \geq L_3 \zeta^3$ (\Cref{lem:hessian-lipschitz}); writing $\Delta_H := 2L_2\zeta^2$, \Cref{ass:bounded-features,ass:bounded-loss} further bound the rank-one Hessian term a single sample contributes by $L_2\zeta^2/n = \Delta_H/(2n)$, in operator and nuclear norm alike, the two coinciding on a rank-one matrix. Finally, the first-derivative bound of \Cref{ass:bounded-loss} bounds the per-sample gradient by $L_\ell := L_1 \zeta$, the gradient sensitivity between neighboring datasets by $G := 2L_1 \zeta / n$ (\Cref{sec:glm}), and the gradient at the origin by $\| \nabla F_S(0) \| = \| \tfrac{1}{n}\sum_i \ell'(0, y_i)\, x_i \| \leq L_1 \zeta$, uniformly in $S$.

For the utility results, the samples are $n$ independent draws from a distribution $\mathcal{D}$, with the feature bound of \Cref{ass:bounded-features} holding $\mathcal{D}$-almost surely. The population risk is $\mathcal{R}_\mu(w) := \mathbb{E}_{(x,y) \sim \mathcal{D}}\bigl[ \ell(\langle x, w \rangle, y) \bigr] + \tfrac{\mu}{2}\| w \|^2$, and it inherits $\mu$-strong convexity and $L$-smoothness by the same computation; $\mathcal{R}_0$ is the unregularized population risk.  We write $w^* := \operatorname*{argmin}_w \mathcal{R}_\mu(w)$ for its minimizer and $w_S^* := \operatorname*{argmin}_w F_S(w)$ for the empirical minimizer.

In the main text we work with regularized GLMs. For the diffeomorphism argument itself, we use only strong convexity, smoothness, and a Lipschitz Hessian (\Cref{subsec:general-setting}). GLM structure enters when we obtain the dimension-free privacy bound and when we transfer the utility guarantee to population risk (\Cref{subsec:glm-role}).

A randomized mechanism $\mathcal{M}$, mapping a dataset $S$ to a random output $\mathcal{M}_S(Z)$, is \emph{$(\epsilon, \delta)$-differentially private} if $\Pr[\mathcal{M}_S(Z) \in E] \leq e^{\epsilon}\,\Pr[\mathcal{M}_{S'}(Z) \in E] + \delta$ for every ordered pair of neighboring datasets $(S, S')$ and every Borel $E \subseteq \mathbb{R}^d$. At each $\epsilon \geq 0$ the smallest such $\delta$ is the hockey-stick divergence, maximized over neighboring pairs:
\begin{equation}\label{eq:privacy-profile}
    \delta_{\mathcal{M}}(\epsilon) := \sup_{S, S'}\, \sup_{E}\, \bigl( \Pr[\mathcal{M}_S(Z) \in E] - e^{\epsilon}\,\Pr[\mathcal{M}_{S'}(Z) \in E] \bigr) = \sup_{S, S'} \int \bigl[ p^S - e^{\epsilon}\, p^{S'} \bigr]_+,
\end{equation}
the last expression valid when the output laws have densities $p^S, p^{S'}$. The map $\epsilon \mapsto \delta_{\mathcal{M}}(\epsilon)$ is the \emph{privacy profile} of the mechanism \citep{balle2018privacy}.

\subsection{The mechanism}\label{subsec:mechanism}

Fix a step size $\eta$ with $0 < \eta \leq 1/L$ and $\eta\mu < 1$; as $1/L \leq 1/\mu$, the second condition excludes only the endpoint $\eta = 1/\mu$, and only when $\mu = L$. Given a perturbation $z \in \mathbb{R}^d$, we run gradient descent on the perturbed objective $w \mapsto F(w) + \langle z, w \rangle$ from the origin, $w_0 = 0$:
\begin{equation}\label{eq:gd-recursion}
    w_{k+1} = w_k - \eta\bigl(\nabla F(w_k) + z\bigr).
\end{equation}
For a fixed number of iterations $N \geq 1$, the \emph{mechanism} $\mathcal{M} : \mathbb{R}^d \to \mathbb{R}^d$ is the map $\mathcal{M}(z) := w_N$.
When emphasizing the dataset we write $\mathcal{M}_S$, $\mathcal{M}_{S'}$ for the mechanisms and $w_k^S$, $w_k^{S'}$ for the iterates under $S$ and $S'$. We write $\kappa:=L/\mu$ for the condition number.  For the main privacy certificate, the sufficient count in \eqref{eq:N-dagger} is $2\kappa\log(\cdot)$ at $\eta=1/L$, with a logarithmic argument polynomial in $n$, $d_\star$, and the problem constants.

Since $F \in C^2$, each iterate $w_k$ is continuously differentiable in $z$. Writing $J_k:=\partial w_k/\partial z$ and $A_k:=I-\eta\nabla^2F(w_k)$, differentiation of \eqref{eq:gd-recursion} gives
\begin{equation}\label{eq:jacobian-recursion}
    J_{k+1} = A_k J_k - \eta I, \qquad J_0 = 0,
\end{equation}
so that $\nabla \mathcal{M}(z) = J_N$. The matrix $A_k$ is symmetric because $\nabla^2 F(w_k)$ is, and its eigenvalues lie in $[1 - \eta L, 1 - \eta\mu]$; since $\eta L \leq 1$ that interval sits in $[0, \rho]$ with $\rho := 1 - \eta\mu \in (0,1)$, so $A_k \succeq 0$ and $\| A_k \| \leq \rho$.

\section{Main results}\label{sec:results}

We first state the privacy guarantee, then the quantitative invertibility result used in its proof, and finally the utility bounds for the released iterate.

\subsection{Privacy}\label{subsec:results-privacy}

The classical privacy analysis of objective perturbation is a change-of-variables argument anchored to the exact minimizer $w^*(z) := \operatorname{argmin}_w \bigl[ F(w) + \langle z, w \rangle \bigr]$ of the tilted objective (\Cref{lem:perturbed-minimizer}): the optimality condition $\nabla F(w^*(z)) + z = 0$ inverts the map $z \mapsto w^*(z)$ in closed form, and from this inverse the ratio of output densities under neighboring datasets factors into a noise-density ratio and a ratio of Jacobian determinants \citep{chaudhuri2011differentially,redberg2023improving}. For a finite iterate $w_N$, this exact optimality-condition inverse is unavailable. Once $\mathcal{M} : z \mapsto w_N$ is certified as a diffeomorphism on a suitable domain (\Cref{subsec:results-diffeo}), the output density at $y = \mathcal{M}(z)$ is $p_Z(z)/|\det J_N(z)|$ by the change of variables, and the privacy loss at an output attained under both datasets, with preimages $z_S$ and $z_{S'}$, splits into two terms:
\begin{equation}\label{eq:privacy-split}
    \log \frac{p_Y^S(y)}{p_Y^{S'}(y)} = \underbrace{\log \frac{p_Z(z_S)}{p_Z(z_{S'})}}_{\text{Gaussian ratio}} + \underbrace{\log \frac{|\det J_N^{S'}(z_{S'})|}{|\det J_N^S(z_S)|}}_{\text{Jacobian ratio}},
\end{equation}
where $p_Z$ is the noise density, $p_Y^T$ the output density under dataset $T$ restricted to the domain of invertibility, and $J_N^T$ its Jacobian at the corresponding preimage; the version with the domain bookkeeping explicit is \eqref{eq:privacy-decomp} in \Cref{sec:privacy}.

We apply the change-of-variables argument on two nested balls in the feature span $\mathcal C$ of \eqref{eq:data-span}. By \Cref{lem:data-span}, the $\mathcal C^\perp$ component follows the same affine recursion under both neighboring datasets and cancels from the privacy loss. For a free parameter $u\geq0$, let $\Omega_0=\{z\in\mathcal C:\|z\|\leq r_u\}$ with $r_u=\sigma(\sqrt{d_\star}+u)$. Gaussian concentration gives $\Pr(\|\Pi_{\mathcal C}Z\|>r_u)\leq e^{-u^2/2}$. We also use the larger ball $\Omega'=\{z\in\mathcal C:\|z\|\leq r'\}$ with $r'=r_u+2LG/\mu$. Under \eqref{eq:N-dagger}, every output produced from $\Omega_0$ has the required $\mathcal M_{S'}$-preimage in $\Omega'$; \Cref{sec:dp-guarantee} proves this by path lifting. Increasing $u$ reduces the truncation probability and changes the required iteration count only through the logarithm in \eqref{eq:N-dagger}.

\begin{restatable}[Differential privacy of gradient descent on perturbed objectives]{theorem}{thmdp}\label{thm:dp}
Let \Cref{ass:bounded-features,ass:bounded-loss,ass:lipschitz-third} hold, and let the step size satisfy $0 < \eta \leq 1/L$ and $\eta\mu < 1$ (\Cref{subsec:mechanism}). Fix $u \geq 0$ and set
\[
    r_u := \sigma\bigl(\sqrt{d_\star} + u\bigr), \qquad
    r' := r_u + \frac{2LG}{\mu}, \qquad
    \bar R := \frac{r' + L_1\zeta}{\mu},
\]
together with $B := 1/\mu$ (item~\ref{itm:Jt-norm} of \Cref{lem:frozen-spectrum}). Suppose the iteration count satisfies
\begin{equation}\label{eq:N-dagger}
    N \geq \frac{2}{\eta\mu}\log\max\Bigl\{ \frac{2\mu \bar R\, r'}{G\sigma},\ \ \frac{8 L M B \bar R}{e\rho\mu\,\min\{1, MG/L\}},\ \ 4 \Bigr\}.
\end{equation}
Writing
\begin{equation}\label{eq:tau-star}
    \tau^\star := \log\Bigl( 1 + \frac{\Delta_H}{2n\mu} \Bigr) + \min\{1, MG/L\}\,\frac{e\eta\mu}{2}\,N\rho^{N/2},
\end{equation}
and setting
\begin{equation}\label{eq:eps-zero}
    \epsilon_0 := \frac{G}{\sigma}\rho^{N/2} + \frac{G^2}{2\sigma^2}\bigl( 1 + \rho^{N/2} \bigr)^2 + \tau^\star,
\end{equation}
the mechanism $S \mapsto \mathcal{M}_S(Z)$ with $Z \sim \mathcal{N}(0, \sigma^2 I_d)$, $\sigma > 0$, is $(\epsilon, \delta(\epsilon))$-differentially private for every $\epsilon \geq \epsilon_0$, where, for a standard Gaussian $W$, we define the integrated tail $\mathcal{Q}_1(t):=\int_t^\infty \Pr(W>s)\,ds$,
\begin{equation}\label{eq:delta-profile}
    \delta(\epsilon) = e^{-u^2/2} + \frac{2G}{\sigma}\,\mathcal{Q}_1\Bigl( \frac{\sigma}{G}\,(\epsilon - \epsilon_0) \Bigr),
\end{equation}
an explicit upper bound on the privacy profile \eqref{eq:privacy-profile}. Once \eqref{eq:N-dagger} holds, the displayed privacy budget has no explicit dependence on the ambient dimension. In the sufficient condition for $N$, the dimension appears only through $d_\star=\min\{d,n+1\}$. 
\end{restatable}

Over the range covered by \eqref{eq:N-dagger}, we obtain $\epsilon_0\downarrow G^2/(2\sigma^2)+\log(1+\Delta_H/(2n\mu))$ as $N\to\infty$. The Gaussian-ratio calculation gives the first two contributions in \eqref{eq:eps-zero}, and the log-determinant calculation gives $\tau^\star$. Since $G=2L_1\zeta/n$ (\Cref{lem:gradient-subspace}), the limiting expression decreases with the sample size for fixed problem parameters. The sufficient iteration count depends on dimension only through $d_\star=\min\{d,n+1\}$ inside a logarithm. We choose the third condition in \eqref{eq:N-dagger} so that $N\rho^{N/2}$ decreases over the stated range. \Cref{rem:asymptotic} gives the exact-minimizer limit of the full certificate, including the truncation contribution to $\delta$.

\begin{remark}[The iteration threshold]
We use the three entries in \eqref{eq:N-dagger} for three estimates: to bound the finite-time inversion residual in the Gaussian-ratio calculation, to make the Jacobian error smaller than the spectral floor of the frozen-Hessian reference, and to restrict attention to the range where $N\rho^{N/2}$ is decreasing. The second estimate also gives the uniform inverse-Jacobian bound used on the truncation domain. In that estimate, we retain the factor $\min\{1,MG/L\}$ to bound the determinant correction by a quantity proportional to $G$ when $MG/L\leq1$, while capping the coefficient at one.
\end{remark}

\begin{sketch}[Proof sketch of \Cref{thm:dp}]
We condition on perturbations in $\Omega_0$ and charge the Gaussian mass outside that ball to $\delta$. For an output produced from $\Omega_0$, we lift the segment from $\mathcal M_{S'}(z_S)$ to the shared output through $\mathcal M_{S'}^{-1}$. The inverse-Jacobian bound controls the speed of this lift by $G(1-\rho^N)/(\mu\alpha_{\Omega'})$, and the choice of $r'$ keeps the lifted path inside $\Omega'$. The change-of-variables formula then separates the privacy loss into the Gaussian ratio and the Jacobian ratio in \eqref{eq:privacy-split}. We integrate the remaining one-dimensional Gaussian projection directly in the hockey-stick divergence (\Cref{lem:hockey-chi}), while the truncation event contributes $e^{-u^2/2}$ to $\delta$.

For the Gaussian ratio, \Cref{lem:one-step} gives
$z=-\nabla F(w_N)+\xi$ with $\|\xi\|\leq\rho^N\|\nabla F(0)+z\|$. Hence the two preimages of a shared output differ by the neighboring-dataset gradient difference plus a residual of order $\rho^N$ (\Cref{cor:preimage}). In a GLM, that gradient difference lies in $\mathrm{span}\{x_n,x_n'\}$, so the leading random contribution depends only on the two scalar projections $\langle Z,x_n\rangle$ and $\langle Z,x_n'\rangle$. We bound the residual contribution in \Cref{thm:gaussian-bound}. For the Jacobian ratio we use the frozen-Hessian determinant comparison summarized in \Cref{subsec:results-diffeo} and proved in \Cref{thm:main}.
\end{sketch}

\begin{remark}[Initialization]\label{rem:initialization}
We state \Cref{thm:dp} for $w_0=0$. The same argument applies to any fixed public initialization after replacing the trajectory-scale constants by their values at that point. 
For a public $w_0$, bounded per-sample gradients give
\[
    \sup_S\|\nabla F_S(w_0)\|\leq L_1 \zeta  +\mu\|w_0\|,
    \qquad
    \bar R_{w_0}:=\frac{r'+L_1 \zeta+\mu\|w_0\|}{\mu}
    =\bar R+\|w_0\|.
\]
We use $\bar R_{w_0}$ in place of $\bar R$ in \eqref{eq:N-dagger}. In the optimization correction \eqref{eq:E-opt}, we use $a_R=\|\nabla F_S(w_0)\|/\mu$ and retain $b_R=1/\mu$.
Dependence on the perturbation is different: a start $w_0=\alpha z$ can destroy privacy for every $\alpha>0$. Take the linear loss $\ell(t, y) = -ty$ with labels $y \in \{\pm 1\}$, which satisfies \Cref{ass:bounded-loss,ass:lipschitz-third} with $L_1 = 1$ and $\ell'' \equiv 0$, so that \eqref{eq:empirical-risk} reads $F_S(w) = \tfrac{\mu}{2}\| w \|^2 + \langle b_S, w \rangle$ with $b_S := -\tfrac{1}{n}\sum_{i=1}^{n} y_i x_i$, and every dataset gives the Hessian $\mu I$. For this instance we may take $L=\mu$. At the step size $\eta = (1 - (1+\alpha\mu)^{-1/N})/\mu$, admissible at every $N$, the perturbation carried by the start and the perturbation carried by the minimizer $w^*(z) = -(b_S + z)/\mu$ cancel exactly in $w_N = a\,w_0 + (1-a)\,w^*(z)$ at $a = 1/(1+\alpha\mu)$: the release is the deterministic $w_N = -\alpha b_S/(1 + \alpha\mu)$, neighboring datasets differ in $b_S$, and two distinct point masses admit no $(\epsilon, \delta)$-differential privacy with $\delta < 1$. For $\alpha \leq 0$ the cancellation has no admissible step size at any $N$.
\end{remark}

\subsection{The map is a diffeomorphism}\label{subsec:results-diffeo}

Fix a bounded convex domain $D$. To prove quantitative invertibility on $D$, we compare the true Jacobian with a frozen-Hessian recursion. Set $H^*:=\nabla^2F(w_N)$ and $A^*:=I-\eta H^*$. Replacing the iterate-dependent curvature matrices by $A^*$ gives a recursion whose solution $\tilde J_N$ is explicit.

Unrolling \eqref{eq:jacobian-recursion} from $J_0=0$ gives
\begin{equation}\label{eq:ordered-products}
    J_N=-\eta S_N,\qquad
    S_N:=\sum_{j=0}^{N-1} A_{N-1}A_{N-2}\cdots A_{N-j},
\end{equation}
where the $j=0$ term is $I$. A lower bound on $S_N^{\rm sym}$ gives both injectivity on a convex domain and a lower bound on the smallest singular value. Products of distinct positive semidefinite curvature factors need not be symmetric, and their symmetric parts can have negative eigenvalues. Since $A_{N-1}\succeq0$ and $\|A_{N-1}\cdots A_{N-j}\|\leq\rho^j$, a direct norm estimate gives
\[
    S_N^{\rm sym}\succeq\left(1-\sum_{j=2}^{N-1}\rho^j\right)I.
\]
This estimate gives a uniform positive floor only when $\rho^2/(1-\rho)<1$. The frozen-Hessian comparison uses the contraction of the trajectory to obtain a positive floor for all sufficiently large $N$ under the standing assumptions.

We quantify the comparison using the frozen-Hessian spectral floor
$\gamma_N=(1-(1-\eta L)^N)/L$ and the Jacobian error. If
$R_D:=\sup_{z\in D}\|\nabla F(0)+z\|/\mu$ and $C_D:=\eta M R_D/\mu$, then \Cref{prop:error-bound} gives
$\|J_N-\tilde J_N\|\leq C_D\tilde\nu_N\rho^{N-1}$ uniformly on $D$, where
$\tilde\nu_N=\sum_{k<N}(1-\rho^{N-k})(1-\rho^k)\leq N$. We therefore obtain the margin
$\alpha_D:=\gamma_N-C_D\tilde\nu_N\rho^{N-1}$. Since $\gamma_N\to1/L$ and $\tilde\nu_N\rho^{N-1}\to0$, this margin is positive for all sufficiently large $N$.

\begin{restatable}[Diffeomorphism onto the image]{corollary}{cordiffeo}\label{cor:diffeo}
Let $F:\mathbb R^d\to\mathbb R$ be twice continuously differentiable, $\mu$-strongly convex, and $L$-smooth, with an $M$-Lipschitz Hessian. With 
step size $0 < \eta \leq 1/L$ and $\eta\mu < 1$ (\Cref{subsec:mechanism}), let $D \subset \mathbb{R}^d$ be a bounded convex \emph{open} set with $\alpha_D > 0$. Then $\mathcal{M}(D)$ is open and $\mathcal{M} : D \to \mathcal{M}(D)$ is a $C^1$-diffeomorphism.
\end{restatable}

\begin{sketch}[Proof sketch of \Cref{cor:diffeo}]
The frozen recursion has the closed form
$\tilde J_N=-(H^*)^{-1}(I-(A^*)^N)$. In an eigenbasis of $H^*$, its eigenvalues are
$-(1-(1-\eta h_i)^N)/h_i$ for $h_i\in[\mu,L]$ (\Cref{lem:closed-form,lem:frozen-spectrum}). Hence $\tilde J_N$ is negative definite for every $N\geq1$, with smallest singular value at least
$\gamma_N=(1-(1-\eta L)^N)/L$. As $N\to\infty$, $w_N\to w^*(z)$ and $\tilde J_N\to-(H^*)^{-1}$, which is the derivative of the exact-minimizer map obtained from its optimality condition.

Using the Lipschitz Hessian and the contraction of gradient descent, we bound the deviation from the frozen-Hessian recursion by $\|J_N-\tilde J_N\|\leq C(z)\tilde\nu_N\rho^{N-1}$ with $\tilde\nu_N\leq\nu_N\leq N$ (\Cref{prop:error-bound}). Since $\tilde J_N$ is symmetric and negative definite, the same estimate gives $J_N^{\mathrm{sym}}\preceq-\alpha_DI$ on $D$ (\Cref{thm:neg-def-J}). Integrating this inequality along a line segment proves global injectivity (\Cref{thm:injective}), while $\|J_Nv\|\geq|\langle J_Nv,v\rangle|$ gives $\sigma_{\min}(\nabla\mathcal M)\geq\alpha_D>0$. Together these facts give the $C^1$-diffeomorphism in \Cref{cor:diffeo}.
\end{sketch}

\begin{sketch}[Proof sketch of \Cref{thm:main}]
Let $P_N^T=-\tilde J_N^T=p_N(H_T^*)$ denote the frozen-Hessian Jacobian up to sign, and write
$-J_N^T=P_N^T(I-\Psi_T)$ with $\Psi_T=(P_N^T)^{-1}\delta_N^T$. Feasibility gives a uniform bound $\|\Psi_T\|<1$, so the determinant separates into a frozen-Hessian factor and a finite-$N$ error factor.

At a shared output, the two frozen reference Hessians are evaluated at the same point. Each has the common contribution from the regularizer and the $n-1$ shared samples, plus one rank-one term from the differing sample. Writing $\phi_N=\log p_N$, \Cref{lem:symbol} shows that $\phi_N'(h)=-g_N(h)/h$ with $0\leq g_N(h)\leq\bar g_N$. The common-anchor rank-one comparison in \Cref{lem:hessian-logdet} then gives
\[
\left|\log\frac{\det P_N^S}{\det P_N^{S'}}\right|
\leq \bar g_N\log\!\left(1+\frac{\Delta_H}{2n\mu}\right).
\]

For the finite-$N$ factors, \Cref{lem:error-nuclear} bounds the nuclear norms of $\delta_N^S$ and $\delta_N^{S'}$ using the rank-one GLM Hessian structure. Applying $\log|\det(I+K)|\leq\|K\|_*$ and the inverse bound supplied by feasibility gives the correction $\tau^{(\mathrm{corr})}$ in \Cref{def:correction}. Adding the frozen-Hessian and finite-$N$ contributions yields \eqref{eq:main-bound}.
\end{sketch}

\begin{remark}[Common-anchor determinant comparison]\label{rem:anchoring}
At a shared output, write
$H_S^*=H_0+uu^\top$ and $H_{S'}^*=H_0+u'u'^\top$, where $H_0\succeq\mu I$ contains the regularizer and the $n-1$ shared samples. The matrix determinant lemma compares each Hessian with the same $H_0$. Since each rank-one addition changes the relevant log determinant by a value in an interval of length
$\log(1+\Delta_H/(2n\mu))$, the difference between the two log determinants has absolute value at most that same quantity. This is the common-anchor calculation used in \Cref{lem:hessian-logdet}.
\end{remark}

\begin{remark}[Classical determinant slack]\label{rem:oracle-recovery}
As $N\to\infty$, both $\tau$ in \Cref{thm:main} and $\tau^\star$ in \eqref{eq:tau-star} converge to
$\log(1+\Delta_H/(2n\mu))$; see \Cref{rem:asymptotic}. In the notation of \citet{chaudhuri2011differentially}, their determinant calculation gives
$2\log(1+\Delta_H/(2n\mu))$: their curvature bound $c$ corresponds to $L_2\zeta^2=\Delta_H/2$, and their regularization strength $\Lambda$ corresponds to $\mu$ in the branch with no additional regularization. The common-anchor calculation in \Cref{rem:anchoring} yields the single logarithmic factor above.
\end{remark}

\subsection{Utility}\label{subsec:results-utility}

We compare the unregularized population risk of the released iterate with that of an arbitrary comparator. The bound includes the comparator's regularization term and holds for every $N\geq1$. The same finite-optimization correction appears throughout this subsection,
\begin{equation}\label{eq:E-opt}
\begin{aligned}
    \mathcal{E}_{\mathrm{opt}}(N)
    &:= \frac{L}{2}\rho^{2N}\bigl( b_R^2 d\sigma^2 + 2 a_R b_R \sigma\sqrt{d} + a_R^2 \bigr)
       + \rho^N\bigl( b_R d\sigma^2 + a_R \sigma\sqrt{d} \bigr), \\
    a_R &:= \frac{\| \nabla F(0; S) \|}{\mu}, \qquad b_R := \frac{1}{\mu}.
\end{aligned}
\end{equation}
Stopping after $N$ steps introduces the correction in \eqref{eq:E-opt}. The distance to the perturbed minimizer decreases as $\rho^N$, which produces the $\rho^{2N}$ smoothness term and the $\rho^N$ inner-product term involving the residual gradient $-z$. Thus $\mathcal E_{\mathrm{opt}}(N)$ decreases geometrically with $N$.

\begin{theorem}[Unregularized excess risk]\label{thm:unregularized}
Under \Cref{ass:bounded-features,ass:bounded-loss,ass:lipschitz-third}, let $z \sim \mathcal{N}(0, \sigma^2 I_d)$ and let $w_N$ be the output of \eqref{eq:gd-recursion}. Then for every comparator $\bar w \in \mathbb{R}^d$, with probability at least $1 - \delta_S$ over $S$,
\[
    \mathbb{E}_z\bigl[ \mathcal{R}_0(w_N) \bigr] - \mathcal{R}_0(\bar w) \leq 2\Bigl[ \frac{d\sigma^2}{2\mu} + \mathcal{E}_{\mathrm{opt}}(N) \Bigr] + \frac{16\, L_\ell^2 \bigl( 32 + \log(1/\delta_S) \bigr)}{\mu n} + \frac{\mu}{2}\| \bar w \|^2.
\]
\end{theorem}

\begin{proof}
Since $\mathcal{R}_0 = \mathcal{R}_\mu - \tfrac{\mu}{2}\| \cdot \|^2$, at every realization of $z$,
\[
    \mathcal{R}_0(w_N) - \mathcal{R}_0(\bar w) = \bigl[ \mathcal{R}_\mu(w_N) - \mathcal{R}_\mu(\bar w) \bigr] - \frac{\mu}{2}\| w_N \|^2 + \frac{\mu}{2}\| \bar w \|^2 \leq \bigl[ \mathcal{R}_\mu(w_N) - \mathcal{R}_\mu(w^*) \bigr] + \frac{\mu}{2}\| \bar w \|^2,
\]
dropping $-\tfrac{\mu}{2}\| w_N \|^2 \leq 0$ and using $\mathcal{R}_\mu(\bar w) \geq \mathcal{R}_\mu(w^*)$, since $w^*$ minimizes $\mathcal{R}_\mu$. Taking $\mathbb{E}_z$ and applying \Cref{thm:population-uniform} on its event bounds the first bracket and gives the claim. 
\end{proof}

For the algebraic comparison below, fix the noise scale $\sigma$ and optimize the public regularization coefficient $\mu$. We evaluate the privacy certificate at the selected value of $\mu$; choosing that value from public information requires no separate composition step. This choice incurs no additional privacy loss. For a prescribed target $(\epsilon,\delta)$, however, $\mu$, $\sigma$, and $N$ are coupled through \Cref{thm:dp} and should be calibrated jointly. We give an explicit joint choice in \Cref{cor:target-budget}.

Writing $A := d\sigma^2 + 16 L_\ell^2\bigl( 32 + \log(1/\delta_S) \bigr)/n$, the right side equals $A/\mu + \tfrac{\mu}{2}\| \bar w \|^2$ up to the geometrically decaying $2\,\mathcal{E}_{\mathrm{opt}}(N)$.  Ignoring that optimization correction for this algebraic comparison, the expression is minimized over $\mu$ at $\mu = \sqrt{2A}/\| \bar w \|$, with value $\| \bar w \|\sqrt{2A}$, the sample-complexity form in which \citet{chaudhuri2011differentially} state their guarantee. Applying $\sqrt{a+b} \leq \sqrt{a} + \sqrt{b}$ twice,
\[
\sqrt{2A} \leq \sqrt{2}\,\sigma\sqrt{d} + 4\sqrt{2}\,\bigl( \sqrt{32} + \sqrt{\log(1/\delta_S)} \bigr) L_\ell/\sqrt{n}
= \sqrt{2}\,\sigma\sqrt{d} + \bigl( 32 + 4\sqrt{2}\,\sqrt{\log(1/\delta_S)} \bigr) L_\ell/\sqrt{n},
\]
the form displayed among the contributions.

\begin{example}\label{ex:logistic}
For a numerical illustration, take regularized logistic regression with features normalized to $\zeta = 1$, sample size $n = 10^6$, dimension $d = 10^4$, and regularization $\mu = 10^{-2}$. The loss constants are $L_1 = 1$, $L_2 = 1/4$, and $L_3 = 1/(6\sqrt{3})$, so $L = \mu + L_2\zeta^2 = 0.26$, $\kappa = 26$, $G = 2/n = 2 \times 10^{-6}$, and $\Delta_H = 1/2$ (\Cref{subsec:notation}). Run the mechanism at step size $\eta = 1/L$, noise scale $\sigma = 5G = 10^{-5}$, and truncation parameter $u = 6$. The threshold \eqref{eq:N-dagger} evaluates to $N \geq 1319$, and we run $N = 1500$: a cost of $nN = 1.5 \times 10^9$ individual-gradient evaluations, or $1500$ full-batch passes.

At that count the finite-$N$ corrections are negligible: $\rho^{N/2} = 1.7 \times 10^{-13}$, so the floor \eqref{eq:eps-zero} is its limit to the displayed digits, $\epsilon_0 = \frac{G^2}{2\sigma^2} + \tau^\star = 0.0200 + 2.5 \times 10^{-5}$. At budget $\epsilon = 1$ the certified profile \eqref{eq:delta-profile} evaluates to $\delta(1) = 5.16 \times 10^{-8}$ (truncation mass $e^{-18} = 1.52 \times 10^{-8}$ plus divergence term $3.64 \times 10^{-8}$), so the release is $(1, 10^{-7})$-differentially private with room to spare.

With the same $(\sigma,N)$, the utility bound gives $\mathcal{E}_{\mathrm{opt}}(N) \leq 3 \times 10^{-27}$, perturbation cost $d\sigma^2/\mu = 10^{-4}$, and sample term $16L_\ell^2(32+\log 100)/(\mu n)=0.0586$ at $\delta_S=10^{-2}$. With probability $0.99$ over the sample, the expected excess risk of the released point is at most $0.059$; the zero predictor has risk $\log 2\approx0.69$. Using \Cref{thm:unregularized}, we also obtain an unregularized excess-risk bound of $0.064$ relative to every unit-norm comparator.

We can also evaluate the more precise certificate of \Cref{thm:dp-full} with the same public radii $r=r_u$ and $r'=r_u+2LG/\mu$. For these radii, the first integer that certifies $(1,10^{-7})$ privacy is $N=519$. The comparison is
\begin{center}
\small
\begin{tabular}{lrr}
\toprule
Certificate & Iterations & Certified $\delta(1)$ \\
\midrule
\Cref{thm:dp}, smallest admissible count & $1319$ & $5.16\times10^{-8}$ \\
\Cref{thm:dp}, count used above & $1500$ & $5.16\times10^{-8}$ \\
\Cref{thm:dp-full}, fixed radii & $519$ & $9.90\times10^{-8}$ \\
\bottomrule
\end{tabular}
\end{center}
At $N=519$, the regularized population excess-risk bound remains $0.0586$. We give the radius check and the certificate calculation in \Cref{subsec:certificate-evaluation}; all quantities used in this comparison are determined by the public problem constants.

\end{example}

To obtain the population bound, we first control the regularized excess $\mathcal{R}_\mu(w_N)-\mathcal{R}_\mu(w^*)$ and then apply the uniform-convergence guarantee of \citet[Theorem~1]{uniform-convergence-fastrates}, stated in the form used here as \Cref{thm:uniform-convergence}.

\begin{restatable}[Excess population risk]{theorem}{thmpopulation}\label{thm:population-uniform}
Under \Cref{ass:bounded-features,ass:bounded-loss,ass:lipschitz-third}, let $z \sim \mathcal{N}(0, \sigma^2 I_d)$ and let $w_N$ be the output of \eqref{eq:gd-recursion}. Then for $\delta_S \in (0,1)$, with probability at least $1 - \delta_S$ over $S$,
\[
    \mathbb{E}_z\bigl[ \mathcal{R}_\mu(w_N) - \mathcal{R}_\mu(w^*) \bigr] \leq 2\Bigl[ \frac{d\sigma^2}{2\mu} + \mathcal{E}_{\mathrm{opt}}(N) \Bigr] + \frac{16\, L_\ell^2 \bigl( 32 + \log(1/\delta_S) \bigr)}{\mu n}.
\]
\end{restatable}

\begin{sketch}[Proof sketch of \Cref{thm:population-uniform}]
We first decompose the empirical excess risk into the effect of the linear perturbation and the finite optimization error. Strong convexity bounds the first contribution by $\|z\|^2/(2\mu)$, and gradient descent controls the second through $\mathcal E_{\mathrm{opt}}(N)$. We then apply uniform convergence to transfer the empirical bound to population risk. This route does not introduce an additional multiplicative factor of $L/\mu$ in the leading statistical terms.
\end{sketch}

\begin{restatable}[Excess empirical risk]{theorem}{thmempirical}\label{thm:empirical-risk}
Under \Cref{ass:bounded-features,ass:bounded-loss,ass:lipschitz-third}, let $z \sim \mathcal{N}(0, \sigma^2 I_d)$ and let $w_N$ be the output of \eqref{eq:gd-recursion}. With $\mathcal{E}_{\mathrm{opt}}(N)$ as in \eqref{eq:E-opt},
\[
    \mathbb{E}_z\bigl[ F(w_N; S) - F(w_S^*; S) \bigr] \leq \frac{d\sigma^2}{2\mu} + \mathcal{E}_{\mathrm{opt}}(N).
\]
\end{restatable}

As $N \to \infty$ the bound tends to $d\sigma^2/(2\mu)$, the standard objective-perturbation rate (\Cref{rem:empirical-limit}).  The preceding results average over the Gaussian perturbation. The same pointwise estimates also give a guarantee for the model produced in a single run.

\begin{corollary}[Utility of a single release]\label{cor:high-probability-utility}
Under \Cref{ass:bounded-features,ass:bounded-loss,ass:lipschitz-third}, fix $N\geq1$ and $\beta\in(0,1)$. Define
\[
 r_\beta=\sigma\left(\sqrt d+\sqrt{2\log(1/\beta)}\right),
 \qquad R_\beta=\frac{L_\ell+r_\beta}{\mu},
\]
\[
 U_N(\beta)=\frac{r_\beta^2}{2\mu}
       +\frac L2\rho^{2N}R_\beta^2
       +\rho^N r_\beta R_\beta.
\]
For every fixed sample $S$, with probability at least $1-\beta$ over the Gaussian perturbation,
\[
 F(w_N;S)-F(w_S^*;S)\leq U_N(\beta).
\]
Under the sampling assumptions of \Cref{thm:population-uniform}, for $\delta_S\in(0,1)$, with probability at least $1-\delta_S-\beta$ over $(S,Z)$,
\[
 \mathcal R_\mu(w_N)-\mathcal R_\mu(w^*)
 \leq2U_N(\beta)
  +\frac{16L_\ell^2(32+\log(1/\delta_S))}{\mu n}.
\]
On this same event, simultaneously for every $\bar w\in\mathbb R^d$,
\[
 \mathcal R_0(w_N)-\mathcal R_0(\bar w)
 \leq2U_N(\beta)
  +\frac{16L_\ell^2(32+\log(1/\delta_S))}{\mu n}
  +\frac\mu2\|\bar w\|^2.
\]
\end{corollary}

\begin{remark}[Comparison with AMP]
Approximate Minima Perturbation (AMP) \citep{iyengar2019towards} releases an approximate minimizer of the perturbed objective together with Gaussian output noise calibrated to a gradient-norm tolerance $\gamma$. Let $F_0(w;S)=\frac1n\sum_{i=1}^n\ell(\langle x_i,w\rangle,y_i)$ and let $\hat w$ minimize $F_0$. For GLM losses, their utility theorem gives
\[
    \mathbb{E}\bigl[ F_0(w_{\mathrm{out}}; S) - F_0(\hat w; S) \bigr] = O\Biggl( \frac{\| \hat w \|\, L_\ell \sqrt{d \log(1/\delta)}}{n\epsilon} + \| \hat w \|\, \sqrt{\frac{L_\ell\, \gamma \sqrt{d \log(1/\delta)}}{\epsilon}} \Biggr).
\]
The first term is the objective-perturbation contribution, while the second depends on the optimization tolerance and becomes lower order when $\gamma=O(1/n^2)$. In our bound, finite optimization appears through $\mathcal E_{\mathrm{opt}}(N)$, which decreases geometrically with the number of gradient steps. AMP accounts for the composition of objective perturbation and the Gaussian output mechanism \citep[Remark~4.2]{redberg2023improving}; our mechanism releases the deterministic $N$-th iterate after the single objective perturbation.
\end{remark}

We next combine the privacy and utility bounds to choose the parameters for a prescribed privacy pair and a public comparator norm bound.
\begin{corollary}[Calibration to a target privacy budget]\label{cor:target-budget}
Under \Cref{ass:bounded-features,ass:bounded-loss,ass:lipschitz-third}, use $G=2L_\ell/n>0$. Fix $0<\epsilon\leq1$, $\delta,\delta_S\in(0,1)$, and a public radius $W>0$. Set
\[
\begin{aligned}
    u&=\sqrt{2\log(2/\delta)},
    &\sigma&=\frac{4Gu}{\epsilon},
    &\mu_{\mathrm{priv}}&=\frac{L_2\zeta^2}{n(e^{\epsilon/4}-1)},\\
    A&=d\sigma^2+\frac{16L_\ell^2(32+\log(1/\delta_S))}{n},
    &\mu_{\mathrm{util}}&=\frac{\sqrt{2A}}{W},
    &\mu&=\max\{\mu_{\mathrm{util}},\mu_{\mathrm{priv}}\}.
\end{aligned}
\]
After choosing $\mu$, take public bounds $L\geq\mu+L_2\zeta^2$ and $M>0$ with $M\geq L_3\zeta^3$, and choose a step size satisfying $0<\eta\leq1/L$ and $\eta\mu<1$. Choose a public integer $N$ satisfying \eqref{eq:N-dagger} and
\[
\begin{split}
 c_N:={}&\frac{G}{\sigma}\rho^{N/2}
 +\frac{G^2}{2\sigma^2}\bigl(2\rho^{N/2}+\rho^N\bigr)\\
 &+\min\{1,MG/L\}\,\frac{e\eta\mu}{2}N\rho^{N/2}
 \ \leq\ \frac{\epsilon}{4}.
\end{split}
\]
Such an integer exists. The release $w_N$ is $(\epsilon,\delta)$-differentially private. Moreover, with probability at least $1-\delta_S$ over the sample, simultaneously for all $\bar w$ with $\|\bar w\|\leq W$,
\[
 \mathbb E_z[\mathcal R_0(w_N)]-\mathcal R_0(\bar w)
 \leq W\sqrt{2A}+\frac{\mu_{\mathrm{priv}}W^2}{2}
       +2\mathcal E_{\mathrm{opt}}(N).
\]
If $\mu_{\mathrm{util}}\geq\mu_{\mathrm{priv}}$, the term $\mu_{\mathrm{priv}}W^2/2$ may be omitted.
\end{corollary}
We use only public quantities to choose the parameters. For the leading term, we obtain
\[
 W\sqrt{2A}
 \leq \frac{8\sqrt{2}\,W L_\ell u\sqrt d}{n\epsilon}
   +\frac{4\sqrt{2}\,W L_\ell\sqrt{32+\log(1/\delta_S)}}{\sqrt n}.
\]
We obtain separate terms for privacy, sampling, regularization, and finite optimization.

\section{Discussion}\label{sec:discussion}

We compare the privacy certificate with composition-free analyses of noisy iterations, explain where the GLM and unconstrained assumptions enter, and then give the resulting computational cost.

\subsection{Three approaches to composition-free privacy}\label{subsec:three-approaches}

We compare our result with two lines of last-iterate privacy analysis for noisy optimization. Privacy amplification by iteration was introduced by \citet{feldman2018privacy}, sharpened for projected noisy SGD by \citet{altschuler2022privacy}, and formulated in $f$-differential privacy by \citet{bok2024shifted}. \citet{chourasia2021differential} analyze privacy-loss dynamics for noisy gradient descent, and \citet{ye2022hidden} extend the hidden-state analysis to mini-batch schemes. \citet{chien2025convergent} further treat nonconvex objectives with H\"older-continuous gradients. The table summarizes the assumptions and iteration dependence of the corresponding certificates.

\begin{center}
\footnotesize
\begin{tabular}{lccc}
\toprule
& \textbf{\begin{tabular}{@{}c@{}}Amplification\\ by iteration\end{tabular}} & \textbf{\begin{tabular}{@{}c@{}}Privacy-loss\\ dynamics\end{tabular}} & \textbf{This work} \\
\midrule
Analytical tool & shifted R\'enyi divergence & \begin{tabular}{@{}c@{}}Fokker--Planck equation,\\ log-Sobolev inequality\end{tabular} & change of variables \\
Noise & Gaussian, per step & Gaussian, per step & Gaussian, one draw \\
Released object & last noisy iterate & last noisy iterate & $w_N$, deterministic in $z$ \\
Requires & \begin{tabular}{@{}c@{}}Lipschitz loss,\\ contractive updates\end{tabular} & \begin{tabular}{@{}c@{}}smoothness,\\ strong convexity\end{tabular} & \begin{tabular}{@{}c@{}}smoothness, strong convexity,\\ Lipschitz Hessian, GLM\end{tabular} \\
Certified bound & rises, then plateaus & plateaus (exp.\ fast) & \begin{tabular}{@{}c@{}}finite-$N$ correction\\ decays geometrically\end{tabular} \\
\bottomrule
\end{tabular}
\end{center}

Privacy amplification by iteration and privacy-loss dynamics both study algorithms that add noise during the iteration. In our setting the randomness enters once through the linear objective perturbation, after which the updates are deterministic. For amplification by iteration, the certified privacy loss rises and then plateaus; for the privacy-loss dynamics it converges exponentially to a plateau. Over the range covered by \eqref{eq:N-dagger}, our finite-iteration correction decreases with $N$ and the certificate approaches the exact-minimizer bound.

All three analyses can give certified bounds without an explicit ambient-dimension factor. In our setting this follows from the GLM structure: neighboring gradients differ in a two-dimensional feature span, Hessian changes have rank-one structure, and the orthogonal complement of the data span evolves identically under neighboring datasets (\Cref{subsec:glm-role}).

The assumptions differ across these analyses. Privacy amplification by iteration uses a Lipschitz loss and contractive updates, while the privacy-loss dynamics use smoothness and strong convexity. Our change-of-variables proof uses smoothness, strong convexity, a Lipschitz Hessian, and GLM structure. The Lipschitz-Hessian assumption supports the frozen-Hessian comparison, and the GLM structure gives the subspace and nuclear-norm bounds used in the privacy analysis.

\subsection{The role of GLM structure and the constrained--unconstrained distinction}\label{subsec:glm-role}

The GLM structure enters the privacy bound in three places. First, the gradient difference between neighboring datasets lies in $\mathrm{span}\{x_n,x_n'\}$, so the leading Gaussian-ratio term depends on two scalar projections of the perturbation. Second, Hessian changes are sums of rank-one matrices, which gives the required nuclear-norm bounds without a factor of $d$. Third, the data-span decomposition makes the $\mathcal C^\perp$ block identical under neighboring datasets, so only the feature span enters the truncation argument. The frozen-Hessian invertibility argument itself uses strong convexity, smoothness, and a Lipschitz Hessian.

We analyze unconstrained gradient descent on $\mathbb R^d$. This preserves the smooth map $z\mapsto w_N$ and the contraction estimates used throughout the proof. In the constrained setting, \citet{noisy-sgd} proved excess-risk lower bounds of $\Omega(\min(1,\sqrt d/(n\epsilon)))$ for convex losses and $\Omega(d/(n^2\epsilon^2))$ for strongly convex losses. For unconstrained GLMs, \citet{song2021evading} obtained excess empirical risk $\widetilde O(L_\ell\sqrt{\mathrm{rank}}/(\epsilon n))$, where $\mathrm{rank}\leq\min\{n,d\}$ is the rank of the feature matrix. Our change-of-variables analysis operates in this unconstrained setting.

\subsection{Computational cost}\label{subsec:computational}
Running $N$ full-batch gradient steps requires $nN$ individual-gradient evaluations. At step size $1/L$, the sufficient condition \eqref{eq:N-dagger} has order $\kappa\log(\cdot)$, giving $O(n\kappa\log(\cdot))$ individual-gradient evaluations. Under the public calibration in \Cref{cor:target-budget}, fix $\epsilon,\delta,\delta_S,W$ and the loss constants. When $d\sigma^2=O(1/n)$ and $\mu_{\mathrm{util}}\geq\mu_{\mathrm{priv}}$, we choose $\mu=\Theta(n^{-1/2})$. Taking $L=\mu+L_2\zeta^2$ with $L_2\zeta^2>0$ then gives $\kappa=\Theta(\sqrt n)$. In this regime, $O(\sqrt n\log n)$ iterations suffice for the explicit privacy certificate and for an optimization correction of order $n^{-1/2}$ or smaller. The resulting sufficient gradient count is $O(n^{3/2}\log n)$. For comparison, DP-SGD \citep{noisy-sgd} attains the optimal constrained-ERM rates after $T = n^2$ steps. \citet{bassily2019private} obtain $O(n \log n)$ gradient complexity for smooth GLMs by using SVRG to solve the perturbed objective to a prescribed accuracy and then adding Gaussian output noise to certify the remaining optimization error. For private stochastic convex optimization, \citet{feldman2020linear} attain optimal population-risk rates with at most linear gradient complexity under their smoothness assumptions. Our result analyzes the deterministic full-batch gradient-descent iterate after a single objective perturbation.

\section*{Acknowledgements}
This work was supported, in part, by 
NSF BIGDATA award IIS-1838139 and NSF CAREER award IIS-1943251.

\bibliographystyle{plainnat}
\bibliography{references}

\newpage

\appendix
% cleveref's own appendix hook does not fire under the LaTeX 2025 kernel's
% first aid; the alias makes appendix sections read "Appendix" on every TeX Live.
\crefalias{section}{appendix}
\crefalias{subsection}{subappendix}
\crefname{appendix}{Appendix}{Appendices}\Crefname{appendix}{Appendix}{Appendices}
\crefname{subappendix}{Appendix}{Appendices}\Crefname{subappendix}{Appendix}{Appendices}

\section{The map \texorpdfstring{$\mathcal{M}$}{M} is a diffeomorphism}\label{sec:diffeo}

On any bounded convex set and for all sufficiently large $N$ (with the threshold depending on the set), the restriction of $\mathcal{M}$ to the set is a $C^1$-diffeomorphism onto its image (\Cref{cor:diffeo}). The argument combines two ingredients: \emph{local} invertibility, from a lower bound on the smallest singular value of the Jacobian $J_N$, and \emph{global} injectivity, from negative definiteness of its symmetric part $J_N^{\mathrm{sym}}$. We obtain both by comparing $J_N$ to the Jacobian $\tilde{J}_N$ of the \emph{frozen-Hessian model}---the map obtained by freezing the Hessian at its value at the released iterate $w_N$, equivalently by replacing $F$ with its quadratic Taylor expansion there---whose spectrum we compute in closed form.
\subsection{General setting}\label{subsec:general-setting}

We prove the structural results below under assumptions that do not require the generalized linear form. The GLM setting of \Cref{subsec:notation} satisfies these assumptions with the constants derived there.

Let $F : \mathbb{R}^d \to \mathbb{R}$ be twice continuously differentiable. Our analysis rests on two regularity assumptions on $F$.

\begin{assumption}[Strong convexity and smoothness]\label{ass:regularity}
$F$ is $\mu$-strongly convex and $L$-smooth: $\mu I \preceq \nabla^2 F(w) \preceq L I$ for all $w \in \mathbb{R}^d$, where $0 < \mu \leq L$.
\end{assumption}

\begin{assumption}[Lipschitz Hessian]\label{ass:lipschitz-hessian}
The Hessian of $F$ is $M$-Lipschitz with $M > 0$: $\| \nabla^2 F(w) - \nabla^2 F(v) \| \leq M \| w - v \|$ for all $w, v \in \mathbb{R}^d$.
\end{assumption}

For the differential-privacy analysis we consider a family of functions $\{F(\cdot\,; S)\}$ indexed by datasets $S$. Two datasets $S$ and $S'$ are \emph{neighboring} if they differ in at most one data point. We require \Cref{ass:regularity,ass:lipschitz-hessian} to hold for every member of the family with common constants $\mu, L, M$ (that is, uniformly in $S$), and impose the following cross-dataset condition.

\begin{assumption}[Gradient sensitivity]\label{ass:gradient-sensitivity}
Neighboring datasets perturb the gradient by at most $G > 0$: $\| \nabla F(w; S) - \nabla F(w; S') \| \leq G$ for all $w \in \mathbb{R}^d$ and all neighboring $S, S'$.
\end{assumption}

\begin{lemma}[Averaged derivative]\label{lem:averaged-derivative}
Let $\mathcal{K} : \mathbb{R}^d \to \mathbb{R}^d$ be continuously differentiable. For $a, b \in \mathbb{R}^d$ set
\begin{equation}\label{eq:averaged-derivative}
    \overline{\nabla\mathcal{K}}(a, b) := \int_0^1 \nabla\mathcal{K}\bigl( b + t(a - b) \bigr)\, dt,
\end{equation}
The fundamental theorem of calculus along the segment from $b$ to $a$ then reads $\mathcal{K}(a) - \mathcal{K}(b) = \overline{\nabla\mathcal{K}}(a, b)\,(a - b)$. At $\mathcal{K} = \nabla F$ we abbreviate $\bar{H}(a, b) := \int_0^1 \nabla^2 F\bigl( b + t(a-b) \bigr)\, dt$. Under \Cref{ass:regularity}:
\begin{enumerate}
    \item\label{itm:avg-hessian} $\nabla F(a) - \nabla F(b) = \bar{H}(a,b)\,(a - b)$, with $\bar{H}(a,b)$ symmetric and $\mu I \preceq \bar{H}(a,b) \preceq L I$;
    \item\label{itm:contraction} $\| I - \eta H \| \leq \rho$ for every symmetric $H$ with $\mu I \preceq H \preceq L I$; in particular $\| A_k \| \leq \rho$ at every iterate, and $\| I - \eta\bar{H}(a,b) \| \leq \rho$ for all $a, b$.
\end{enumerate}
\end{lemma}

\begin{proof}
The map $t \mapsto \mathcal{K}(b + t(a-b))$ is continuously differentiable on $[0,1]$ with derivative $\nabla\mathcal{K}(b + t(a-b))\,(a-b)$, so the fundamental theorem of calculus gives $\mathcal{K}(a) - \mathcal{K}(b) = \int_0^1 \nabla\mathcal{K}(b + t(a-b))\,(a-b)\, dt$, out of which the constant vector $a - b$ factors. For item~\ref{itm:avg-hessian}, each integrand $\nabla^2 F(b + t(a-b))$ is symmetric with spectrum in $[\mu, L]$ (\Cref{ass:regularity}), and averaging preserves both properties: $\bar{H}(a,b)$ is symmetric, and $\langle \bar{H}(a,b)\,v, v \rangle = \int_0^1 \langle \nabla^2 F(b + t(a-b))\,v, v \rangle\, dt$ lies between $\mu\| v \|^2$ and $L\| v \|^2$ for every $v$. For item~\ref{itm:contraction}, $I - \eta H$ is symmetric with eigenvalues in $[1 - \eta L,\, 1 - \eta\mu]$, an interval of nonnegative numbers since $\eta L \leq 1$, so $\| I - \eta H \| \leq 1 - \eta\mu = \rho$; the two instances are $H = \nabla^2 F(w_k)$, for which $I - \eta H = A_k$, and $H = \bar{H}(a,b)$.
\end{proof}

\subsection{The trajectory}\label{subsec:prelim}

We use the perturbed minimizer as the reference point for the utility analysis and control its distance from the iterates through the trajectory bounds below.

\begin{lemma}[Perturbed minimizer]\label{lem:perturbed-minimizer}
For each $z \in \mathbb{R}^d$ there is a unique $w^*(z) \in \mathbb{R}^d$ satisfying $\nabla F(w^*(z)) + z = 0$, namely the minimizer of $w \mapsto F(w) + \langle z, w \rangle$.
\end{lemma}

\begin{proof}
The linear term has zero Hessian, so $\nabla^2\bigl[F(w) + \langle z, w \rangle\bigr] = \nabla^2 F(w) \succeq \mu I$ by \Cref{ass:regularity}; the perturbed objective is therefore $\mu$-strongly convex. Existence, uniqueness, and the first-order characterization of the minimizer of a strongly convex function are standard (\Cref{lem:ext-minimizer}).
\end{proof}

The trajectory estimates all compare two runs of \eqref{eq:gd-recursion} driven by nearly the same gradient: the step map contracts by $\rho$, so a run driven within $\beta$ of the exact one stays within $\beta/\mu$ of it.

\begin{lemma}[Comparison of runs]\label{lem:comparison}
Fix $z \in \mathbb{R}^d$, let $(w_k)_{k \geq 0}$ be the iterates \eqref{eq:gd-recursion}, and let $(w_k')_{k \geq 0}$ satisfy $w_{k+1}' = w_k' - \eta\, u_k$ for vectors $u_k$ obeying
\[
    \bigl\| u_k - \bigl( \nabla F(w_k') + z \bigr) \bigr\| \; \leq \; \beta \qquad \text{for every } k \geq 0.
\]
Then, for every $k \geq 0$,
\begin{equation}\label{eq:comparison}
    \| w_k - w_k' \| \; \leq \; \rho^k\, \| w_0 - w_0' \| \; + \; \frac{\beta}{\mu}\bigl( 1 - \rho^k \bigr).
\end{equation}
\end{lemma}

\begin{proof}
Subtracting the two recursions and splitting the driver difference at $w_k'$,
\[
    w_{k+1} - w_{k+1}' = (w_k - w_k') - \eta\bigl( \nabla F(w_k) - \nabla F(w_k') \bigr) - \eta\bigl( \nabla F(w_k') + z - u_k \bigr).
\]
By item~\ref{itm:avg-hessian} of \Cref{lem:averaged-derivative} the first gradient difference is $\bar{H}(w_k, w_k')(w_k - w_k')$, so the first two terms together are $\bigl( I - \eta\bar{H}(w_k, w_k') \bigr)(w_k - w_k')$, of norm at most $\rho\,\| w_k - w_k' \|$ by item~\ref{itm:contraction}; the last term has norm at most $\eta\beta$ by hypothesis. Hence $\| w_{k+1} - w_{k+1}' \| \leq \rho\,\| w_k - w_k' \| + \eta\beta$, and iterating from $k = 0$,
\[
    \| w_k - w_k' \| \; \leq \; \rho^k\,\| w_0 - w_0' \| + \eta\beta\sum_{j=0}^{k-1}\rho^j \; = \; \rho^k\,\| w_0 - w_0' \| + \frac{\beta}{\mu}\bigl( 1 - \rho^k \bigr),
\]
using $1 - \rho = \eta\mu$.
\end{proof}

The minimizer is the run that stands still, so comparing the iterates with it drives them to $w^*(z)$ at the rate $\rho$.

\begin{lemma}[Convergence of the iterates]\label{lem:convergence}
For every $z \in \mathbb{R}^d$ and every $k \geq 0$,
\[
    \| w_k - w^*(z) \| \leq \rho^k\, \| w_0 - w^*(z) \|.
\]
\end{lemma}

\begin{proof}
By \Cref{lem:perturbed-minimizer}, $\nabla F(w^*(z)) + z = 0$, so the constant sequence $w_k' \equiv w^*(z)$ satisfies $w_{k+1}' = w_k' - \eta\bigl( \nabla F(w_k') + z \bigr)$: it is the run of \eqref{eq:gd-recursion} started at the minimizer, and it never moves. \Cref{lem:comparison} at $\beta = 0$ with $w_0' = w^*(z)$ is the claim. This is the standard contraction estimate for gradient descent in the strongly convex, smooth regime; \citet[\S 3.4.2]{gd-linear-convergence} states it at $\eta = 1/L$, and the argument above gives it for every $\eta \in (0, 1/L]$.
\end{proof}

The minimizer also depends stably on the perturbation, and its distance from any point is controlled by the perturbed gradient there.

\begin{lemma}[Sensitivity of the perturbed minimizer]\label{lem:w-star}
For all $z, z' \in \mathbb{R}^d$,
\[
    \| w^*(z) - w^*(z') \| \leq \frac{1}{\mu}\, \| z - z' \|,
\]
and consequently $\| w^*(z) - w \| \leq \tfrac{1}{\mu}\| \nabla F(w) + z \|$ for every $w \in \mathbb{R}^d$.
\end{lemma}

\begin{proof}
Run \eqref{eq:gd-recursion} at $z$ and at $z'$, writing $(w_k)$ and $(w_k')$ for the two trajectories; both start at $w_0$, so each is the iterate sequence to which \Cref{lem:convergence} applies. The second is driven by $u_k = \nabla F(w_k') + z'$, so $\| u_k - (\nabla F(w_k') + z) \| = \| z - z' \|$ for every $k$, and \Cref{lem:comparison} with $w_0' = w_0$ and $\beta = \| z - z' \|$ gives
\[
    \| w_k - w_k' \| \; \leq \; \frac{\| z - z' \|}{\mu}\bigl( 1 - \rho^k \bigr) \; \leq \; \frac{\| z - z' \|}{\mu} \qquad \text{for every } k \geq 0.
\]
By \Cref{lem:convergence} the two trajectories converge to $w^*(z)$ and $w^*(z')$, so the bound passes to the limit. For the second bound, fix $w$ and take $z' = -\nabla F(w)$; then $\nabla F(w) + z' = 0$, so $w^*(-\nabla F(w)) = w$ by \Cref{lem:perturbed-minimizer}, and the Lipschitz bound yields $\| w^*(z) - w \| \leq \tfrac{1}{\mu}\| \nabla F(w) + z \|$.
\end{proof}

The scale every subsequent bound is measured against is set by the iteration itself: the length of its first step. Since consecutive steps contract, that one length bounds every later step and the distance to the minimizer.

\begin{lemma}[Trajectory scale]\label{lem:traj-scale}
Set
\begin{equation}\label{eq:R-def}
    R(z) := \frac{\| \nabla F(w_0) + z \|}{\mu} = \frac{\| w_1 - w_0 \|}{\eta\mu}.
\end{equation}
Then, for every $z \in \mathbb{R}^d$:
\begin{enumerate}
    \item\label{itm:R-affine} $R(z) \leq b_R\,\| z \| + a_R$, with $b_R := \frac{1}{\mu}$ and $a_R := \frac{\| \nabla F(w_0) \|}{\mu}$;
    \item\label{itm:step-decay} $\| w_{k+1} - w_k \| \leq \eta\mu\,\rho^{k}\,R(z)$ for every $k \geq 0$, and consequently
    \begin{equation}\label{eq:iterate-anchor}
        \| w_k - w_N \| \leq \bigl( \rho^k - \rho^N \bigr)\, R(z), \qquad 0 \leq k \leq N;
    \end{equation}
    \item\label{itm:init-dist} $\| w_0 - w^*(z) \| \leq R(z)$.
\end{enumerate}
For a bounded set $D \subset \mathbb{R}^d$ we write $R_D := \sup_{z \in D} R(z)$, which is finite.
\end{lemma}

\begin{proof}
The second expression in \eqref{eq:R-def} is the first step of \eqref{eq:gd-recursion}, $w_1 - w_0 = -\eta(\nabla F(w_0) + z)$, and item~\ref{itm:R-affine} is the triangle inequality.

For item~\ref{itm:step-decay}, the shifted sequence $w_k' := w_{k+1}$ is the run of \eqref{eq:gd-recursion} started at $w_1$ instead of $w_0$, so \Cref{lem:comparison} at $\beta = 0$ compares the trajectory with its own time shift:
\[
    \| w_k - w_{k+1} \| \; \leq \; \rho^k\, \| w_0 - w_1 \| \; = \; \eta\mu\,\rho^k\, R(z),
\]
the equality by \eqref{eq:R-def}. This is the per-step bound; summing it over $j = k, \ldots, N-1$ and using $1 - \rho = \eta\mu$,
\[
    \| w_k - w_N \| \; \leq \; \sum_{j=k}^{N-1}\| w_{j+1} - w_j \| \; \leq \; \eta\mu\,R(z)\sum_{j=k}^{N-1}\rho^j \; = \; R(z)\bigl( \rho^k - \rho^N \bigr).
\]
Item~\ref{itm:init-dist} is the second bound of \Cref{lem:w-star} at $w = w_0$. Since $D$ is bounded, $\sup_{z \in D}\| z \| < \infty$, and item~\ref{itm:R-affine} gives $R_D \leq b_R \sup_{z \in D}\| z \| + a_R < \infty$.
\end{proof}

Two runs from the shared initialization, under neighboring datasets and the same $z$, are driven within $G$ of each other at every step, so they never separate by more than $G/\mu$.

\begin{lemma}[Map sensitivity]\label{lem:map-sensitivity}
For any $z \in \mathbb{R}^d$ and any $k \geq 0$,
\[
    \| w_k^S - w_k^{S'} \| \leq \frac{G}{\mu}\bigl( 1 - \rho^k \bigr) \leq \frac{G}{\mu};
\]
in particular, at the final iterate,
\[
    \| \mathcal{M}_S(z) - \mathcal{M}_{S'}(z) \| \leq \frac{G}{\mu}\bigl( 1 - \rho^N \bigr) \leq \frac{G}{\mu}.
\]
\end{lemma}

\begin{proof}
The two runs share the initialization, $w_0^S = w_0^{S'} = w_0$. Read against $F(\cdot\,; S)$, the run under $S'$ is driven by $u_k = \nabla F(w_k^{S'}; S') + z$, and
\[
    \bigl\| u_k - \bigl( \nabla F(w_k^{S'}; S) + z \bigr) \bigr\| = \bigl\| \nabla F(w_k^{S'}; S') - \nabla F(w_k^{S'}; S) \bigr\| \leq G
\]
by \Cref{ass:gradient-sensitivity}. \Cref{lem:comparison} for $F(\cdot\,; S)$, with $w_0' = w_0$ and $\beta = G$, gives the first bound, and the second follows from $1 - \rho^k \leq 1$. Taking $k = N$ gives the final display, since $\mathcal{M}_S(z) = w_N^S$ and $\mathcal{M}_{S'}(z) = w_N^{S'}$.
\end{proof}

\subsection{The frozen-Hessian model}\label{subsec:reference}

To analyze the Jacobian $J_N = \nabla \mathcal{M}(z)$, we compare it with what we call the \emph{frozen-Hessian model}: the mechanism run with the Hessian held fixed at its value at the point the mechanism releases. The model solves in closed form, and at every $N \geq 1$ each eigenvalue of its Jacobian is at most $-\gamma_N$, with $\gamma_N = (1 - (1 - \eta L)^N)/L \to 1/L$ (\Cref{lem:frozen-spectrum})---the floor the true Jacobian will inherit. Fix $z \in \mathbb{R}^d$ and a horizon $N$, and set
\[
    H^* := \nabla^2 F\bigl( \mathcal{M}(z) \bigr) = \nabla^2 F(w_N), \qquad A^* := I - \eta H^*,
\]
so that $H^*$ is symmetric with eigenvalues in $[\mu, L]$ by \Cref{ass:regularity}---which holds at every point, so the anchor may be any one---hence invertible. Freezing the Hessian at $H^*$ amounts to replacing $F$ with its quadratic Taylor expansion at $w_N$; the Jacobian recursion \eqref{eq:jacobian-recursion} then runs with the iterate-dependent $A_k$ replaced by the constant matrix $A^*$. Writing $\tilde{J}_k$ for the resulting Jacobian, an auxiliary sequence indexed by $k \leq N$ at the anchor the horizon $N$ fixes,
\begin{equation}\label{eq:reference-recursion}
    \tilde{J}_{k+1} = A^* \tilde{J}_k - \eta I, \qquad \tilde{J}_0 = 0.
\end{equation}
Here we evaluate $H^*$ at the fixed base perturbation $z$ and hold it constant throughout the auxiliary recursion. Because $A^*$ is constant, this affine recursion unrolls into a geometric sum. Its fixed point is $-(H^*)^{-1}$, since substituting $\tilde{J} = -(H^*)^{-1}$ into \eqref{eq:reference-recursion} returns it unchanged.

\begin{lemma}[Closed form of $\tilde{J}_N$]\label{lem:closed-form}
For all $N \geq 0$,
\begin{equation}\label{eq:reference-geometric}
    \tilde{J}_N = -\eta \sum_{j=0}^{N-1} (A^*)^j = -(H^*)^{-1}\bigl(I - (A^*)^N\bigr).
\end{equation}
\end{lemma}

\begin{proof}
The first equality is an induction on $N$. The sum is empty at $N = 0$, matching $\tilde{J}_0 = 0$; and if it holds at $N$, then \eqref{eq:reference-recursion} gives
\[
    \tilde{J}_{N+1} = A^*\Bigl( -\eta \sum_{j=0}^{N-1}(A^*)^j \Bigr) - \eta I = -\eta \sum_{j=1}^{N}(A^*)^j - \eta (A^*)^0 = -\eta \sum_{j=0}^{N}(A^*)^j.
\]
For the second equality, $I - A^* = \eta H^*$, so the telescoping identity $I - (A^*)^N = (I - A^*)\sum_{j=0}^{N-1}(A^*)^j$ reads $I - (A^*)^N = \eta H^*\sum_{j=0}^{N-1}(A^*)^j$; multiplying on the left by $(H^*)^{-1}$ and negating gives the claim.
\end{proof}

Because every term is a power of the symmetric matrix $A^*$, we can read off the symmetry, spectrum, operator-norm bound, and singular-value floor directly in an eigenbasis of $H^*$.

\begin{lemma}[Spectrum and norms of $\tilde{J}_N$]\label{lem:frozen-spectrum}
Let $h_1, \ldots, h_d \in [\mu, L]$ be the eigenvalues of $H^*$.
\begin{enumerate}
\item\label{itm:Jt-eigen} $\tilde{J}_N$ is symmetric, with eigenvalues
\[
    \tilde{\lambda}_i(N) = -\eta\sum_{j=0}^{N-1}(1 - \eta h_i)^j = -\frac{1}{h_i}\bigl( 1 - (1 - \eta h_i)^N \bigr), \qquad i = 1, \ldots, d.
\]
\item\label{itm:Jt-norm} For all $N \geq 0$,
\[
    \| \tilde{J}_N \| \leq \frac{1 - \rho^N}{\mu} =: B_N \leq \frac{1}{\mu} =: B.
\]
\item\label{itm:Jt-floor} For every $N \geq 1$ every eigenvalue of $\tilde{J}_N$ is negative, and
\[
    \sigma_{\min}(\tilde{J}_N) \geq \frac{1 - (1 - \eta L)^N}{L} =: \gamma_N \geq \eta > 0,
\]
with equality in the first inequality whenever $L$ is an eigenvalue of $H^*$.
\end{enumerate}
\end{lemma}

\begin{proof}
\emph{Item~\ref{itm:Jt-eigen}.} Each $(A^*)^j$ is a power of the symmetric matrix $A^*$, hence symmetric, so \eqref{eq:reference-geometric} makes $\tilde{J}_N$ symmetric. Fix an orthonormal eigenbasis $v_1, \ldots, v_d$ of $H^*$, with $H^* v_i = h_i v_i$. Then $A^* v_i = (1 - \eta h_i) v_i$, so $v_i$ is an eigenvector of every $(A^*)^j$ with eigenvalue $(1 - \eta h_i)^j$, and by \eqref{eq:reference-geometric} an eigenvector of $\tilde{J}_N$ with eigenvalue $-\eta\sum_{j<N}(1 - \eta h_i)^j$. Because $h_i \in [\mu, L]$ and $\eta L \leq 1$, the ratio satisfies $0 \leq 1 - \eta h_i \leq 1 - \eta\mu = \rho < 1$, so summing the geometric series gives the second expression. Ranging over $i$ exhausts the spectrum, and symmetry makes $\| \tilde{J}_N \| = \max_i |\tilde{\lambda}_i(N)|$ and $\sigma_{\min}(\tilde{J}_N) = \min_i |\tilde{\lambda}_i(N)|$.

\emph{Items~\ref{itm:Jt-norm} and~\ref{itm:Jt-floor}.} By item~\ref{itm:Jt-eigen} and symmetry, every singular value of $\tilde{J}_N$ is the value at some $h_i$ of the scalar profile
\[
    f_N(h) := \eta\sum_{j=0}^{N-1}(1 - \eta h)^j = \frac{1 - (1 - \eta h)^N}{h}, \qquad h \in [\mu, L].
\]
On this interval $1 - \eta h \in [1 - \eta L, \rho] \subseteq [0, 1)$, so each summand $(1 - \eta h)^j$ is a nonnegative power of the nonincreasing map $h \mapsto 1 - \eta h$, hence itself nonnegative and nonincreasing in $h$; the profile is therefore nonincreasing on $[\mu, L]$, and both norm bounds are its endpoint values. At the left endpoint the sum is geometric with ratio $\rho$, so
\[
    \| \tilde{J}_N \| = \max_i f_N(h_i) \leq f_N(\mu) = \eta\,\frac{1 - \rho^N}{1 - \rho} = \frac{1 - \rho^N}{\mu} = B_N,
\]
using $1 - \rho = \eta\mu$. Since $\rho^N \geq 0$, $B_N \leq 1/\mu$. For $N \geq 1$ the $j = 0$ summand alone gives $f_N \geq \eta > 0$ on all of $[\mu, L]$, so every eigenvalue $-f_N(h_i)$ is negative, and at the right endpoint
\[
    \sigma_{\min}(\tilde{J}_N) = \min_i f_N(h_i) \geq f_N(L) = \frac{1 - (1 - \eta L)^N}{L} = \gamma_N \geq \eta > 0,
\]
with equality whenever some $h_i$ equals $L$: the floor is the minimum of the profile over the curvature interval, not a relaxation of it.
\end{proof}

\subsection{Perturbation bounds}\label{subsec:perturbation}

We now bound the gap between the true Jacobian $J_N$ and the Jacobian $\tilde{J}_N$ of the frozen-Hessian model. Write $\delta_k := J_k - \tilde{J}_k$ for the Jacobian error; it is controlled by the operator-norm deviation $\| \nabla^2 F(w_k) - H^* \|$ of the Hessian along the trajectory, and $\| \delta_N \|$ will be shown to decay as $\tilde{\nu}_N \rho^{N-1}$, where the factor $\tilde{\nu}_N$ weights each step by how far the reference Jacobian has grown by then and by how far the iterate still stands from the anchor. Both factors vanish at an end: the Hessian deviation is largest at $k = 0$ and zero at $k = N$, while the reference Jacobian starts at zero and grows.

\begin{lemma}[Error recursion]\label{lem:error-recursion}
The Jacobian error satisfies
\[
    \delta_{k+1} = A_k \delta_k - \eta\bigl(\nabla^2 F(w_k) - H^*\bigr)\tilde{J}_k, \qquad \delta_0 = 0,
\]
and hence, writing $\Theta_{N,k+1} := A_{N-1}A_{N-2}\cdots A_{k+1}$ for the state-transition product ($\Theta_{N,N} := I$), unrolls to
\begin{equation}\label{eq:error-unrolled}
    \delta_N = -\eta \sum_{k=0}^{N-1} \Theta_{N,k+1}\,\bigl( \nabla^2 F(w_k) - H^* \bigr)\,\tilde{J}_k, \qquad N \geq 0.
\end{equation}
\end{lemma}

\begin{proof}
Subtracting the frozen-Hessian recursion \eqref{eq:reference-recursion} from the Jacobian recursion \eqref{eq:jacobian-recursion},
\[
    \delta_{k+1} = (A_k J_k - \eta I) - (A^* \tilde{J}_k - \eta I) = A_k J_k - A_k \tilde{J}_k + A_k \tilde{J}_k - A^* \tilde{J}_k = A_k(J_k - \tilde{J}_k) + (A_k - A^*)\tilde{J}_k.
\]
The first term is $A_k \delta_k$. For the second, $A_k - A^* = (I - \eta\nabla^2 F(w_k)) - (I - \eta H^*) = -\eta\bigl(\nabla^2 F(w_k) - H^*\bigr)$, giving $-\eta(\nabla^2 F(w_k) - H^*)\tilde{J}_k$. Finally $\delta_0 = J_0 - \tilde{J}_0 = 0 - 0 = 0$.

Induction on the iterate index, the anchor $H^*$ held at $\nabla^2 F(w_N)$ throughout, gives \eqref{eq:error-unrolled}: the sum is empty at index $0$, matching $\delta_0 = 0$; and applying the recursion once left-multiplies every term present by $A_N$, which turns each $\Theta_{N,k+1}$ into $\Theta_{N+1,k+1}$, and appends the $k = N$ term with $\Theta_{N+1,N+1} = I$.
\end{proof}

The inhomogeneous term contains the Hessian difference $\nabla^2F(w_k)-H^*$. Since $H^*=\nabla^2F(w_N)$, the Lipschitz-Hessian assumption and \eqref{eq:iterate-anchor} give, for all $0\leq k\leq N$,
\begin{equation}\label{eq:hessian-error}
    \| \nabla^2 F(w_k) - H^* \| \; \leq \; M \| w_k - w_N \| \; \leq \; M\bigl( \rho^k - \rho^N \bigr) R(z).
\end{equation}
The deviation vanishes at $k = N$, where the anchor is read. Feeding this decay through the error recursion bounds $\| \delta_N \|$.

\begin{proposition}[Jacobian error bound]\label{prop:error-bound}
Let $\| \cdot \|_\diamond$ be either the operator norm or the nuclear norm, and suppose the Hessian of $F$ is $M$-Lipschitz in it: $\| \nabla^2 F(w) - \nabla^2 F(v) \|_\diamond \leq M\, \| w - v \|$ for all $w, v$. \Cref{ass:lipschitz-hessian} is this hypothesis at the operator norm. Let $C(z) := \eta M R(z)\, B$ and
\begin{equation}\label{eq:nu-N}
    \nu_N := \sum_{k=0}^{N-1}\bigl( 1 - \rho^k \bigr) = N - \frac{1 - \rho^N}{\eta\mu}, \qquad 0 \leq \nu_N \leq N.
\end{equation}
For all $N \geq 1$,
\[
    \| \delta_N \|_\diamond \leq C(z)\, \tilde{\nu}_N\, \rho^{N-1},
\]
where
\begin{equation}\label{eq:nu-tilde}
    \tilde{\nu}_N := \sum_{k=0}^{N-1}\bigl( 1 - \rho^{N-k} \bigr)\bigl( 1 - \rho^k \bigr) = \nu_N + N\rho^N - \frac{\rho\,\bigl( 1 - \rho^N \bigr)}{\eta\mu}, \qquad 0 \; \leq \; \tilde{\nu}_N \; \leq \; \nu_N,
\end{equation}
and $\tilde{\nu}_1 = \nu_1 = 0$.
\end{proposition}

\begin{proof}
Both norms satisfy $\| XYZ \|_\diamond \leq \| X \|\,\| Y \|_\diamond\,\| Z \|$: for the operator norm this is submultiplicativity, for the nuclear norm the Schatten--H\"older bound of \Cref{lem:ext-trace-nuclear}. Apply it to each summand of \eqref{eq:error-unrolled}, taking $Y$ to be the Hessian-deviation factor. Each $A_j$ satisfies $\| A_j \| \leq \rho$ (item~\ref{itm:contraction} of \Cref{lem:averaged-derivative}), so submultiplicativity over the $N-1-k$ factors gives $\| \Theta_{N,k+1} \| \leq \rho^{N-1-k}$. With the Hessian deviation bound \eqref{eq:hessian-error}, which the hypothesis gives verbatim in $\| \cdot \|_\diamond$, and $\| \tilde{J}_k \| \leq B_k = (1 - \rho^k)/\mu$ (item~\ref{itm:Jt-norm} of \Cref{lem:frozen-spectrum}), the $k$-th summand has $\| \cdot \|_\diamond$ at most
\[
    \eta\,\rho^{N-1-k}\cdot M \bigl( \rho^k - \rho^N \bigr) R(z) \cdot \frac{1 - \rho^k}{\mu} = C(z)\,\rho^{N-1}\bigl( 1 - \rho^{N-k} \bigr)\bigl( 1 - \rho^k \bigr),
\]
since $C(z) = \eta M R(z)/\mu$ and $\rho^{N-1-k}\bigl( \rho^k - \rho^N \bigr) = \rho^{N-1}\bigl( 1 - \rho^{N-k} \bigr)$. Summing over $k$ gives $\| \delta_N \|_\diamond \leq C(z)\,\tilde{\nu}_N \rho^{N-1}$ with $\tilde{\nu}_N$ the sum in \eqref{eq:nu-tilde}, nonnegative because every factor is. For its closed form, split the summand: $\sum_{k<N}(1 - \rho^k) = \nu_N$, while $\sum_{k<N}\rho^{N-k}(1 - \rho^k) = \sum_{k<N}\bigl( \rho^{N-k} - \rho^N \bigr) = \sum_{j=1}^{N}\rho^{j} - N\rho^N = \rho(1 - \rho^N)/(\eta\mu) - N\rho^N$, using $\sum_{j=1}^{N}\rho^j = \rho(1-\rho^N)/(1-\rho)$ and $1 - \rho = \eta\mu$. That second sum, subtracted here rather than added, is a sum of nonnegative terms, giving $\tilde{\nu}_N \leq \nu_N$; at $N = 1$ its only term is $(1-\rho)(1 - \rho^0) = 0$, so $\tilde{\nu}_1 = 0$. The closed form in \eqref{eq:nu-N} follows from $\sum_{k=0}^{N-1}\rho^k = (1 - \rho^N)/(1 - \rho)$ with $1 - \rho = \eta\mu$, and $0 \leq 1 - \rho^k \leq 1$ for every $k \geq 0$ gives $0 \leq \nu_N \leq N$.
\end{proof}

\subsection{Local invertibility, injectivity, and the diffeomorphism}\label{subsec:diffeo-proof}

We can now prove that $\mathcal{M}$ is a diffeomorphism on a bounded convex set. The perturbation bound transfers the spectral control from $\tilde{J}_N$ to the symmetric part of $J_N$, and negative definiteness of $J_N^{\mathrm{sym}}$ yields both the singular-value floor---hence local invertibility---and global injectivity on convex sets. Throughout, $\nabla\mathcal{M}(z) = J_N$.

\begin{theorem}[Negative definiteness and invertibility of the Jacobian]\label{thm:neg-def-J}
Let $D \subset \mathbb{R}^d$ be bounded, write $C_D := \eta M R_D\, B$, and suppose $N$ satisfies $\gamma_N > C_D\, \tilde{\nu}_N \rho^{N-1}$. Then, with $\alpha_D := \gamma_N - C_D\, \tilde{\nu}_N \rho^{N-1} > 0$, for every $z \in D$:
\begin{enumerate}
\item\label{itm:negdef} $J_N(z)^{\mathrm{sym}} \preceq -\alpha_D I$;
\item\label{itm:sigmin} $\sigma_{\min}(\nabla\mathcal{M}(z)) \geq \alpha_D$, so $\nabla\mathcal{M}(z)$ is invertible and $\mathcal{M}$ is a local diffeomorphism on $D$.
\end{enumerate}
Since $\gamma_N \to 1/L$ and $C_D\, \tilde{\nu}_N \rho^{N-1} \to 0$ as $N \to \infty$, the hypothesis holds for all sufficiently large $N$.
\end{theorem}

\begin{proof}
The hypothesis gives $\gamma_N > C_D\, \tilde{\nu}_N \rho^{N-1} \geq 0$, and $\gamma_N = (1 - (1 - \eta L)^N)/L > 0$ forces $N \geq 1$, the hypothesis of item~\ref{itm:Jt-floor} of \Cref{lem:frozen-spectrum}. By items~\ref{itm:Jt-eigen} and~\ref{itm:Jt-floor} of that lemma, every eigenvalue of the symmetric $\tilde{J}_N$ is negative with modulus at least $\gamma_N$, i.e.\ $\tilde{J}_N \preceq -\gamma_N I$. Write $J_N = \tilde{J}_N + \delta_N$, so $J_N^{\mathrm{sym}} = \tilde{J}_N + \delta_N^{\mathrm{sym}}$. For any unit vector $v$,
\[
    \langle J_N^{\mathrm{sym}} v, v \rangle = \langle \tilde{J}_N v, v \rangle + \langle \delta_N^{\mathrm{sym}} v, v \rangle \leq -\gamma_N + \| \delta_N^{\mathrm{sym}} \| \leq -\gamma_N + \| \delta_N \| \leq -\gamma_N + C_D\, \tilde{\nu}_N \rho^{N-1} = -\alpha_D,
\]
using $\| \delta_N^{\mathrm{sym}} \| \leq \| \delta_N \|$ from \Cref{subsec:notation} and \Cref{prop:error-bound} (with $C(z) \leq C_D$ for $z \in D$); as this holds for every unit $v$, item~\ref{itm:negdef} follows. For item~\ref{itm:sigmin}, Cauchy--Schwarz and the identity $\langle A v, v \rangle = \langle A^{\mathrm{sym}} v, v \rangle$ (\Cref{subsec:notation}) give, for every unit $v$,
\[
    \| J_N v \| \geq |\langle J_N v, v \rangle| = |\langle J_N^{\mathrm{sym}} v, v \rangle| \geq \alpha_D,
\]
so $\sigma_{\min}(\nabla\mathcal{M}(z)) = \sigma_{\min}(J_N) \geq \alpha_D > 0$ and $\nabla\mathcal{M}(z)$ is invertible; as $\mathcal{M} \in C^1$, the inverse function theorem \citep[Theorem~9.24]{rudin-principles} makes $\mathcal{M}$ a local diffeomorphism at each $z \in D$. Finally $\gamma_N \to 1/L > 0$ while $\tilde{\nu}_N \rho^{N-1} \leq N \rho^{N-1} \to 0$ for $\rho \in (0,1)$, so the hypothesis holds once $N$ is large enough.
\end{proof}

On a convex domain, negative definiteness of the symmetric part also makes $\mathcal{M}$ injective.

\begin{theorem}[Injectivity]\label{thm:injective}
Let $D \subset \mathbb{R}^d$ be a bounded convex set satisfying the hypotheses of \Cref{thm:neg-def-J}. Then for all $z_1, z_2 \in D$,
\[
    \langle \mathcal{M}(z_1) - \mathcal{M}(z_2),\, z_1 - z_2 \rangle \leq -\alpha_D \| z_1 - z_2 \|^2;
\]
in particular $\| \mathcal{M}(z_1) - \mathcal{M}(z_2) \| \geq \alpha_D \| z_1 - z_2 \|$, and $\mathcal{M}$ is injective on $D$.
\end{theorem}

\begin{proof}
Let $z_1, z_2 \in D$; the claims are trivial for $z_1 = z_2$, so assume $z_1 \neq z_2$. Since $\mathcal{M}$ is $C^1$ with $\nabla\mathcal{M} = J_N$, \Cref{lem:averaged-derivative} at $\mathcal{K} = \mathcal{M}$ gives $\mathcal{M}(z_1) - \mathcal{M}(z_2) = \overline{\nabla\mathcal{M}}(z_1, z_2)\,(z_1 - z_2)$, the average taken along the segment $z_2 + t(z_1 - z_2)$, which lies in $D$ by convexity. Taking symmetric parts commutes with that average, so item~\ref{itm:negdef} of \Cref{thm:neg-def-J} bounds each integrand and
\[
    \overline{\nabla\mathcal{M}}(z_1, z_2)^{\mathrm{sym}} = \int_0^1 J_N\bigl( z_2 + t(z_1 - z_2) \bigr)^{\mathrm{sym}}\, dt \; \preceq \; -\alpha_D I.
\]
With the identity $\langle A v, v \rangle = \langle A^{\mathrm{sym}} v, v \rangle$ (\Cref{subsec:notation}) at $v = z_1 - z_2$,
\[
    \langle \mathcal{M}(z_1) - \mathcal{M}(z_2),\, z_1 - z_2 \rangle = \bigl\langle \overline{\nabla\mathcal{M}}(z_1, z_2)^{\mathrm{sym}}\,(z_1 - z_2),\, z_1 - z_2 \bigr\rangle \leq -\alpha_D \| z_1 - z_2 \|^2 < 0,
\]
so $\mathcal{M}(z_1) \neq \mathcal{M}(z_2)$; and Cauchy--Schwarz gives $\| \mathcal{M}(z_1) - \mathcal{M}(z_2) \|\, \| z_1 - z_2 \| \geq \alpha_D \| z_1 - z_2 \|^2$, which after dividing by $\| z_1 - z_2 \|$ is the displacement bound.
\end{proof}

Local invertibility and injectivity together make $\mathcal{M}$ a diffeomorphism onto its image.

\cordiffeo*

\begin{proof}
By \Cref{thm:neg-def-J}, $\mathcal{M}$ is a local diffeomorphism at each point of $D$; in particular it is an open map, so $\mathcal{M}(D)$ is open. By \Cref{thm:injective} it is injective on the convex set $D$, hence a bijection onto $\mathcal{M}(D)$. Near each $y \in \mathcal{M}(D)$ the global inverse $\mathcal{M}^{-1}$ coincides with the $C^1$ local inverse furnished by the inverse function theorem \citep[Theorem~9.24]{rudin-principles} (these agree by injectivity), so $\mathcal{M}^{-1}$ is $C^1$ on $\mathcal{M}(D)$. Thus $\mathcal{M} : D \to \mathcal{M}(D)$ is a $C^1$-diffeomorphism.
\end{proof}

\section{GLM structure and rank-one Hessian perturbations}\label{sec:glm}

Three features of GLMs remove explicit ambient-dimension factors from the privacy bound. Hessian changes are sums of rank-one matrices with controlled coefficients, so we can bound their nuclear norms directly rather than using the generic inequality $\|A\|_*\leq d\|A\|$. The gradient difference between neighboring datasets lies in the span of the two differing feature vectors. Finally, the data span separates from its orthogonal complement, where both neighboring mechanisms follow the same affine recursion. Throughout,
\[
\nabla^2F_S(w)=\frac1n\sum_{i=1}^n\ell''(\langle x_i,w\rangle,y_i)x_ix_i^\top+\mu I.
\]

The same rank-one form fixes the range of curvature the analysis meets, supplying the constants $\mu$ and $L$ of \Cref{ass:regularity} for GLMs.

\begin{lemma}[GLM curvature range]\label{lem:glm-smoothness}
Under \Cref{ass:bounded-features,ass:bounded-loss}, for each dataset $T \in \{S, S'\}$ and all $w \in \mathbb{R}^d$,
\[
    \mu I \preceq \nabla^2 F_T(w) \preceq \bigl( \mu + L_2\zeta^2 \bigr) I.
\]
In particular \Cref{ass:regularity} holds in the GLM setting for any $L \geq \mu + L_2\zeta^2$.
\end{lemma}

\begin{proof}
Each coefficient $\ell''(\langle x_i, w \rangle, y_i)$ is nonnegative (\Cref{ass:bounded-loss}) and each $x_i x_i^\top$ is positive semidefinite, so the sum is positive semidefinite and $\nabla^2 F_T(w) \succeq \mu I$. For the upper bound the same two facts give $\tfrac{1}{n}\sum_i \ell''(\langle x_i, w \rangle, y_i)\, x_i x_i^\top \preceq L_2\, \tfrac{1}{n}\sum_i x_i x_i^\top$, and the triangle inequality with $\| x_i x_i^\top \| = \| x_i \|^2 \leq \zeta^2$ (\Cref{ass:bounded-features}) bounds the operator norm of the second factor by $\zeta^2$.
\end{proof}

The same rank-one structure makes the Hessian Lipschitz in nuclear norm.

\begin{lemma}[GLM Hessian displacement]\label{lem:hessian-lipschitz}
Under \Cref{ass:bounded-features,ass:lipschitz-third}, for each dataset $T \in \{S, S'\}$ and all $w, w' \in \mathbb{R}^d$,
\[
    \| \nabla^2 F_T(w) - \nabla^2 F_T(w') \|_* \leq L_3 \zeta^3\, \| w - w' \|,
\]
and hence the same bound in operator norm, since $\| A \| \leq \| A \|_*$. In particular \Cref{ass:lipschitz-hessian} holds in the GLM setting for any $M \geq L_3 \zeta^3$.
\end{lemma}

\begin{proof}
The constant regularizer Hessian cancels, leaving
\[
    \nabla^2 F_T(w) - \nabla^2 F_T(w') = \frac{1}{n}\sum_{i=1}^{n} c_i\, x_i x_i^\top, \qquad c_i := \ell''(\langle x_i, w \rangle, y_i) - \ell''(\langle x_i, w' \rangle, y_i).
\]
\Cref{ass:lipschitz-third}, read at the arguments $\langle x_i, w \rangle$ and $\langle x_i, w' \rangle$, followed by Cauchy--Schwarz, gives $|c_i| \leq L_3\,| \langle x_i, w - w' \rangle | \leq L_3\,\| x_i \|\,\| w - w' \|$. Each $c_i\, x_i x_i^\top$ is rank-one, so its nuclear norm is $|c_i|\,\| x_i \|^2$, and the triangle inequality with $\| x_i \| \leq \zeta$ (\Cref{ass:bounded-features}) gives
\[
    \| \nabla^2 F_T(w) - \nabla^2 F_T(w') \|_* \leq \frac{1}{n}\sum_{i=1}^{n} |c_i|\,\| x_i \|^2 \leq \frac{L_3\,\| w - w' \|}{n}\sum_{i=1}^{n} \| x_i \|^3 \leq L_3 \zeta^3\, \| w - w' \|. \qedhere
\]
\end{proof}

\begin{remark}\label{rem:nuclear-dimension-free}
\Cref{lem:hessian-lipschitz} removes the dimension $d$ from that contribution. For a general symmetric $A \in \mathbb{R}^{d \times d}$ one has only $\| A \|_* \leq d\, \| A \|$, and this is tight. The GLM Hessian perturbation, being a sum of rank-one terms with individually controlled coefficients, instead obeys in nuclear norm the bound one would expect only in operator norm, eliminating the factor of $d$.
\end{remark}

Carrying this through \Cref{prop:error-bound} bounds the \emph{nuclear} norm of the Jacobian error by the same quantity as its operator norm---the form the log-determinant analysis of \Cref{sec:log-det} requires. We state it per dataset $T \in \{S, S'\}$. Here and below, $F_T := F(\cdot\,; T)$; each $F_T$ satisfies \Cref{ass:regularity,ass:lipschitz-hessian} with the uniform constants of \Cref{subsec:notation}, so every result of \Cref{sec:diffeo} holds for $F_T$, and we cite those results for $F_T$ directly. Likewise $R_T(z)$ denotes the trajectory scale \eqref{eq:R-def} of $F_T$, and $w^{*,T}(z)$ its perturbed minimizer (\Cref{lem:perturbed-minimizer})---the objects written $R(z)$ and $w^*(z)$ when a single dataset is in scope.

\begin{lemma}[Nuclear norm of the Jacobian perturbation error]\label{lem:error-nuclear}
Fix $T \in \{S, S'\}$ and a perturbation $z_T$, and write $R^T := R_T(z_T)$ for the trajectory scale \eqref{eq:R-def} of $F_T$, $C_T := \eta M R^T B$, and $\delta_N^T := J_N^T(z_T) - \tilde{J}_N^T$ for the Jacobian error of \Cref{lem:error-recursion} applied to $F_T$. Under \Cref{ass:bounded-features,ass:lipschitz-third,ass:regularity},
\[
    \| \delta_N^T \|_* \leq C_T\, \tilde{\nu}_N \rho^{N-1},
\]
with $\tilde{\nu}_N$ as in \eqref{eq:nu-tilde}.
\end{lemma}

\begin{proof}
\Cref{lem:hessian-lipschitz} makes the Hessian of $F_T$ $M$-Lipschitz in nuclear norm for any $M \geq L_3\zeta^3$, which is the hypothesis of \Cref{prop:error-bound} at $\| \cdot \|_\diamond = \| \cdot \|_*$. Applying that proposition to $F_T$ at $z_T$ gives the claim, with $C(z_T) = C_T$.
\end{proof}

A further structural fact, again from the rank-one form of GLM gradients, makes the leading Gaussian-ratio term of \Cref{sec:gaussian-ratio} dimension free: the gradient difference between neighboring datasets is confined to a fixed subspace of dimension at most two.

\begin{lemma}[Gradient subspace structure]\label{lem:gradient-subspace}
In the GLM setting, for neighboring datasets $S, S'$ differing in the $n$-th sample, the gradient difference $\nabla F_S(w) - \nabla F_{S'}(w)$ lies in $V := \mathrm{span}\{ x_n, x_n' \}$, a subspace of dimension at most $2$, for every $w \in \mathbb{R}^d$. Moreover \Cref{ass:gradient-sensitivity} holds with $G = 2L_1\zeta/n$.
\end{lemma}

\begin{proof}
The per-sample gradients in \eqref{eq:empirical-risk} agree except at the $n$-th sample, and the regularizer gradient cancels (the regularizer is dataset-independent), so
\[
    \nabla F_S(w) - \nabla F_{S'}(w) = \frac{1}{n}\bigl[ \ell'(\langle x_n, w \rangle, y_n)\, x_n - \ell'(\langle x_n', w \rangle, y_n')\, x_n' \bigr],
\]
which lies in $\mathrm{span}\{ x_n, x_n' \}$. The triangle inequality with $|\ell'| \leq L_1$ (\Cref{ass:bounded-loss}) and $\| x_n \|, \| x_n' \| \leq \zeta$ (\Cref{ass:bounded-features}) gives
\[
    \| \nabla F_S(w) - \nabla F_{S'}(w) \| \leq \frac{1}{n}\bigl( L_1 \| x_n \| + L_1 \| x_n' \| \bigr) \leq \frac{2L_1\zeta}{n},
\]
which is \Cref{ass:gradient-sensitivity} with $G = 2L_1\zeta/n$.
\end{proof}

The subspace $V$ confines the gradient \emph{difference}. The feature span $\mathcal{C}$ of \eqref{eq:data-span} confines the mechanism itself, and $V \subseteq \mathcal{C}$. Write $d_{\mathcal{C}} := \dim \mathcal{C} \leq d_\star$ and $z = z_{\mathcal{C}} + z_{\mathcal{C}^\perp}$ for the orthogonal decomposition of $z \in \mathbb{R}^d$.

\begin{lemma}[Data-span reduction]\label{lem:data-span}
Let $\mathcal{C}$ be as in \eqref{eq:data-span} and let $B_N = (1-\rho^N)/\mu > 0$ be as in item~\ref{itm:Jt-norm} of \Cref{lem:frozen-spectrum}. For each $T \in \{S, S'\}$:
\begin{enumerate}
    \item\label{itm:span-mechanism} $\mathcal{M}_T(z) = \mathcal{M}_T^{\mathcal{C}}(z_{\mathcal{C}}) - B_N z_{\mathcal{C}^\perp}$, where $\mathcal{M}_T^{\mathcal{C}}$ is the mechanism of \Cref{subsec:mechanism} run on the restriction $F_T|_{\mathcal{C}}$ within $\mathcal{C}$. The second summand is independent of both the dataset and $z_{\mathcal C}$. A fixed public initialization $w_0\neq0$ adds the dataset-independent term $\rho^N\Pi_{\mathcal C^\perp}w_0$.
    \item\label{itm:span-jacobian} $J_N^T(z)\,v = J_N^{T,\mathcal{C}}(z_{\mathcal{C}})\,v_{\mathcal{C}} - B_N v_{\mathcal{C}^\perp}$ for every $v \in \mathbb{R}^d$, where $J_N^{T,\mathcal{C}}$ is the Jacobian of $\mathcal{M}_T^{\mathcal{C}}$.
    \item\label{itm:span-restriction} $F_T|_{\mathcal{C}}$ is again an empirical risk \eqref{eq:empirical-risk}, on the features $x_i$ read in $\mathcal{C}$ and with the same $\ell$, $\zeta$, $\mu$; so it satisfies \Cref{ass:bounded-features,ass:bounded-loss,ass:lipschitz-third}, and with them \Cref{ass:regularity,ass:lipschitz-hessian} with the same $\mu, L, M$, while the pair $F_S|_{\mathcal{C}}, F_{S'}|_{\mathcal{C}}$ satisfies \Cref{ass:gradient-sensitivity} with the same $G$.
\end{enumerate}
\end{lemma}

\begin{proof}
Throughout, $\nabla^2 F_T(w)$ maps $\mathcal{C}$ into $\mathcal{C}$ and acts on $\mathcal{C}^\perp$ as $\mu I$, at every $w$: by \eqref{eq:empirical-risk}, $\nabla^2 F_T(w) = \tfrac{1}{n}\sum_i \ell''(\langle x_i, w \rangle, y_i)\, x_i x_i^\top + \mu I$, and every $x_i \in \mathcal{C}$, so each $x_i x_i^\top$ has range in $\mathcal{C}$ while $\langle x_i, v \rangle = 0$ for $v \in \mathcal{C}^\perp$ makes $x_i x_i^\top v = 0$; $\mu I$ preserves both subspaces.

\emph{Item~\ref{itm:span-mechanism}.} Differentiating \eqref{eq:empirical-risk}, $\nabla F_T(w) = \tfrac{1}{n}\sum_i \ell'(\langle x_i, w \rangle, y_i)\, x_i + \mu w$. Since $x_i \in \mathcal{C}$, projecting gives $\Pi_{\mathcal{C}^\perp} \nabla F_T(w) = \mu\, \Pi_{\mathcal{C}^\perp} w$ and $\Pi_{\mathcal{C}} \nabla F_T(w) = \tfrac{1}{n}\sum_i \ell'(\langle x_i, w \rangle, y_i)\, x_i + \mu\, \Pi_{\mathcal{C}} w$, which depends on $w$ only through $\Pi_{\mathcal{C}} w$ because $\langle x_i, w \rangle = \langle x_i, \Pi_{\mathcal{C}} w \rangle$. Projecting \eqref{eq:gd-recursion} therefore yields two recursions that do not communicate,
\[
    \Pi_{\mathcal{C}} w_{k+1} = \Pi_{\mathcal{C}} w_k - \eta\bigl( \nabla (F_T|_{\mathcal{C}})(\Pi_{\mathcal{C}} w_k) + z_{\mathcal{C}} \bigr), \qquad
    \Pi_{\mathcal{C}^\perp} w_{k+1} = \rho\, \Pi_{\mathcal{C}^\perp} w_k - \eta\, z_{\mathcal{C}^\perp},
\]
both started at $0$ since $w_0 = 0$. The first is \eqref{eq:gd-recursion} for $F_T|_{\mathcal{C}}$ within $\mathcal{C}$, so $\Pi_{\mathcal{C}} w_N = \mathcal{M}_T^{\mathcal{C}}(z_{\mathcal{C}})$. The second unrolls to $\Pi_{\mathcal{C}^\perp} w_N = -\eta\bigl( \sum_{k<N}\rho^k \bigr) z_{\mathcal{C}^\perp} = -\tfrac{\eta(1-\rho^N)}{1-\rho}\, z_{\mathcal{C}^\perp} = -B_N z_{\mathcal{C}^\perp}$, using $1 - \rho = \eta\mu$; started instead at $w_0$, it unrolls to $\rho^N \Pi_{\mathcal{C}^\perp} w_0 - B_N z_{\mathcal{C}^\perp}$. As $\rho \in (0,1)$ and $N \geq 1$, $B_N > 0$.

\emph{Item~\ref{itm:span-jacobian}.} Differentiate item~\ref{itm:span-mechanism} in $z$, using that $z \mapsto z_{\mathcal{C}}$ and $z \mapsto z_{\mathcal{C}^\perp}$ are the linear maps $\Pi_{\mathcal{C}}$ and $\Pi_{\mathcal{C}^\perp}$.

\emph{Item~\ref{itm:span-restriction}.} Writing $Q$ for a matrix whose columns are an orthonormal basis of $\mathcal{C}$, the restriction reads $F_T|_{\mathcal{C}}(\cdot) = \tfrac{1}{n}\sum_i \ell(\langle Q^\top x_i, \cdot \rangle, y_i) + \tfrac{\mu}{2}\| \cdot \|^2$, since $\langle x_i, Q\,\cdot \rangle = \langle Q^\top x_i, \cdot \rangle$. This is \eqref{eq:empirical-risk} in $\mathcal{C}$, and $\| Q^\top x_i \| = \| x_i \| \leq \zeta$ because $x_i \in \mathcal{C}$, so \Cref{ass:bounded-features,ass:bounded-loss,ass:lipschitz-third} hold there with the same constants; \Cref{lem:gradient-subspace} and the bounds of \Cref{sec:glm} follow for the restricted pair, as does $\| \nabla (F_T|_{\mathcal{C}})(0) \| \leq L_1\zeta$. For the general-setting assumptions, fix $w, v \in \mathcal{C}$. By the splitting above and symmetry of $\nabla^2 F_T(w)$, the orthogonal splitting $\mathbb{R}^d = \mathcal{C} \oplus \mathcal{C}^\perp$ is invariant, so the spectrum of $\nabla^2 (F_T|_{\mathcal{C}})(w)$---the restriction of $\nabla^2 F_T(w)$ to $\mathcal{C}$---is contained in that of $\nabla^2 F_T(w)$, hence in $[\mu, L]$ by \Cref{ass:regularity}: this is \Cref{ass:regularity} for $F_T|_{\mathcal{C}}$. Restricting an operator does not increase its norm, so
$\| \nabla^2 (F_T|_{\mathcal{C}})(w) - \nabla^2 (F_T|_{\mathcal{C}})(v) \| \leq \| \nabla^2 F_T(w) - \nabla^2 F_T(v) \| \leq M \| w - v \|$, which is \Cref{ass:lipschitz-hessian}. Finally $\nabla (F_S|_{\mathcal{C}})(w) - \nabla (F_{S'}|_{\mathcal{C}})(w) = \Pi_{\mathcal{C}}\bigl( \nabla F_S(w) - \nabla F_{S'}(w) \bigr)$, and the difference lies in $V \subseteq \mathcal{C}$ by \Cref{lem:gradient-subspace}, so $\Pi_{\mathcal{C}}$ fixes it and its norm is still at most $G$.
\end{proof}

\section{Privacy}\label{sec:privacy}

We analyze the mechanism that releases $y = \mathcal{M}_S(z)$ for $z \sim \mathcal{N}(0, \sigma^2 I_d)$. Throughout this section \Cref{ass:regularity,ass:lipschitz-hessian} hold uniformly across datasets with common constants $\mu, L, M$, \Cref{ass:gradient-sensitivity} is a cross-dataset condition by definition, and $\sup_S \| \nabla F(0; S) \| < \infty$, so that the constants of \Cref{lem:traj-scale} are uniform in $S$, with $a_R = \mu^{-1}\sup_S \| \nabla F(0; S) \|$ and $b_R = 1/\mu$.

\begin{convention}[Closed restriction]\label{conv:closed-restriction}
For a bounded convex set $D$, ``$\mathcal{M}_T$ is a $C^1$-diffeomorphism on $D$'' abbreviates the statement that $\mathcal{M}_T$ restricts to $D$ from a $C^1$-diffeomorphism on a convex open neighborhood of $D$, as furnished by \Cref{cor:diffeo}; for open $D$ this is \Cref{cor:diffeo} itself.
\end{convention}

Fix neighboring datasets $S, S'$ and a bounded convex set $\Omega$ on which both $\mathcal{M}_S$ and $\mathcal{M}_{S'}$ are $C^1$-diffeomorphisms (\Cref{conv:closed-restriction}). For $y \in \mathcal{M}_S(\Omega)$ with preimage $z_S := \mathcal{M}_S^{-1}(y) \in \Omega$, the change-of-variables formula (\Cref{lem:ext-pushforward}) expresses the contribution to the output density from preimages in $\Omega$ as
\[
    p_Y^S(y; \Omega) := \frac{p_Z(z_S)}{|\det J_N^S(z_S)|},
\]
where $p_Z$ is the $\mathcal{N}(0, \sigma^2 I_d)$ density. When $y$ also has a preimage $z_{S'} := \mathcal{M}_{S'}^{-1}(y) \in \Omega$, the restricted privacy loss splits into two terms:
\begin{equation}\label{eq:privacy-decomp}
    \log \frac{p_Y^S(y; \Omega)}{p_Y^{S'}(y; \Omega)} = \underbrace{\log \frac{p_Z(z_S)}{p_Z(z_{S'})}}_{\text{Gaussian ratio}} + \underbrace{\log \frac{|\det J_N^{S'}(z_{S'})|}{|\det J_N^S(z_S)|}}_{\text{Jacobian ratio}}.
\end{equation}
Both terms are already confined to $\mathcal{C}$. If $y = \mathcal{M}_S(z_S) = \mathcal{M}_{S'}(z_{S'})$, item~\ref{itm:span-mechanism} of \Cref{lem:data-span} gives $B_N \Pi_{\mathcal{C}^\perp} z_S = B_N \Pi_{\mathcal{C}^\perp} z_{S'}$ and hence $\Pi_{\mathcal{C}^\perp} z_S = \Pi_{\mathcal{C}^\perp} z_{S'}$, so the $\mathcal{C}^\perp$ components cancel from the Gaussian ratio $(\| z_{S'} \|^2 - \| z_S \|^2)/(2\sigma^2)$; and item~\ref{itm:span-jacobian} factors each determinant as $\det J_N^{T, \mathcal{C}}(z_{\mathcal{C}}) \cdot (-B_N)^{\,d - d_{\mathcal{C}}}$, whose second factor is common to the two datasets and cancels from the Jacobian ratio. What \eqref{eq:privacy-decomp} measures is therefore the privacy loss of the pair $\mathcal{M}_S^{\mathcal{C}}, \mathcal{M}_{S'}^{\mathcal{C}}$, which by item~\ref{itm:span-restriction} is a mechanism of \Cref{subsec:general-setting} within $\mathcal{C}$ carrying the same $\mu, L, M, G$, driven by the perturbation $\Pi_{\mathcal{C}} Z \sim \mathcal{N}(0, \sigma^2 I_{d_{\mathcal{C}}})$, again isotropic. Every result of \Cref{sec:diffeo,sec:glm} applies to it with $d_{\mathcal{C}}$ in place of $d$, and the ambient dimension survives only in the $\mathcal{C}^\perp$ block, which multiplies both densities by the same factor.

We treat the Gaussian ratio in \Cref{sec:gaussian-ratio} using the two-dimensional gradient subspace from \Cref{lem:gradient-subspace}, and we treat the Jacobian ratio in \Cref{sec:log-det} using the nuclear-norm bounds from \Cref{sec:glm}. The residual terms depend on the trajectory scale over the truncation domain and decrease with $N$. Choosing $N$ according to \eqref{eq:N-dagger} makes these residual terms small enough to obtain the dimension-free bound in \Cref{thm:dp}. Section~\ref{sec:dp-guarantee} combines the pointwise estimates with Gaussian truncation and the hockey-stick divergence.

\subsection{The Gaussian-ratio bound}\label{sec:gaussian-ratio}

The Gaussian ratio in \eqref{eq:privacy-decomp} is $\log[p_Z(z_S)/p_Z(z_{S'})] = (\| z_{S'} \|^2 - \| z_S \|^2)/(2\sigma^2)$, so bounding it amounts to controlling how far the two preimages of a shared output can differ. We build up to a bound whose only randomness is the projection factor of \Cref{prop:term-a} (\Cref{thm:gaussian-bound}). The trajectories under neighboring datasets separate by at most $G/\mu$ (\Cref{lem:map-sensitivity}); what remains is to invert the relation between $z$ and $y = \mathcal{M}(z)$. The next identity expresses the perturbation $z$ through the gradient at the output, up to a residual that decays geometrically in $N$. We state it for a generic objective $F$ with mechanism $\mathcal{M}$; it applies verbatim to both $F(\cdot\,; S)$ and $F(\cdot\,; S')$.

\begin{lemma}[One-step residual identity]\label{lem:one-step}
Let $z \in \mathbb{R}^d$, let $w_0, w_1, \ldots, w_N$ be the iterates \eqref{eq:gd-recursion}, and set $y := \mathcal{M}(z) = w_N$. With $\bar{H}$ the averaged Hessian $\bar{H}(y, w_{N-1})$ of \Cref{lem:averaged-derivative} over the final step, one has
\begin{equation}\label{eq:one-step-id}
    z = -\nabla F(y) + \xi_F(z), \qquad \xi_F(z) := \Bigl( \tfrac{1}{\eta} I - \bar{H} \Bigr)(w_{N-1} - y),
\end{equation}
and, under \Cref{ass:regularity},
\begin{equation}\label{eq:xi-bound}
    \| \xi_F(z) \| \leq \mu\, \rho^N\, R(z), \qquad R(z) \text{ the trajectory scale of \Cref{lem:traj-scale}}.
\end{equation}
\end{lemma}

\begin{proof}
The final step of \eqref{eq:gd-recursion} reads $y = w_{N-1} - \eta(\nabla F(w_{N-1}) + z)$, so
\begin{equation}\label{eq:final-step-rearr}
    \nabla F(w_{N-1}) + z = \frac{w_{N-1} - y}{\eta}.
\end{equation}
By item~\ref{itm:avg-hessian} of \Cref{lem:averaged-derivative}, $\nabla F(y) - \nabla F(w_{N-1}) = \bar{H}\,(y - w_{N-1})$. Adding $\nabla F(w_{N-1}) + z$ to both sides and using \eqref{eq:final-step-rearr},
\[
    \nabla F(y) + z = \frac{w_{N-1} - y}{\eta} + \bar{H}\,(y - w_{N-1}) = \Bigl( \tfrac{1}{\eta} I - \bar{H} \Bigr)(w_{N-1} - y),
\]
which rearranges to \eqref{eq:one-step-id}.

For \eqref{eq:xi-bound} we bound the two factors of $\xi_F(z)$. The second is a single step of the iteration, $w_{N-1} - y = w_{N-1} - w_N$, so item~\ref{itm:step-decay} of \Cref{lem:traj-scale} at $k = N-1$ gives
\[
    \| w_{N-1} - y \| \leq \eta\mu\, \rho^{N-1}\, R(z).
\]
The first factor is $\tfrac{1}{\eta}\bigl( I - \eta\bar{H} \bigr)$, of norm at most $\rho/\eta$ by item~\ref{itm:contraction} of \Cref{lem:averaged-derivative}. Multiplying the two bounds,
\[
    \| \xi_F(z) \| \leq \frac{\rho}{\eta}\cdot\eta\mu\, \rho^{N-1}\, R(z) = \mu\, \rho^N\, R(z). \qedhere
\]
\end{proof}

To compare the two preimages of a common output we use the \emph{conjugacy map}. Fix the bounded convex set $\Omega$ introduced above, and assume it is chosen so that $\Omega \subseteq \Omega'$ and $\mathcal{M}_S(\Omega) \subseteq \mathcal{M}_{S'}(\Omega')$ for some bounded convex set $\Omega'$ on which $\mathcal{M}_{S'}$ is a $C^1$-diffeomorphism (\Cref{conv:closed-restriction}); this strengthens the standing assumption to cover $\Omega'$, and \Cref{sec:dp-guarantee} constructs such a pair explicitly. Then
\[
    \Phi := \mathcal{M}_{S'}^{-1} \circ \mathcal{M}_S : \Omega \to \Omega', \qquad \mathcal{M}_{S'}(\Phi(z)) = \mathcal{M}_S(z) \ \text{ for all } z \in \Omega,
\]
is well-defined and sends each $z$ to the $\mathcal{M}_{S'}$-preimage of the same output $\mathcal{M}_S(z)$. We abbreviate $\mathcal{M} := \mathcal{M}_S$, $\mathcal{M}' := \mathcal{M}_{S'}$, $F := F(\cdot\,; S)$, $F' := F(\cdot\,; S')$, and set
\[
    R_{\Omega'} := \sup_{z \in \Omega',\, T \in \{S, S'\}} R_T(z),
\]
the largest trajectory scale over $\Omega'$ for either dataset, finite by \Cref{lem:traj-scale} since $\Omega'$ is bounded. As $z \in \Omega \subseteq \Omega'$ and $\Phi(z) \in \Omega'$, both $R_S(z)$ and $R_{S'}(\Phi(z))$ are at most $R_{\Omega'}$. We also fix a radius $r > 0$ with the closed ball $\{ z : \| z \| \leq r \} \subseteq \Omega$. These standing choices constitute the \emph{hypotheses of \Cref{sec:gaussian-ratio}}.

\begin{corollary}[Preimage difference]\label{cor:preimage}
For $z \in \Omega$, set $z' := \Phi(z)$ and $y := \mathcal{M}(z) = \mathcal{M}'(z')$, and define
\[
    g(y) := \nabla F(y) - \nabla F'(y), \qquad \xi_\Delta := \xi_F(z) - \xi_{F'}(z').
\]
Then
\begin{equation}\label{eq:preimage-id}
    z - z' = -g(y) + \xi_\Delta,
\end{equation}
with $\| g(y) \| \leq G$ by \Cref{ass:gradient-sensitivity} and
\begin{equation}\label{eq:xi-delta-bound}
    \| \xi_\Delta \| \leq \mu \rho^N \bigl( R_S(z) + R_{S'}(z') \bigr) \leq 2 \mu \rho^N R_{\Omega'}.
\end{equation}
\end{corollary}

\begin{proof}
\Cref{lem:one-step} applied to $(z, F)$ gives $z = -\nabla F(y) + \xi_F(z)$, since $y = \mathcal{M}(z)$. Applied to $(z', F')$ it gives $z' = -\nabla F'(\mathcal{M}'(z')) + \xi_{F'}(z') = -\nabla F'(y) + \xi_{F'}(z')$, since $\mathcal{M}'(z') = y$ by conjugacy. Subtracting,
\[
    z - z' = -\bigl( \nabla F(y) - \nabla F'(y) \bigr) + \bigl( \xi_F(z) - \xi_{F'}(z') \bigr) = -g(y) + \xi_\Delta.
\]
The bound $\| g(y) \| \leq G$ is \Cref{ass:gradient-sensitivity} for the neighboring pair $S, S'$. For \eqref{eq:xi-delta-bound}, the triangle inequality and the residual bound \eqref{eq:xi-bound}---with the constants $L, \rho$ common to $F$ and $F'$---applied to $\xi_F(z)$ and to $\xi_{F'}(z')$ give $\| \xi_\Delta \| \leq \mu \rho^N R_S(z) + \mu \rho^N R_{S'}(z')$, and both terms are at most $\mu \rho^N R_{\Omega'}$ by the definition of $R_{\Omega'}$.
\end{proof}

We now pass to the random perturbation. Let $Z \sim \mathcal{N}(0, \sigma^2 I_d)$. On the event $\{ Z \in \Omega \}$ set $Z' := \Phi(Z)$ and $Y := \mathcal{M}(Z)$, so the $S$- and $S'$-preimages of $Y$ are $Z$ and $Z'$, and the Gaussian ratio in \eqref{eq:privacy-decomp} is $\log[p_Z(Z)/p_Z(Z')] = (\| Z' \|^2 - \| Z \|^2)/(2\sigma^2)$. With $g, \xi_\Delta$ as in \Cref{cor:preimage}, the conjugacy applies on $\{ Z \in \Omega \}$ and gives $Z' = Z + g(Y) - \xi_\Delta$. For the exact-minimizer map, the optimality condition gives a preimage difference of exactly $g(Y)$. We therefore add and subtract $Z+g(Y)$ in the finite-iterate Gaussian ratio, separating the exact-minimizer contribution from the residual $-\xi_\Delta$:
\begin{equation}\label{eq:decomp}
    \Lambda(Z) := \frac{\| Z' \|^2 - \| Z \|^2}{2\sigma^2} = \underbrace{\frac{\| Z + g(Y) \|^2 - \| Z \|^2}{2\sigma^2}}_{(\mathrm{A})} \; + \; \underbrace{\frac{\| Z' \|^2 - \| Z + g(Y) \|^2}{2\sigma^2}}_{(\mathrm{B})},
\end{equation}
whose terms are $(\mathrm{A}) = \langle g(Y),\, 2Z + g(Y) \rangle/(2\sigma^2)$, the Gaussian ratio the exact-minimizer displacement alone produces, and $(\mathrm{B}) = -\langle \xi_\Delta,\, 2Z + 2g(Y) - \xi_\Delta \rangle/(2\sigma^2)$, which vanishes with the residual.

By \Cref{lem:gradient-subspace}, $g(Y)$ lies in $\operatorname{span}\{x_n,x_n'\}$. We therefore control $\langle Z,g(Y)\rangle$ through the two scalar projections $\langle Z,x_n\rangle$ and $\langle Z,x_n'\rangle$, without introducing an ambient-dimension factor.

\begin{proposition}[Bound on Term (A)]\label{prop:term-a}
Write $U := \langle Z, x_n \rangle/\sigma$ and $U' := \langle Z, x_n' \rangle/\sigma$ for the two scalars the differing samples read off the perturbation, and put
\begin{equation}\label{eq:T-def}
    \mathcal{T} := \frac{|U| + |U'|}{2\zeta}.
\end{equation}
Then \eqref{eq:term-a-bound} holds on $\{ Z \in \Omega \}$:
\begin{equation}\label{eq:term-a-bound}
    \frac{|\langle g(Y),\, 2Z + g(Y) \rangle|}{2\sigma^2} \leq \frac{G}{\sigma}\, \mathcal{T} + \frac{G^2}{2\sigma^2},
\end{equation}
and, unconditionally, $\mathcal{T}$ is the midpoint average of two folded Gaussians whose scales the features bound:
\begin{equation}\label{eq:T-avg}
    \mathcal{T} = \tfrac{1}{2}\Bigl( \frac{|U|}{\zeta} + \frac{|U'|}{\zeta} \Bigr), \qquad \operatorname{Var}\bigl( U/\zeta \bigr) = \frac{\| x_n \|^2}{\zeta^2} \leq 1, \qquad \operatorname{Var}\bigl( U'/\zeta \bigr) = \frac{\| x_n' \|^2}{\zeta^2} \leq 1,
\end{equation}
with no independence of $U$ and $U'$ assumed.
\end{proposition}

\begin{proof}
The inner product expands as $\langle g(Y), 2Z + g(Y) \rangle = 2\langle Z, g(Y) \rangle + \| g(Y) \|^2$. \Cref{lem:gradient-subspace} gives $g(Y)$ explicitly as $\tfrac{1}{n}[\ell'(\langle x_n, Y\rangle, y_n)\, x_n - \ell'(\langle x_n', Y\rangle, y_n')\, x_n']$, so only two scalars of the perturbation are read, and $|\ell'| \leq L_1$ (\Cref{ass:bounded-loss}) with $G = 2L_1\zeta/n$ gives
\begin{equation}\label{eq:term-a-CS}
    |\langle Z, g(Y) \rangle| \leq \frac{L_1}{n}\bigl( |\langle Z, x_n \rangle| + |\langle Z, x_n' \rangle| \bigr) = \frac{L_1\sigma}{n}\bigl( |U| + |U'| \bigr) = G\,\sigma\, \mathcal{T}.
\end{equation}
With $\| g(Y) \| \leq G$ (\Cref{cor:preimage}), $|\langle g(Y), 2Z + g(Y) \rangle| \leq 2G\sigma\mathcal{T} + G^2$; dividing by $2\sigma^2$ gives \eqref{eq:term-a-bound}. The two scalars are bounded separately rather than through the norm of the projection of $Z$ onto $V$ (\Cref{lem:gradient-subspace}), which leaves a single Gaussian degree of freedom in play; Cauchy--Schwarz against that norm would charge two.

For \eqref{eq:T-avg}, the first display is \eqref{eq:T-def} rearranged, and $U = \langle Z, x_n \rangle/\sigma$ is centred Gaussian with $\operatorname{Var} U = \| x_n \|^2 \leq \zeta^2$ (\Cref{ass:bounded-features}), likewise $U'$.
\end{proof}

We control Term~(B) using $\Delta_{\Omega'}:=2\mu\rho^N R_{\Omega'}$, which bounds $\|\xi_\Delta\|$ on $\{Z\in\Omega\}$ by \eqref{eq:xi-delta-bound}. We do not have a corresponding low-dimensional representation for $\xi_\Delta$, so we bound its inner product using $\|Z\|$ and the truncation radius.

Assembling the two terms gives the Gaussian-ratio bound, which holds at every point of the ball rather than with high probability. The projection factor $\mathcal{T}$ is carried forward as a random variable rather than replaced by a high-probability value, because \Cref{sec:dp-guarantee} integrates against its law, which for $G \leq \sigma$ costs less than charging its tail.

\begin{theorem}[Gaussian-ratio bound]\label{thm:gaussian-bound}
Under the hypotheses of \Cref{sec:gaussian-ratio}, let $Z \sim \mathcal{N}(0, \sigma^2 I_d)$ and let $\mathcal{T}$ be as in \eqref{eq:T-def}. Then, on the event $\mathcal{E}_1 := \{ \| Z \| \leq r \}$, the point $\Phi(Z)$ is defined and
\begin{equation}\label{eq:gaussian-main-bound}
    \left| \frac{\| \Phi(Z) \|^2 - \| Z \|^2}{2\sigma^2} \right| \leq \frac{G}{\sigma}\,\mathcal{T} + \frac{r \Delta_{\Omega'}}{\sigma^2} + \frac{(G + \Delta_{\Omega'})^2}{2\sigma^2}.
\end{equation}
\end{theorem}

\begin{proof}
On $\mathcal{E}_1$ we have $\| Z \| \leq r$, hence $Z \in \Omega$ (as $\{ z : \| z \| \leq r \} \subseteq \Omega$), so $\Phi(Z)$ is defined and the decomposition \eqref{eq:decomp} holds. By the triangle inequality,
\[
    \left| \frac{\| \Phi(Z) \|^2 - \| Z \|^2}{2\sigma^2} \right| = |\Lambda(Z)| \leq \frac{|\langle g(Y), 2Z + g(Y) \rangle|}{2\sigma^2} + \frac{|\langle \xi_\Delta, 2Z + 2g(Y) - \xi_\Delta \rangle|}{2\sigma^2}.
\]
Since $\mathcal{E}_1 \subseteq \{ Z \in \Omega \}$, \Cref{prop:term-a} bounds the first term by $(G/\sigma)\mathcal{T} + G^2/(2\sigma^2)$. For the second, \Cref{cor:preimage} gives $\| \xi_\Delta \| \leq \Delta_{\Omega'}$ (by \eqref{eq:xi-delta-bound}) and $\| g(Y) \| \leq G$, so $\| 2Z + 2g(Y) - \xi_\Delta \| \leq 2r + 2G + \Delta_{\Omega'}$ by the triangle inequality and $|\langle \xi_\Delta, 2Z + 2g(Y) - \xi_\Delta \rangle| \leq \Delta_{\Omega'}(2r + 2G + \Delta_{\Omega'})$ by Cauchy--Schwarz; dividing by $2\sigma^2$ and separating the term in $r$ bounds it by $r\Delta_{\Omega'}/\sigma^2 + (2G\Delta_{\Omega'} + \Delta_{\Omega'}^2)/(2\sigma^2)$. Their sum is $(G/\sigma)\mathcal{T} + r\Delta_{\Omega'}/\sigma^2 + (G^2 + 2G\Delta_{\Omega'} + \Delta_{\Omega'}^2)/(2\sigma^2)$, and the last numerator is $(G + \Delta_{\Omega'})^2$, which gives \eqref{eq:gaussian-main-bound}.
\end{proof}

\subsection{Log-determinant sensitivity}\label{sec:log-det}

We now bound the Jacobian ratio in \eqref{eq:privacy-decomp} for GLMs. For $T\in\{S,S'\}$ and perturbation $z_T$, write $\tilde J_N^T$ for the frozen-Hessian Jacobian and $\delta_N^T:=J_N^T(z_T)-\tilde J_N^T$. Using the geometric-sum form from \Cref{lem:closed-form}, we write
\begin{equation}\label{eq:jacobian-decomp}
    -J_N^T(z_T) = P_N^T - \delta_N^T, \qquad P_N^T := \eta\sum_{j=0}^{N-1}(A_T^*)^j = -\tilde{J}_N^T = p_N(H_T^*),
\end{equation}
where $p_N(h):=\eta\sum_{j=0}^{N-1}(1-\eta h)^j$. The frozen factor also has the form
$P_N^T=(H_T^*)^{-1}(I-(A_T^*)^N)$ and therefore converges to the exact-minimizer Jacobian $(H_T^*)^{-1}$. We bound the remaining error $\delta_N^T$ by $O(\tilde\nu_N\rho^{N-1})$ through \Cref{lem:error-nuclear}. The argument below compares the two frozen factors at the shared output and then adds the finite-$N$ correction.

Let $z_S$ and $z_{S'}$ satisfy $\mathcal{M}_S(z_S) = \mathcal{M}_{S'}(z_{S'}) =: w$. Each reference Hessian is read at its own dataset's $N$-th iterate, and those coincide: $w_N^S = \mathcal{M}_S(z_S) = w = \mathcal{M}_{S'}(z_{S'}) = w_N^{S'}$, so $H_S^* = \nabla^2 F_S(w)$ and $H_{S'}^* = \nabla^2 F_{S'}(w)$ are the two datasets' Hessians at one and the same point. The reference Hessians therefore differ only in the sample the datasets do not share, and the determinant comparison is a dataset swap at a fixed point of evaluation.

The two Jacobians are compared through the polynomial $p_N$ of \eqref{eq:jacobian-decomp}, so what governs the comparison is how $\log p_N$ responds to the curvature. Its logarithmic derivative factors as a bounded multiple of the logarithmic derivative of $h \mapsto 1/h$, the exact-minimizer case.

\begin{lemma}[Log-derivative of the frozen symbol]\label{lem:symbol}
Let $\phi_N := \log p_N$ on $[\mu, L]$, with $p_N$ as in \eqref{eq:jacobian-decomp}; for $N \geq 1$ it is well defined there, since $1 - \eta h \geq 0$ on $[\mu, L]$ makes every summand of $p_N(h) = \eta\sum_{j<N}(1-\eta h)^j$ nonnegative and its $j = 0$ term equal to $\eta$, so $p_N \geq \eta > 0$. Then $\phi_N'(h) = -g_N(h)/h$ with
\[
    0 \; \leq \; g_N(h) \; \leq \; 1 - (1 - \eta h)^{N-1} \qquad\text{for all } h \in [\mu, L],
\]
and consequently
\[
    \bar{g}_N := \sup_{[\mu,L]} g_N \; \leq \; 1 - (1 - \eta L)^{N-1} \; \leq \; 1.
\]
\end{lemma}

\begin{proof}
Write $a := 1 - \eta h$, so that $a \in [1 - \eta L, \rho] \subseteq [0, \rho]$ for $h \in [\mu, L]$, by $\eta L \leq 1$ (\Cref{subsec:mechanism}). Abbreviate $D(a) := \sum_{j=0}^{N-1}a^j$, so $p_N(h) = \eta D(a)$ and, by the chain rule in $h$, $p_N'(h) = -\eta^2 D'(a)$. Hence $\phi_N'(h) = p_N'(h)/p_N(h) = -\eta D'(a)/D(a)$, and since $\eta h = 1 - a$,
\[
    -h\,\phi_N'(h) = \frac{(1-a)\,D'(a)}{D(a)} =: g_N(h).
\]
Differentiating the identity $(1-a)D(a) = 1 - a^N$ gives $(1-a)D'(a) = D(a) - N a^{N-1}$, so
\begin{equation}\label{eq:g-ratio}
    g_N(h) = 1 - \frac{N a^{N-1}}{D(a)} = \frac{\sum_{j=0}^{N-1}\bigl( a^{j} - a^{N-1} \bigr)}{\sum_{j=0}^{N-1} a^{j}}.
\end{equation}
Every summand of the numerator is nonnegative, since $a \in [0,1]$ and $j \leq N-1$, and the denominator is positive; this gives $g_N \geq 0$. For the upper bound, $g_N \leq 1 - a^{N-1}$ rearranges to $a^{N-1}D(a) \leq N a^{N-1}$, which holds because $D(a) \leq N$, each of its $N$ summands being at most $1$. The map $a \mapsto 1 - a^{N-1}$ is nonincreasing on $[0,1]$, so over $h \in [\mu, L]$ it is largest at $a = 1 - \eta L$, giving the bound on $\bar{g}_N$, which is at most $1$ since $(1 - \eta L)^{N-1} \geq 0$.
\end{proof}

At a shared output, both Hessians are rank-one updates of the same matrix formed by the regularizer and the $n-1$ common samples. Comparing each update to this common matrix gives the determinant ratio in one step. Replacing $-\log h$ by $\phi_N(h)=\log p_N(h)$ introduces the factor $\bar g_N$ from \Cref{lem:symbol}.

\begin{lemma}[Reference determinant ratio]\label{lem:hessian-logdet}
Under \Cref{ass:bounded-features,ass:bounded-loss,ass:lipschitz-third}, let $z_S$ and $z_{S'}$ satisfy $\mathcal{M}_S(z_S) = \mathcal{M}_{S'}(z_{S'})$. Then, for every $N \geq 1$,
\begin{equation}\label{eq:hessian-logdet}
    \Bigl| \log \frac{\det P_N^S}{\det P_N^{S'}} \Bigr| \leq \bar{g}_N \log\Bigl( 1 + \frac{\Delta_H}{2n\mu} \Bigr),
\end{equation}
with $\bar{g}_N$ as in \Cref{lem:symbol}.
\end{lemma}

\begin{proof}
Since $P_N^T = p_N(H_T^*)$ and $\phi_N = \log p_N$, we have $\log\det P_N^T = \mathrm{tr}\,\phi_N(H_T^*)$, so
\begin{equation}\label{eq:logdet-split}
    \log\frac{\det P_N^S}{\det P_N^{S'}} = \mathrm{tr}\,\phi_N(H_S^*) - \mathrm{tr}\,\phi_N(H_{S'}^*),
\end{equation}
a dataset swap at the shared output $w$, both Hessians being read there. Every Hessian appearing here has spectrum in $[\mu, L]$ (\Cref{ass:regularity}), as does the anchor $H_0$ introduced next.

Since $S$ and $S'$ agree on the first $n-1$ samples and share the regularizer, both $H_S^*$ and $H_{S'}^*$---read at the same point $w$---are rank-one updates of the common matrix
\[
    H_0 := \frac{1}{n}\sum_{i=1}^{n-1} \ell''(\langle x_i, w \rangle, y_i)\, x_i x_i^\top + \mu I \succeq \mu I,
\]
the sum being positive semidefinite by $\ell'' \geq 0$ (\Cref{ass:bounded-loss}), so $H_0 \succeq \mu I$; and $H_0 \preceq L I$, since $\nabla^2 F_S(w) - H_0 = \tfrac{1}{n}\ell''(\langle x_n, w \rangle, y_n)\, x_n x_n^\top \succeq 0$, so $H_0 \preceq \nabla^2 F_S(w) \preceq L I$ by \Cref{ass:regularity}. Because both are read at $w$, the two datasets contribute the \emph{same} $H_0$; explicitly, $H_S^* = H_0 + u u^\top$ and $H_{S'}^* = H_0 + u' u'^\top$ with
\[
    u := \sqrt{\tfrac{1}{n}\ell''(\langle x_n, w \rangle, y_n)}\; x_n, \qquad u' := \sqrt{\tfrac{1}{n}\ell''(\langle x_n', w \rangle, y_n')}\; x_n',
\]
the square roots being real by $\ell'' \geq 0$. Apply item~\ref{itm:trace-rankone} of \Cref{lem:ext-trace-segment} with $\phi = \phi_N$, whose derivative has the required form $-g_N(h)/h$ with $0 \leq g_N \leq \bar{g}_N$ by \Cref{lem:symbol}; with $u^\top H_0^{-1}u \leq \| u \|^2/\mu \leq L_2\zeta^2/(n\mu) = \Delta_H/(2n\mu)$, using $0 \leq \ell'' \leq L_2$, $\| x_n \| \leq \zeta$ and $\Delta_H = 2L_2\zeta^2$,
\[
    -\bar{g}_N \log\Bigl( 1 + \frac{\Delta_H}{2n\mu} \Bigr) \; \leq \; \mathrm{tr}\,\phi_N(H_S^*) - \mathrm{tr}\,\phi_N(H_0) \; \leq \; 0,
\]
and the same for $H_{S'}^*$. The right-hand side of \eqref{eq:logdet-split} is the difference of two quantities lying in that one interval, hence is at most $\bar{g}_N \log(1 + \Delta_H/(2n\mu))$ in absolute value, which is \eqref{eq:hessian-logdet}.
\end{proof}

\Cref{lem:error-nuclear} bounds each error term in nuclear norm: writing
\begin{equation}\label{eq:beta-N}
    \beta_N := (C_S + C_{S'})\,\tilde{\nu}_N\rho^{N-1},
\end{equation}
we have $\| \delta_N^S \|_* + \| \delta_N^{S'} \|_* \leq \beta_N$.

One quantity organizes the final argument: feasibility of $N$, the single condition under which the analysis runs. We state it uniformly over a bounded set, so that it applies at every preimage in it.

\begin{definition}[Feasibility]\label{def:feasibility}
For a bounded set $D \subset \mathbb{R}^d$, write
\[
    R_D := \sup_{z \in D,\, T \in \{S, S'\}} R_T(z), \qquad C_D := \eta M R_D B, \qquad \alpha_D := \gamma_N - C_D \tilde{\nu}_N\rho^{N-1},
\]
extending the single-objective domain constant of \Cref{lem:traj-scale} to range over both datasets (at $D = \Omega'$ this is the $R_{\Omega'}$ of \Cref{sec:gaussian-ratio}); the constant of each $F_T$ alone is at most $R_D$, so the conclusions of \Cref{sec:diffeo} applied to $F_T$ on $D$ hold a fortiori with $C_D$ and $\alpha_D$ in place of their single-objective counterparts. We say $N$ is \emph{feasible} for $D$ if $\alpha_D > 0$; written out this reads $C_D \tilde{\nu}_N\rho^{N-1} < \gamma_N$. Since $\gamma_N \to 1/L > 0$ and $C_D \tilde{\nu}_N\rho^{N-1} \to 0$ as $N \to \infty$, feasibility holds for all sufficiently large $N$.
\end{definition}

The estimates assemble into the finite-$N$ correction that, together with the reference term $\tau^{(0)}$ of \Cref{lem:hessian-logdet}, enters the Jacobian-ratio bound of \Cref{thm:main}.

\begin{definition}[Finite-$N$ correction]\label{def:correction}
For $N$ feasible for a bounded set $\Omega'$ containing the preimages $z_S, z_{S'}$,
\begin{equation}\label{eq:epsilon-corr}
    \tau^{(\mathrm{corr})} := \frac{\beta_N}{\alpha_{\Omega'}},
\end{equation}
well-defined since feasibility gives $\alpha_{\Omega'} > 0$.
\end{definition}

The dimension dependence of $\beta_N$ and $\alpha_{\Omega'}$---carried by $R_{\Omega'}$ in both---is controlled in \Cref{thm:dp}. We can now bound the log-determinant sensitivity, with a leading term free of dimension.

\begin{restatable}[Dimension-free log-determinant sensitivity]{theorem}{thmmain}\label{thm:main}
Under \Cref{ass:bounded-features,ass:bounded-loss,ass:lipschitz-third}, with step size satisfying $0 < \eta \leq 1/L$ and $\eta\mu < 1$ (\Cref{subsec:mechanism}), let $\Omega'$ be a bounded convex set, and let $N \geq 1$ be feasible for $\Omega'$, that is $\alpha_{\Omega'} > 0$ (\Cref{def:feasibility}); the margin $\alpha_{\Omega'} > 0$ makes $\mathcal{M}_S$ and $\mathcal{M}_{S'}$ $C^1$-diffeomorphisms onto their images on $\Omega'$ (\Cref{lem:feasible-diffeo}). For an output $w$ with preimages $z_S = \mathcal{M}_S^{-1}(w)$ and $z_{S'} = \mathcal{M}_{S'}^{-1}(w)$ in $\Omega'$, set
\begin{equation}\label{eq:tau-zero}
    \tau^{(0)} := \bar{g}_N \log\Bigl( 1 + \frac{\Delta_H}{2n\mu} \Bigr), \qquad \tau := \tau^{(0)} + \tau^{(\mathrm{corr})}
\end{equation}
with $\tau^{(\mathrm{corr})}$ as in \Cref{def:correction}. Then
\begin{equation}\label{eq:main-bound}
    \Bigl| \log \frac{|\det J_N^S(z_S)|}{|\det J_N^{S'}(z_{S'})|} \Bigr| \leq \tau.
\end{equation}
\end{restatable}

\begin{proof}
\emph{Exact factorization.} By \eqref{eq:jacobian-decomp}, $-J_N^T = P_N^T - \delta_N^T = P_N^T(I - \Psi_T)$ with $\Psi_T := (P_N^T)^{-1}\delta_N^T$, for each $T \in \{S, S'\}$; the inverse exists because $P_N^T = -\tilde{J}_N^T$ is symmetric with $\sigma_{\min}(P_N^T) \geq \gamma_N > 0$ (item~\ref{itm:Jt-floor} of \Cref{lem:frozen-spectrum}). Its eigenvalues are the $p_N(h_i) \geq \eta > 0$ of \Cref{lem:symbol}, which the singular-value floor alone would not sign, so $\det P_N^T > 0$. The factors $(-1)^d$ in the two determinants cancel, and
\begin{equation}\label{eq:det-ratio}
    \frac{|\det J_N^S(z_S)|}{|\det J_N^{S'}(z_{S'})|} = \frac{\det P_N^S}{\det P_N^{S'}} \cdot \frac{|\det(I - \Psi_S)|}{|\det(I - \Psi_{S'})|}.
\end{equation}
We bound the frozen-Hessian ratio using \Cref{lem:hessian-logdet} and the second ratio using the finite-$N$ Jacobian error.

\emph{Size of the error matrices.} For $T \in \{S, S'\}$, by $\| (P_N^T)^{-1} \| \leq 1/\gamma_N$ and \Cref{prop:error-bound},
\[
    \| \Psi_T \| \leq \frac{\| \delta_N^T \|}{\gamma_N} \leq \frac{C_T \tilde{\nu}_N\rho^{N-1}}{\gamma_N} \leq \frac{C_{\Omega'} \tilde{\nu}_N\rho^{N-1}}{\gamma_N} =: \psi, \qquad \| \Psi_T \|_* \leq \frac{\| \delta_N^T \|_*}{\gamma_N},
\]
the nuclear bound by \Cref{lem:error-nuclear} and $\| X Y \|_* \leq \| X \|\,\| Y \|_*$ (\Cref{lem:ext-trace-nuclear}), and $C_T \leq C_{\Omega'}$ in both. Feasibility is exactly $C_{\Omega'} \tilde{\nu}_N\rho^{N-1} < \gamma_N$, so $\psi < 1$ and each $I - \Psi_T$ is invertible with $\| (I - \Psi_T)^{-1} \| \leq (1-\psi)^{-1}$ (\Cref{lem:ext-neumann} and summation of the geometric series). Moreover $\gamma_N(1 - \psi) = \gamma_N - C_{\Omega'}\tilde{\nu}_N\rho^{N-1} = \alpha_{\Omega'}$.

\emph{Reference factor.} \Cref{lem:hessian-logdet} bounds $\bigl| \log(\det P_N^S/\det P_N^{S'}) \bigr|$ by $\tau^{(0)}$.

\emph{Error factor.} Applying \eqref{eq:det-nuclear} of \Cref{lem:ext-logdet} to $K = -\Psi_T$ gives $\log|\det(I - \Psi_T)| \leq \| \Psi_T \|_*$, and applying it instead to $(I - \Psi_T)^{-1} - I = \Psi_T(I - \Psi_T)^{-1}$ gives
\[
    -\log|\det(I - \Psi_T)| = \log\bigl| \det\bigl( I + \Psi_T(I - \Psi_T)^{-1} \bigr) \bigr| \leq \| \Psi_T \|_*\,\| (I - \Psi_T)^{-1} \| \leq \frac{\| \Psi_T \|_*}{1 - \psi},
\]
so $\bigl| \log|\det(I - \Psi_T)| \bigr| \leq \| \Psi_T \|_*/(1-\psi)$ for each $T$. Summing the two and using the error bounds above,
\[
    \Bigl| \log \frac{|\det(I - \Psi_S)|}{|\det(I - \Psi_{S'})|} \Bigr| \leq \frac{\| \delta_N^S \|_* + \| \delta_N^{S'} \|_*}{\gamma_N(1 - \psi)} \leq \frac{\beta_N}{\alpha_{\Omega'}} = \tau^{(\mathrm{corr})},
\]
by \eqref{eq:beta-N} and $\gamma_N(1-\psi) = \alpha_{\Omega'}$.

\emph{Both directions.} Taking logarithms in \eqref{eq:det-ratio} and adding the two bounds gives
\[
\left| \log \frac{|\det J_N^S(z_S)|}{|\det J_N^{S'}(z_{S'})|} \right|
\leq \tau^{(0)} + \tau^{(\mathrm{corr})} = \tau,
\]
which is \eqref{eq:main-bound}.
\end{proof}

\subsection{The differential privacy guarantee}\label{sec:dp-guarantee}

We assemble the differential privacy guarantee by combining the Gaussian-ratio bound (\Cref{thm:gaussian-bound}) and the log-determinant bound (\Cref{thm:main}) with two further devices: a truncation confining the perturbation to a bounded domain, and a divergence calculation pricing the projection factor the Gaussian-ratio bound leaves random (\Cref{lem:hockey-chi}). We first record how the restricted output density enters the accounting.

\begin{remark}\label{rem:density-contribution}
For a bounded measurable set $\Omega$ contained in an open set on which $\mathcal{M}_S$ is a $C^1$-diffeomorphism, and $y \in \mathcal{M}_S(\Omega)$ with preimage $z_S = \mathcal{M}_S^{-1}(y) \in \Omega$, the change-of-variables formula (\Cref{lem:ext-pushforward}) gives the contribution to the output density from preimages in $\Omega$ as $p_Y^S(y; \Omega) = p_Z(z_S)/|\det J_N^S(z_S)|$: for every measurable $E$, the probability $\Pr[\mathcal{M}_S(Z) \in E,\, Z \in \Omega]$ is $p_Y^S(\cdot\,; \Omega)$ integrated over $E \cap \mathcal{M}_S(\Omega)$. The truncation in \Cref{thm:dp-full} works entirely with this restricted density.
\end{remark}

The margin $\alpha_D > 0$ of \Cref{def:feasibility} is certified on the domain itself, while \Cref{cor:diffeo} requires an \emph{open} domain; the next lemma bridges the two, and is what the phrase ``diffeomorphism on a bounded convex set'' of \Cref{sec:privacy} rests on.

\begin{lemma}[Feasibility yields a diffeomorphism]\label{lem:feasible-diffeo}
Let $D \subset \mathbb{R}^d$ be a bounded convex set with $\alpha_D > 0$ (\Cref{def:feasibility}). Then there is a bounded convex open set $\mathcal{B} \supseteq D$ on which $\mathcal{M}_S$ and $\mathcal{M}_{S'}$ are $C^1$-diffeomorphisms onto their images; in particular, both maps are diffeomorphisms on $D$ in the sense of \Cref{conv:closed-restriction}.
\end{lemma}

\begin{proof}
For $s > 0$ let $D_s := \{ z : \operatorname{dist}(z, D) < s \}$, the Minkowski sum of $D$ with the open ball of radius $s$---a bounded convex open set containing $D$. For each $T \in \{S, S'\}$ the map $z \mapsto R_T(z)$ is $(1/\mu)$-Lipschitz by \eqref{eq:R-def}, and every point of $D_s$ lies within distance $s$ of $D$, so $R_{D_s} \leq R_D + s/\mu$ and hence
\[
    \alpha_{D_s} \geq \alpha_D - \eta M B\, \tfrac{s}{\mu}\, \tilde{\nu}_N\rho^{N-1}.
\]
If $\tilde{\nu}_N\rho^{N-1} = 0$ any $s > 0$ serves; otherwise every $s < \alpha_D / \bigl( \eta M B \mu^{-1}\, \tilde{\nu}_N\rho^{N-1} \bigr)$ keeps $\alpha_{D_s} > 0$. Fix such an $s$ and set $\mathcal{B} := D_s$. For each $T$, the single-objective margin of $F_T$ on $\mathcal{B}$ is at least $\alpha_{\mathcal{B}} > 0$ (\Cref{def:feasibility}), so \Cref{cor:diffeo} applied to $F_T$ makes $\mathcal{M}_T$ a $C^1$-diffeomorphism of $\mathcal{B}$ onto its image; since $D \subseteq \mathcal{B}$, both maps restrict to $D$ from diffeomorphisms on the convex open neighborhood $\mathcal{B}$.
\end{proof}

We handle the two random contributions differently. We truncate the perturbation to $\Omega_0$ and add the probability of its complement directly to $\delta$. For the projection variable $\mathcal T$ in \eqref{eq:T-def}, we keep the pointwise privacy-loss bound and integrate it in the hockey-stick divergence. Its folded-Gaussian representation in \eqref{eq:T-avg} gives the closed-form bound in \Cref{lem:hockey-chi}.

\begin{lemma}[Divergence against a folded-Gaussian average]\label{lem:hockey-chi}
Let $\mathcal{T} = \tfrac{1}{2}(a|W| + a'|W'|)$ for standard Gaussians $W, W'$ under any joint law and constants $a, a' \in [0, 1]$, and let $\lambda > 0$ and $t \geq 0$. Writing $\varphi$ for the standard normal density, $\mathcal{Q}(x) := \Pr(\mathcal{N}(0,1) > x)$, and
\begin{equation}\label{eq:Q-one}
    \mathcal{Q}_1(t) := \int_t^\infty \mathcal{Q}(x)\, dx = \varphi(t) - t\,\mathcal{Q}(t)
\end{equation}
for the once-integrated tail,
\begin{equation}\label{eq:hockey-chi}
    \mathbb{E}\Bigl[ \bigl( 1 - e^{-\lambda(\mathcal{T} - t)} \bigr)_+ \Bigr] \;\leq\; 2\lambda\, \mathcal{Q}_1(t).
\end{equation}
The constant is attained: at $a = a' = 1$ with $W' = W$, the left side is $2\lambda\,\mathcal{Q}_1(t)\,(1 + o(1))$ as $\lambda \to 0$, at every $t \geq 0$.
\end{lemma}

\begin{proof}
For every real $v$ one has $(1 - e^{-v})_+ \leq v_+$: at $v \geq 0$ this is the elementary $1 - e^{-v} \leq v$, and at $v < 0$ both sides vanish. With $v = \lambda(\mathcal{T} - t)$, so that $v_+ = \lambda(\mathcal{T} - t)_+$ since $\lambda > 0$,
\begin{equation}\label{eq:hockey-linearized}
    \mathbb{E}\Bigl[ \bigl( 1 - e^{-\lambda(\mathcal{T} - t)} \bigr)_+ \Bigr] \;\leq\; \lambda\, \mathbb{E}\bigl[ (\mathcal{T} - t)_+ \bigr].
\end{equation}
Three pointwise facts bound the expectation on the right. Convexity of $x \mapsto (x - t)_+$ splits the average, $(\mathcal{T} - t)_+ \leq \tfrac{1}{2}(a|W| - t)_+ + \tfrac{1}{2}(a'|W'| - t)_+$. Next, $(a|W| - t)_+ \leq a\,(|W| - t)_+$ for $a \in [0, 1]$ and $t \geq 0$: if $|W| \leq t$ both sides vanish, since $a|W| \leq at \leq t$; if $|W| > t$ then $a|W| - t \leq a(|W| - t)$, and the right side is nonnegative. Finally, for a standard Gaussian the layer cake evaluates exactly,
\[
    \mathbb{E}\bigl[ (|W| - t)_+ \bigr] = \int_t^\infty \Pr(|W| > x)\, dx = 2\int_t^\infty \mathcal{Q}(x)\, dx = 2\,\mathcal{Q}_1(t),
\]
the closed form in \eqref{eq:Q-one} following from $\tfrac{d}{da}[\varphi(a) - a\mathcal{Q}(a)] = -a\varphi(a) - \mathcal{Q}(a) + a\varphi(a) = -\mathcal{Q}(a)$ and both sides vanishing at $a = \infty$. Taking expectations and assembling,
\[
    \mathbb{E}\bigl[ (\mathcal{T} - t)_+ \bigr] \;\leq\; \frac{a + a'}{2}\cdot 2\,\mathcal{Q}_1(t) \;\leq\; 2\,\mathcal{Q}_1(t),
\]
which against \eqref{eq:hockey-linearized} gives \eqref{eq:hockey-chi}. No step reads the joint law: the first two facts are pointwise and the third is a computation on one marginal. For the attainment, at $a = a' = 1$ and $W' = W$ both pointwise inequalities are equalities and $\mathbb{E}[(\mathcal{T} - t)_+] = 2\,\mathcal{Q}_1(t)$ exactly; only the linearization $(1 - e^{-v})_+ \leq v_+$ then separates the two sides of \eqref{eq:hockey-chi}, and it is an equality to leading order as $\lambda \to 0$.
\end{proof}

\begin{theorem}[Differential privacy: exact finite-$N$ form]\label{thm:dp-full}
Let \Cref{ass:bounded-features,ass:bounded-loss,ass:lipschitz-third} hold. Fix a radius $r > 0$, set $\Omega_0 := \{ z : \| z \| \leq r \}$, and let $\Omega' := \{ z : \| z \| \leq r' \}$ for a radius $r'$ satisfying, with the uniform trajectory scale $\bar R := (r' + L_1\zeta)/\mu$, the fixed-point inequality
\begin{equation}\label{eq:domain-fixed-point}
    r' \geq r + \frac{G(1 - \rho^N)}{\mu\, \bar\alpha}, \qquad \bar\alpha := \gamma_N - \eta M \bar R B\, \tilde{\nu}_N\rho^{N-1} > 0.
\end{equation}
Such $r'$ exists for all sufficiently large $N$ (\Cref{rem:domain-existence}). Write
\begin{equation}\label{eq:dp-full-constants}
    \Delta_\xi := 2\mu\rho^N \bar R, \qquad \tau := \tau^{(0)} + \frac{2\eta M \bar R B\, \tilde{\nu}_N\rho^{N-1}}{\bar\alpha},
\end{equation}
with $\tau^{(0)}$ as in \eqref{eq:tau-zero}. Then for every $t \geq 0$ the mechanism $S \mapsto \mathcal{M}_S(Z)$ with $Z \sim \mathcal{N}(0, \sigma^2 I_d)$, $\sigma > 0$, is $(\epsilon_{\mathrm{total}}, \delta_{\mathrm{total}})$-differentially private, where
\begin{equation}\label{eq:eps-total}
    \epsilon_{\mathrm{total}} = \frac{G}{\sigma}\,t + \frac{r \Delta_\xi}{\sigma^2} + \frac{(G + \Delta_\xi)^2}{2\sigma^2} + \tau,
\end{equation}
\begin{equation}\label{eq:delta-total}
    \delta_{\mathrm{total}} = \Pr\bigl( \| Z \| > r \bigr) + \frac{2G}{\sigma}\,\mathcal{Q}_1(t).
\end{equation}
\end{theorem}

\begin{proof}
Let $(S, S')$ be an ordered pair of neighboring datasets, let $Z \sim \mathcal{N}(0, \sigma^2 I_d)$, and let $\mathcal{O} \subseteq \mathbb{R}^d$ be a measurable set; we show $\Pr[\mathcal{M}_S(Z) \in \mathcal{O}] \leq e^{\epsilon_{\mathrm{total}}}\Pr[\mathcal{M}_{S'}(Z) \in \mathcal{O}] + \delta_{\mathrm{total}}$.

\emph{Domain construction.} Since $w_0 = 0$ and $\| \nabla F_T(0) \| \leq L_1\zeta$ for every dataset (\Cref{subsec:notation}), item~\ref{itm:R-affine} of \Cref{lem:traj-scale} gives $R_T(z) \leq (\| z \| + L_1\zeta)/\mu \leq \bar R$ for every $z \in \Omega'$ and $T \in \{ S, S' \}$. Hence $R_{\Omega'} \leq \bar R$, and $\alpha_{\Omega'} \geq \bar\alpha > 0$ makes $N$ feasible for $\Omega'$ (\Cref{def:feasibility}), so \Cref{lem:feasible-diffeo} furnishes a bounded convex open set $\mathcal{B} \supseteq \Omega'$ on which $\mathcal{M}_S$ and $\mathcal{M}_{S'}$ are $C^1$-diffeomorphisms onto their images; every preimage below lies in $\Omega' \subset \mathcal{B}$, where the diffeomorphisms and the change-of-variables formula are valid.

\emph{Preimage existence.} Fix $z_S \in \Omega_0$ and set $y := \mathcal{M}_S(z_S)$. Consider the path $y_v := (1 - v)\mathcal{M}_{S'}(z_S) + v y$ for $v \in [0,1]$, and let $\mathcal{I}$ be the set of $v \in [0,1]$ such that $y_{v'} \in \mathcal{M}_{S'}(\mathcal{B})$ for every $v' \in [0,v]$ and the lift $z_{v'} := \mathcal{M}_{S'}^{-1}(y_{v'})$ obeys $\| z_{v'} \| \leq r + v'(r' - r)$ on $[0,v]$. The bound keeps each $z_{v'}$ in $\Omega'$, where the lift moves at controlled speed: $\dot z_{v'} = (J_N^{S'}(z_{v'}))^{-1}(y - \mathcal{M}_{S'}(z_S))$ by the chain rule, so \Cref{thm:neg-def-J} ($\| (J_N^{S'})^{-1} \| \leq 1/\alpha_{\Omega'}$ on $\Omega'$) and \Cref{lem:map-sensitivity} ($\| y - \mathcal{M}_{S'}(z_S) \| = \| \mathcal{M}_S(z_S) - \mathcal{M}_{S'}(z_S) \| \leq G(1 - \rho^N)/\mu$) give
\[
    \| \dot z_{v'} \| \leq \frac{G(1 - \rho^N)}{\mu\, \alpha_{\Omega'}} \leq \frac{G(1 - \rho^N)}{\mu\, \bar\alpha} \leq r' - r,
\]
the middle step by $\alpha_{\Omega'} \geq \bar\alpha$ and the last by \eqref{eq:domain-fixed-point}. Since $[0,1]$ is connected, it suffices to check that $\mathcal{I}$ is nonempty, relatively open, and relatively closed; then $\mathcal{I} = [0,1]$. \emph{Nonempty:} $0 \in \mathcal{I}$, as $y_0 = \mathcal{M}_{S'}(z_S)$ has preimage $z_S$ with $\| z_S \| \leq r$. \emph{Open:} for $v \in \mathcal{I}$ with $v < 1$ the bound at $v$ is strict, $\| z_v \| < r'$. Since $y_v = \mathcal{M}_{S'}(z_v)$ lies in the open set $\mathcal{M}_{S'}(\mathcal{B})$, continuity of the path extends the lift to $[0, v + \varepsilon)$; shrinking $\varepsilon$ keeps $\| z_{v'} \| < r'$, hence $z_{v'} \in \Omega'$, and integrating the speed bound from $0$ restores $\| z_{v'} \| \leq \| z_S \| + v'(r' - r) \leq r + v'(r' - r)$ there. So $[0, v + \varepsilon) \subseteq \mathcal{I}$. \emph{Closed:} as $v \uparrow v^* := \sup \mathcal{I}$, the speed bound makes $v \mapsto z_v$ Lipschitz, so $z_v$ converges to some $z^*$ in the compact set $\Omega'$; continuity gives $\mathcal{M}_{S'}(z^*) = y_{v^*}$ and $\| z^* \| \leq r + v^*(r' - r)$, so $v^* \in \mathcal{I}$. Taking $v = 1$ and integrating the speed bound once more yields $z_{S'} := \mathcal{M}_{S'}^{-1}(y) \in \Omega'$ with $\| z_{S'} - z_S \| \leq G(1 - \rho^N)/(\mu\, \alpha_{\Omega'})$.

\emph{Density ratio.} For $z_S \in \Omega_0$, set $y := \mathcal{M}_S(z_S)$ and let $z_{S'} \in \Omega'$ be its $\mathcal{M}_{S'}$-preimage. The restricted privacy loss \eqref{eq:privacy-decomp} splits into the Gaussian and Jacobian ratios. We instantiate the hypotheses of \Cref{sec:gaussian-ratio} with the pair $(\Omega_0, \Omega')$, scale $R_{\Omega'}$, and radius $r$: both maps are $C^1$-diffeomorphisms on $\Omega_0$, and $\mathcal{M}_{S'}$ is one on $\Omega'$ (each restricting from $\mathcal{B}$; \Cref{conv:closed-restriction}); $\Omega_0 \subseteq \Omega'$ since $r \leq r'$ by \eqref{eq:domain-fixed-point}; the preimage-existence step shows $\mathcal{M}_S(\Omega_0) \subseteq \mathcal{M}_{S'}(\Omega')$---so the conjugacy map $\Phi : \Omega_0 \to \Omega'$ is well-defined, with $\Phi(z_S) = z_{S'}$---and the closed ball of radius $r$ is $\Omega_0$ itself. Under this instantiation $\mathcal{E}_1 = \{ \| z \| \leq r \} = \{ z \in \Omega_0 \}$, and \Cref{thm:gaussian-bound}---whose bound \eqref{eq:gaussian-main-bound} holds pointwise on $\mathcal{E}_1$ viewed as a deterministic set---bounds the Gaussian ratio in absolute value by $(G/\sigma)\mathcal{T}(z_S) + r\Delta_\xi/\sigma^2 + (G + \Delta_\xi)^2/(2\sigma^2)$, with $\mathcal{T}(z) = \bigl( |\langle z, x_n \rangle| + |\langle z, x_n' \rangle| \bigr)/(2\zeta\sigma)$ as in \eqref{eq:T-def}: its residual scale is $\Delta_{\Omega'} = 2\mu\rho^N R_{\Omega'} \leq \Delta_\xi$, and both terms carrying it are nondecreasing in it. By \Cref{thm:main} (feasibility on $\Omega'$), the Jacobian ratio is at most $\tau^{(0)} + \beta_N/\alpha_{\Omega'}$ in absolute value; here $R^S, R^{S'} \leq R_{\Omega'} \leq \bar R$ gives $\beta_N \leq 2\eta M \bar R B\, \tilde{\nu}_N\rho^{N-1}$, and $\alpha_{\Omega'} \geq \bar\alpha$, so the ratio is at most $\tau$. Adding the two and writing
\begin{equation}\label{eq:shift-def}
    s := \frac{r \Delta_\xi}{\sigma^2} + \frac{(G + \Delta_\xi)^2}{2\sigma^2} + \tau, \qquad\text{so that}\qquad \epsilon_{\mathrm{total}} = \frac{G}{\sigma}\,t + s,
\end{equation}
we obtain, for every $z_S \in \Omega_0$,
\begin{equation}\label{eq:total-ratio-bound}
    \log\frac{p_Y^S(y; \Omega')}{p_Y^{S'}(y; \Omega')} \leq s + \frac{G}{\sigma}\,\mathcal{T}(z_S).
\end{equation}

\emph{Assembly.} Since $\Pr(Z \notin \Omega_0) = \Pr(\| Z \| > r)$,
\[
    \Pr[\mathcal{M}_S(Z) \in \mathcal{O}] \leq \Pr[\mathcal{M}_S(Z) \in \mathcal{O},\, Z \in \Omega_0] + \Pr(\| Z \| > r),
\]
and the change of variables (\Cref{rem:density-contribution}) evaluates the first term as $\int_{\mathcal{O} \cap \mathcal{M}_S(\Omega_0)} p_Y^S(y; \Omega_0)\, dy$, the preimage of each such $y$ lying in $\Omega_0 \subseteq \Omega'$ so that $p_Y^S(y; \Omega_0) = p_Y^S(y; \Omega')$ there. By the preimage-existence step $\mathcal{M}_S(\Omega_0) \subseteq \mathcal{M}_{S'}(\Omega')$, so $p_Y^{S'}(\cdot\,; \Omega')$ is defined on that set and integrates over $\mathcal{O} \cap \mathcal{M}_S(\Omega_0)$ to at most $\Pr[\mathcal{M}_{S'}(Z) \in \mathcal{O}]$. 

Add and subtract $e^{\epsilon_{\mathrm{total}}} p_Y^{S'}$ and discard the constraint $y \in \mathcal{O}$ from the nonnegative remainder:
\[
    \int_{\mathcal{O} \cap \mathcal{M}_S(\Omega_0)} \!\!\! p_Y^S(y; \Omega')\, dy \;\leq\; e^{\epsilon_{\mathrm{total}}}\Pr[\mathcal{M}_{S'}(Z) \in \mathcal{O}] + \int_{\mathcal{M}_S(\Omega_0)} \!\!\! \bigl[ p_Y^S(y; \Omega') - e^{\epsilon_{\mathrm{total}}} p_Y^{S'}(y; \Omega') \bigr]_+ dy.
\]
Factoring $p_Y^S$ out of the last integrand and changing variables back,
\[
    \int_{\mathcal{M}_S(\Omega_0)} \bigl[ p_Y^S - e^{\epsilon_{\mathrm{total}}} p_Y^{S'} \bigr]_+ dy = \mathbb{E}\Bigl[ \mathbf{1}\{ Z \in \Omega_0 \}\bigl( 1 - e^{\epsilon_{\mathrm{total}} - \Lambda_{\mathrm{tot}}(Z)} \bigr)_+ \Bigr],
\]
where $\Lambda_{\mathrm{tot}}(z) := \log\bigl[ p_Y^S(\mathcal{M}_S(z); \Omega')/p_Y^{S'}(\mathcal{M}_S(z); \Omega') \bigr]$ is the restricted privacy loss---the sum of the two ratios of \eqref{eq:privacy-decomp}, not the Gaussian ratio $\Lambda$ of \eqref{eq:decomp} alone. The map $x \mapsto (1 - e^{\epsilon_{\mathrm{total}} - x})_+$ is nondecreasing, so \eqref{eq:total-ratio-bound} and \eqref{eq:shift-def} bound the integrand on $\{ Z \in \Omega_0 \}$ by $\bigl( 1 - e^{-(G/\sigma)(\mathcal{T} - t)} \bigr)_+$; dropping the indicator, which only enlarges a nonnegative integrand, gives
\[
    \int_{\mathcal{M}_S(\Omega_0)} \bigl[ p_Y^S - e^{\epsilon_{\mathrm{total}}} p_Y^{S'} \bigr]_+ dy \;\leq\; \mathbb{E}\Bigl[ \bigl( 1 - e^{-(G/\sigma)(\mathcal{T} - t)} \bigr)_+ \Bigr] \;\leq\; \frac{2G}{\sigma}\,\mathcal{Q}_1(t),
\]
the second inequality by \Cref{lem:hockey-chi} at $\lambda = G/\sigma$, whose hypothesis on $\mathcal{T}$ is the average form \eqref{eq:T-avg}, with scales $\| x_n \|/\zeta$ and $\| x_n' \|/\zeta$ at most one. Together with the truncation split, these give $\Pr[\mathcal{M}_S(Z) \in \mathcal{O}] \leq e^{\epsilon_{\mathrm{total}}}\Pr[\mathcal{M}_{S'}(Z) \in \mathcal{O}] + \Pr(\| Z \| > r) + (2G/\sigma)\mathcal{Q}_1(t)$, and the last two terms are $\delta_{\mathrm{total}}$ by \eqref{eq:delta-total}. The two summands of $\delta_{\mathrm{total}}$ are charged for different reasons. The first is mass on which no ratio bound is available at all; the second is the divergence contributed by outputs whose ratio bound exceeds $e^{\epsilon_{\mathrm{total}}}$.

\emph{Uniformity.} The pair was arbitrary, and \eqref{eq:eps-total} and \eqref{eq:delta-total} name only $r$, $r'$, $t$, $N$ and constants common to all datasets, so one $(\epsilon_{\mathrm{total}}, \delta_{\mathrm{total}})$ serves every ordered pair, as the definition of differential privacy requires.
\end{proof}

\begin{remark}[Existence of the domain radius]\label{rem:domain-existence}
Given the radius $r$ of \Cref{thm:dp-full}, a radius $r'$ satisfying \eqref{eq:domain-fixed-point} exists for all sufficiently large $N$. Throughout, $N \geq 1$, so $\gamma_N > 0$. Write $\bar\alpha(\tilde r) := \gamma_N - \eta M \bigl( (\tilde r + L_1\zeta)/\mu \bigr) B\, \tilde{\nu}_N\rho^{N-1}$, the value of $\bar\alpha$ at $r' = \tilde r$, and $\phi(\tilde r) := r + G(1 - \rho^N)/(\mu\, \bar\alpha(\tilde r))$. Put $\tilde r_{\max} := r + 2G/(\mu\gamma_N)$ and $C_{\max} := \eta M \bigl( (\tilde r_{\max} + L_1\zeta)/\mu \bigr) B$. When
\begin{equation}\label{eq:N-condition}
    \tilde{\nu}_N\rho^{N-1} < \frac{\gamma_N}{2 C_{\max}}
\end{equation}
(which holds for all large $N$, as $\tilde{\nu}_N\rho^{N-1} \leq N\rho^{N-1} \to 0$ and $\gamma_N \to 1/L$), every $\tilde r \in [r, \tilde r_{\max}]$ has $\bar\alpha(\tilde r) \geq \gamma_N - C_{\max} \tilde{\nu}_N\rho^{N-1} > \gamma_N/2 > 0$, so $\phi$ is continuous on $[r, \tilde r_{\max}]$ with $\phi(r) > r$ and $\phi(\tilde r_{\max}) \leq r + 2G(1-\rho^N)/(\mu\gamma_N) < \tilde r_{\max}$. By the intermediate value theorem $\phi$ has a fixed point $r' \in (r, \tilde r_{\max})$, which satisfies \eqref{eq:domain-fixed-point}.
\end{remark}

The bound \eqref{eq:eps-total} is exact but $d$-dependent, through $r$ and through $\Delta_\xi$ and $\tau$ (both via $\bar R$); the corrections in $\tau$ involve the factors $N\rho^{N-1}$ and $\tilde{\nu}_N\rho^{N-1} \leq N\rho^{N-1}$, which the following tail bound controls, and which the explicit threshold \eqref{eq:N-dagger} then meets by spending one $\rho^{N/2}$ on the coefficients. Applied in full, the bound absorbs the factor $N$ and leaves $\rho^{N/2}$; applied after that spending, it leaves $N\rho^{N/2}$.

\begin{lemma}[Tail bound]\label{lem:tail}
For $\rho = 1 - \eta\mu \in (0,1)$ (the contraction factor of \Cref{subsec:mechanism}) and every integer $N \geq 0$,
\[
    \frac{e\,\eta\mu}{2}\,N\rho^{N/2} \leq 1, \qquad\text{equivalently}\qquad N\rho^{N-1} \leq \frac{2}{e\,\rho\,\eta\mu}\,\rho^{N/2}.
\]
\end{lemma}

\begin{proof}
For $s \in (0,1)$ the function $x \mapsto x s^x$ on $[0,\infty)$ is maximized at $x = 1/\log(1/s)$, where it equals $\frac{1}{e\log(1/s)}$; hence $x s^x \leq \frac{1}{e\log(1/s)}$ for all $x \geq 0$. Taking $s = \rho^{1/2}$ and $x = N$ gives $N\rho^{N/2} \leq \frac{2}{e\log(1/\rho)}$, and $\log(1/\rho) \geq \eta\mu$---from $\rho = 1 - \eta\mu$ and $-\log(1-v) \geq v$ for $v \in (0,1)$---gives the first form. Multiplying it by $\rho^{N/2 - 1}$ gives the second.
\end{proof}

\thmdp*

\begin{proof}
\emph{Reduction to the feature span.} By \Cref{lem:data-span}, the two mechanisms differ only in their $\mathcal{C}$ blocks $\mathcal{M}_S^{\mathcal{C}}, \mathcal{M}_{S'}^{\mathcal{C}}$, which form a mechanism pair of \Cref{subsec:general-setting} within $\mathcal{C}$ carrying the same $\mu, L, M, G$; their $\mathcal{C}^\perp$ blocks are one shared affine map, which cancels from both ratios of \eqref{eq:privacy-decomp} as recorded after that display. Since $Z \sim \mathcal{N}(0, \sigma^2 I_d)$ is isotropic, $\Pi_{\mathcal{C}} Z \sim \mathcal{N}(0, \sigma^2 I_{d_{\mathcal{C}}})$ independently of $\Pi_{\mathcal{C}^\perp} Z$, so each output law is a product whose second factor, the law of $-B_N \Pi_{\mathcal{C}^\perp} Z$, is the same under both datasets. For Borel $\mathcal{O} \subseteq \mathbb{R}^d$ and each $y_{\mathcal{C}^\perp}$, applying the $(\epsilon, \delta)$ inequality for $\mathcal{M}_S^{\mathcal{C}}, \mathcal{M}_{S'}^{\mathcal{C}}$ to the slice $\{ y_{\mathcal{C}} \in \mathcal{C} : y_{\mathcal{C}} + y_{\mathcal{C}^\perp} \in \mathcal{O} \}$ and integrating against that common factor gives the same inequality for $\mathcal{M}_S, \mathcal{M}_{S'}$ at $\mathcal{O}$: each output law is the product of its two block laws, the slices of a Borel $\mathcal{O}$ are Borel \citep[Proposition~2.34]{folland-real}, the slice probabilities integrate to the total by Tonelli's theorem \citep[Theorem~2.37]{folland-real}, and the $\delta$-term integrates to $\delta$ because the common factor is a probability measure. We argue within $\mathcal{C}$ for the rest of the proof, so that \Cref{thm:dp-full} and everything it rests on is applied with $d_{\mathcal{C}}$ in place of $d$, and write $z$ for $z_{\mathcal{C}}$.

Set $\Omega' := \{ \| z \| \leq r' \}$ with $r' = r_u + 2LG/\mu$ as in the statement, so the $\bar R$ fixed there is the uniform trajectory scale of \Cref{thm:dp-full} at this $r'$.

\emph{Truncation mass.} Gaussian norm concentration (\Cref{cor:ext-gauss-norm}), applied in $\mathcal{C}$ to $\Pi_{\mathcal{C}} Z \sim \mathcal{N}(0, \sigma^2 I_{d_{\mathcal{C}}})$, gives $\Pr(\| \Pi_{\mathcal{C}} Z \| \geq \sigma(\sqrt{d_{\mathcal{C}}} + v)) \leq e^{-v^2/2}$ for every $v \geq 0$. Since $\dim \mathcal{C} = d_{\mathcal{C}} \leq d_\star$ by \eqref{eq:data-span}, the radius fixed in the statement obeys $r_u = \sigma(\sqrt{d_\star} + u) \geq \sigma(\sqrt{d_{\mathcal{C}}} + u)$, so at $v = u$ this gives $\Pr(\| \Pi_{\mathcal{C}} Z \| > r_u) \leq e^{-u^2/2}$, the first term of \eqref{eq:delta-profile}. Only $d_\star$ enters $r_u$, and it is fixed by $d$ and $n$ before the neighboring pair is named.

We verify the remaining hypotheses from the three entries of \eqref{eq:N-dagger}: as $\rho = 1 - \eta\mu \leq e^{-\eta\mu}$, the condition gives $\rho^{N/2} \leq e^{-N\eta\mu/2} \leq 1/X$ for each entry $X$ of the maximum. The third entry gives $\rho^{N/2} \leq \tfrac14$; squaring, $\rho^N \leq \frac{1}{16}$, and $0 \leq 1 - \eta L \leq \rho$ gives $(1 - \eta L)^N \leq \rho^N \leq \frac{1}{16}$, so that $\gamma_N = \frac{1 - (1 - \eta L)^N}{L} \geq \frac{15}{16L} \geq \frac{12}{16L} = \frac{3}{4L}$. The second entry gives
\begin{equation}\label{eq:entry2-sharp}
    \rho^{N/2} \leq \frac{e\rho\mu\,\min\{1, MG/L\}}{8LMB \bar R}.
\end{equation}
By \eqref{eq:nu-tilde}, $\tilde{\nu}_N \leq \nu_N \leq N$, so with $N\rho^{N-1} = \rho^{-1}\rho^{N/2}\bigl( N\rho^{N/2} \bigr)$ and \eqref{eq:entry2-sharp} applied to the isolated $\rho^{N/2}$,
\[
    \bar R\, \tilde{\nu}_N\rho^{N-1} \leq \bar R\, N\rho^{N-1} \leq \frac{\min\{1, MG/L\}}{8LMB\eta}\cdot e\,\eta\mu\, N\rho^{N/2} \leq \frac{\min\{1, MG/L\}}{4LMB\eta},
\]
the last step by \Cref{lem:tail}. Multiplying by $\eta M B$ and using $\min\{1, MG/L\} \leq 1$,
\begin{equation}\label{eq:C-sharp}
    \eta M \bar R B\, \tilde{\nu}_N\rho^{N-1} \leq \frac{\min\{1, MG/L\}}{4L} \leq \frac{1}{4L}.
\end{equation}
With $\gamma_N \geq \frac{3}{4L}$, the singular-value margin is $\bar\alpha = \gamma_N - \eta M \bar R B\, \tilde{\nu}_N\rho^{N-1} \geq \frac{3}{4L} - \frac{1}{4L} = \frac{1}{2L} > 0$, and the fixed-point inequality \eqref{eq:domain-fixed-point} follows from $1 - \rho^N \leq 1$:
\[
    r_u + \frac{G(1-\rho^N)}{\mu\,\bar\alpha} \leq r_u + \frac{2LG}{\mu} = r'.
\]

\emph{Gaussian terms.} The first entry of \eqref{eq:N-dagger} gives $\rho^{N/2} \leq G\sigma/(2\mu\bar R r')$, so
\[
    \Delta_\xi = 2\mu\rho^N \bar R = 2\mu\bar R\, \rho^{N/2}\cdot \rho^{N/2} \leq \frac{G\sigma}{r'}\,\rho^{N/2}.
\]
Since $r_u \leq r'$, the second term of \eqref{eq:eps-total} obeys $r_u \Delta_\xi/\sigma^2 \leq (G/\sigma)\rho^{N/2}$; and $r' \geq r_u \geq \sigma$ gives $\Delta_\xi \leq G\rho^{N/2}$, so the third obeys $(G + \Delta_\xi)^2/(2\sigma^2) \leq G^2(1 + \rho^{N/2})^2/(2\sigma^2)$.

For the Jacobian ratio, \Cref{lem:symbol,lem:hessian-logdet} give
$\tau^{(0)}\leq\log(1+\Delta_H/(2n\mu))$ because both reference Hessians are evaluated at the shared output. For the finite-$N$ correction, the estimate behind \eqref{eq:C-sharp}, retaining the factor $N\rho^{N/2}$, gives
\[
    2\eta M \bar R B\, \tilde{\nu}_N\rho^{N-1} \; \leq \; \frac{\min\{1, MG/L\}}{4L}\,e\eta\mu\, N\rho^{N/2};
\]
so with $\bar\alpha \geq \frac{1}{2L}$ the correction term of \eqref{eq:dp-full-constants} is at most
\[
    2L \cdot \frac{\min\{1, MG/L\}}{4L}\,e\eta\mu\, N\rho^{N/2} \; = \; \min\{1, MG/L\}\,\frac{e\eta\mu}{2}\,N\rho^{N/2}.
\]
Collecting the two,
\[
    \tau \leq \log\Bigl( 1 + \frac{\Delta_H}{2n\mu} \Bigr) + \min\{1, MG/L\}\,\frac{e\eta\mu}{2}\,N\rho^{N/2} = \tau^\star.
\]
Now fix $\epsilon \geq \epsilon_0$ and apply \Cref{thm:dp-full} at the radius $r = r_u$ and the threshold $t := (\sigma/G)(\epsilon - \epsilon_0) \geq 0$. The released map $z \mapsto w_N$ depends on neither, so a different $t$ may be taken for each $\epsilon$. Its budget \eqref{eq:eps-total} has first term $(G/\sigma)t = \epsilon - \epsilon_0$, while its remaining three are majorized by the three terms of \eqref{eq:eps-zero}---the first two by the Gaussian-terms step above, the third by $\tau \leq \tau^\star$---so $\epsilon_{\mathrm{total}} \leq \epsilon$. Its additive parameter \eqref{eq:delta-total} is $\Pr(\| \Pi_{\mathcal{C}} Z \| > r_u) + (2G/\sigma)\mathcal{Q}_1(t) \leq \delta(\epsilon)$ by the truncation-mass step and \eqref{eq:delta-profile}. A mechanism that is $(\epsilon_{\mathrm{total}}, \delta_{\mathrm{total}})$-private with $\epsilon_{\mathrm{total}} \leq \epsilon$ and $\delta_{\mathrm{total}} \leq \delta(\epsilon)$ is a fortiori $(\epsilon, \delta(\epsilon))$-private.
\end{proof}

\begin{remark}[Oracle limit]\label{rem:asymptotic}
As $N\to\infty$, $\Delta_\xi\to0$ and $\tau^{(\mathrm{corr})}\to0$. For every fixed $h\in[\mu,L]$, \eqref{eq:g-ratio} gives $g_N(h)\to1$, and therefore $\bar g_N\to1$. Hence the nontruncation part of \eqref{eq:eps-total}, at any fixed threshold $t$, converges to
\[
    \frac{G}{\sigma}t+\frac{G^2}{2\sigma^2}+\log\Bigl(1+\frac{\Delta_H}{2n\mu}\Bigr).
\]
To recover the full exact-minimizer certificate, choose the free truncation parameter to grow with $N$; for example, let $u_N=N$. Then $r_{u_N}$, $r'$, and $\bar R$ grow at most linearly in $N$, so the right-hand side of \eqref{eq:N-dagger} is $O((\eta\mu)^{-1}\log N)$ and the condition holds for all sufficiently large $N$. At the same time, $e^{-u_N^2/2}\to0$, $\rho^{N/2}\to0$, and $N\rho^{N/2}\to0$. Thus, for every fixed $\epsilon>\epsilon_\infty$, where
\[
    \epsilon_\infty:=\frac{G^2}{2\sigma^2}+\log\Bigl(1+\frac{\Delta_H}{2n\mu}\Bigr),
\]
our bound converges to the exact-minimizer certificate obtained from the same Gaussian-ratio and common-anchor determinant bounds. When $\Delta_H/(n\mu)\ll1$, the determinant term is $\Delta_H/(2n\mu)+O((\Delta_H/(n\mu))^2)$; \Cref{rem:oracle-recovery} compares it with the classical determinant bound.
\end{remark}

\subsection{Evaluation of the finite-iteration certificates}\label{subsec:certificate-evaluation}
We evaluate \Cref{thm:dp,thm:dp-full} at the public constants of \Cref{ex:logistic}, with the radii $r = r_u$ and $r' = r_u + 2LG/\mu$ of \Cref{thm:dp}. Let $\bar\alpha$, $\Delta_\xi$, and $\tau$ be as in \eqref{eq:domain-fixed-point} and \eqref{eq:dp-full-constants}. At $\eta = 1/L$ and $N \geq 2$, $g_N(L) = 1$, so $\bar g_N = 1$ by \Cref{lem:symbol} and $\tau^{(0)} = \log(1 + \Delta_H/(2n\mu))$. Write
\[
    s_N := \frac{r\Delta_\xi}{\sigma^2} + \frac{(G + \Delta_\xi)^2}{2\sigma^2} + \tau
\]
for the part of $\epsilon_{\mathrm{total}}$ in \eqref{eq:eps-total} that does not depend on $t$. For a candidate iteration count we check $\bar\alpha > 0$ and the radius condition \eqref{eq:domain-fixed-point}. When both hold and $\epsilon \geq s_N$, \Cref{thm:dp-full} at $t = (\epsilon - s_N)\sigma/G$ certifies $(\epsilon, \delta_N(\epsilon))$-differential privacy with
\[
    \delta_N(\epsilon) := e^{-u^2/2} + \frac{2G}{\sigma}\,\mathcal{Q}_1\Bigl( \frac{\sigma}{G}(\epsilon - s_N) \Bigr),
\]
the truncation mass in \eqref{eq:delta-total} being at most $e^{-u^2/2}$ at $r = r_u$ by Gaussian norm concentration (\Cref{cor:ext-gauss-norm}).

For the parameters of \Cref{ex:logistic} at $N = 519$,
\[
\begin{aligned}
    \bar\alpha &\simeq 3.843548916, & s_{519} &\simeq 0.05209595621,\\
    \frac{G(1-\rho^{N})}{\mu\bar\alpha} &\simeq 5.20352425\times10^{-5}, & r'-r &= 1.04\times10^{-4},
\end{aligned}
\]
so the radius condition holds. Evaluating at $\epsilon = 1$ gives $\delta_{519}(1) \leq 9.89505873\times10^{-8}$, rounded upward, against $1.01736108\times10^{-7}$ at $N = 518$; checking every smaller count gives $519$ as the first passing integer for these radii.

For comparison across regularization strengths, we keep $n=10^6$, $d=10^4$, $\sigma/G=5$, $u=6$, $\epsilon=1$, and the target $\delta=10^{-7}$. For each $\mu$, we use $L=\mu+1/4$, $\eta=1/L$, $r=\sigma(\sqrt{d_\star}+u)$, and $r'=r+2LG/\mu$.
\begin{center}
\small
\begin{tabular}{rrrrr}
\toprule
$\mu$ & $\kappa$ & \begin{tabular}{@{}r@{}}Smallest $N$\\satisfying \eqref{eq:N-dagger}\end{tabular}
& \begin{tabular}{@{}r@{}}First passing $N$\\for \Cref{thm:dp-full}\end{tabular}
& \begin{tabular}{@{}r@{}}Certified $\delta(1)$\\at that count\end{tabular} \\
\midrule
$10^{-2}$ & $26$ & $1319$ & $519$ & $9.89506\times10^{-8}$ \\
$10^{-3}$ & $251$ & $16148$ & $6079$ & $9.99539\times10^{-8}$ \\
\bottomrule
\end{tabular}
\end{center}
For the fourth column, we evaluate the more precise certificate with the fixed radii specified above.  By using the Jacobian and residual estimates before the simplifications in \eqref{eq:N-dagger}, we certify the same target privacy pair at a smaller iteration count.

\subsection{An exact Gaussian benchmark}\label{subsec:gaussian-benchmark}
Linear losses with ridge regularization provide an exact reference for evaluating privacy certificates. For the exact minimizer, the reduction to a Gaussian mechanism appears in \citet[Example~3.4]{redberg2023improving}. Here we evaluate the full finite-gradient-descent trajectory. Consider
\[
 F_S(w)=\frac{\mu}{2}\|w\|^2+\langle b_S,w\rangle,
 \qquad
 b_S=-\frac1n\sum_{i=1}^n y_i x_i,
 \qquad y_i\in\{-1,1\}.
\]
Starting at zero, we can solve the recursion explicitly,
\[
 w_N=-\frac{1-\rho^N}{\mu}(b_S+Z).
\]
For every $N\geq1$, the multiplier is public and nonzero. The release is therefore an invertible rescaling of the Gaussian mechanism $b_S+Z$. For a neighboring pair with $\|b_S-b_{S'}\|=G>0$, the privacy profile for the Gaussian mechanism is  
\[
 \delta_{\mathrm G}(\epsilon)=
 \Phi_{\mathrm G}\!\left(\frac{G}{2\sigma}-\frac{\epsilon\sigma}{G}\right)
 -e^\epsilon\Phi_{\mathrm G}\!\left(-\frac{G}{2\sigma}-\frac{\epsilon\sigma}{G}\right),
\]
where $\Phi_{\mathrm G}$ is the standard Gaussian cumulative distribution function \citep[Theorem~2.7 and Corollary~2.13]{dong2022gaussian}. This pairwise privacy profile is the same for every positive iteration count.

For this objective, the Jacobian is identical under neighboring datasets. We can therefore compare the exact profile with the integrated-tail certificate using zero log-determinant contribution. At $\sigma/G=5$ and $\epsilon=1$, the two values are
\[
 \delta_{\mathrm G}(1)\simeq1.75463333\times10^{-8},
 \qquad
 \frac{2G}{\sigma}\mathcal Q_1\!\left(
 \frac{\sigma}{G}\left(1-\frac{G^2}{2\sigma^2}\right)\right)
 \simeq3.63850762\times10^{-8}.
\]
This gives an exact reference for the Gaussian-ratio accounting, separate from the finite-iteration and truncation terms.

\section{Utility}\label{sec:utility}

We bound the excess risk of the mechanism $\mathcal{M}_S(z) = w_N$---both the excess empirical risk and the excess population risk. The analysis is agnostic to the privacy mechanism: it bounds risk as a function of the noise scale $\sigma$ and the iteration count $N$, for the same gradient descent on the perturbed objective, and applies for any $\sigma$.

\subsection{Setting and assumptions}\label{subsec:util-assumptions}

Throughout this section the empirical risk has the sample-average structure
\[
    F(w; S) = \frac{1}{n}\sum_{i=1}^{n} \ell(w; s_i) + r(w),
\]
where $S = (s_1, \ldots, s_n)$ consists of $n$ independent draws from a distribution $\mathcal{D}$, the per-sample loss is $\ell(\cdot\,; s)$, and $r$ is a convex, dataset-independent regularizer. In addition to \Cref{ass:regularity} (which supplies the $\mu$-strong convexity and $L$-smoothness of $F(\cdot\,; S)$), we impose the following.

\begin{assumption}[Bounded per-sample gradient]\label{ass:lipschitz}
$\| \nabla \ell(w; s) \| \leq L_\ell$ for all $w$ and $s$.
\end{assumption}

Under the sample-average structure and \Cref{ass:lipschitz}, the gradient sensitivity is at most $2L_\ell/n$, so \Cref{ass:gradient-sensitivity} holds with $G = 2L_\ell/n$: for neighboring datasets $S, S'$ differing in a single sample $s_j \to s_j'$, the regularizer cancels and the triangle inequality gives
\[
    \| \nabla F(w; S) - \nabla F(w; S') \| = \frac{1}{n}\| \nabla \ell(w; s_j) - \nabla \ell(w; s_j') \| \leq \frac{2L_\ell}{n}.
\]

\begin{assumption}[Population regularity]\label{ass:population}
The population risk $\mathcal{R}_\mu(w) = \mathbb{E}_{s \sim \mathcal{D}}[\ell(w; s)] + r(w)$ is $\mu$-strongly convex and $L$-smooth, with the same constants $\mu, L$ as \Cref{ass:regularity}.
\end{assumption}

Throughout, $w^*$ denotes the population minimizer of $\mathcal{R}_\mu$, and $w_S^*$ denotes the empirical minimizer of $F(\cdot;S)$ (equivalently, $w_S^*=w^*(0)$). We continue to write $w^*(z)$ for the perturbed minimizer of \Cref{lem:perturbed-minimizer}.

\begin{remark}\label{rem:regularity}
Ridge regularization gives the standard instance of \Cref{ass:regularity,ass:population}: take $r(w)=\frac\mu2\|w\|^2$ and assume each per-sample loss is convex and $(L-\mu)$-smooth. Then $F(\cdot;S)$ and $\mathcal R_\mu$ are $\mu$-strongly convex and $L$-smooth. The utility results below use these properties together with the per-sample gradient bound. For the privacy result, the isotropic ridge term also yields the data-span decomposition in \Cref{lem:data-span}. Since $r$ is dataset independent, it does not contribute to gradient sensitivity.
\end{remark}

\subsection{Excess empirical risk}\label{subsec:empirical}

We decompose the excess empirical risk into a \emph{perturbation error}, from adding the linear term $\langle z, w \rangle$ to the objective, and an \emph{optimization error}, from running finitely many gradient steps.

\thmempirical*

\begin{proof}
Write $F_z(w) := F(w; S) + \langle z, w \rangle$. The linear term does not change the Hessian, so $F_z$ is $\mu$-strongly convex and $L$-smooth with minimizer $w^*(z)$ (\Cref{lem:perturbed-minimizer}), and the iterates \eqref{eq:gd-recursion} are exactly gradient descent on $F_z$ from $w_0 = 0$. The excess risk splits at $w^*(z)$ into the two errors of the decomposition:
\[
    F(w_N; S) - F(w_S^*; S) = \underbrace{\bigl[ F(w_N; S) - F(w^*(z); S) \bigr]}_{\text{optimization error}} + \underbrace{\bigl[ F(w^*(z); S) - F(w_S^*; S) \bigr]}_{\text{perturbation error}}.
\]

\emph{Perturbation error.} At $w^*(z)$, \Cref{lem:perturbed-minimizer} gives $\nabla F(w^*(z);S)=-z$. The strong-convexity lower bound of \Cref{lem:ext-quadratic} therefore gives
\[
    F(w_S^*; S) \geq F(w^*(z); S) + \langle -z,\, w_S^* - w^*(z) \rangle + \frac{\mu}{2}\| w_S^* - w^*(z) \|^2.
\]
Write $u := w_S^* - w^*(z)$, rearrange, and complete the square:
\[
    F(w^*(z); S) - F(w_S^*; S) \leq \langle z, u \rangle - \frac{\mu}{2}\| u \|^2 = -\frac{\mu}{2}\Bigl\| u - \frac{z}{\mu} \Bigr\|^2 + \frac{\| z \|^2}{2\mu} \leq \frac{\| z \|^2}{2\mu}.
\]

\emph{Optimization error.} By \Cref{lem:convergence} applied to $F_z$ and item~\ref{itm:init-dist} of \Cref{lem:traj-scale}, $\| w_N - w^*(z) \| \leq \rho^N \| w_0 - w^*(z) \| \leq \rho^N R(z)$. Since $\nabla F_z(w^*(z)) = 0$, the smoothness upper bound of \Cref{lem:ext-quadratic} gives
\[
    F_z(w_N) - F_z(w^*(z)) \leq \frac{L}{2}\| w_N - w^*(z) \|^2 \leq \frac{L}{2}\rho^{2N} R(z)^2.
\]
As $F(w; S) = F_z(w) - \langle z, w \rangle$,
\[
    F(w_N; S) - F(w^*(z); S) = \bigl[ F_z(w_N) - F_z(w^*(z)) \bigr] - \langle z,\, w_N - w^*(z) \rangle,
\]
and Cauchy--Schwarz with the contraction bound gives $|\langle z,\, w_N - w^*(z) \rangle| \leq \| z \|\,\| w_N - w^*(z) \| \leq \rho^N \| z \|\, R(z)$; adding the two displays bounds the optimization error by  $(L/2)\rho^{2N} R(z)^2 + \rho^N \| z \|\, R(z)$; with the perturbation error, for every $z$
\begin{equation}\label{eq:pointwise-excess}
    F(w_N; S) - F(w_S^*; S) \leq \frac{\| z \|^2}{2\mu} + \frac{L}{2}\rho^{2N} R(z)^2 + \rho^N \| z \|\, R(z).
\end{equation}

\emph{Taking expectations.} From $R(z) \leq b_R \| z \| + a_R$ (\Cref{lem:traj-scale}) we get $R(z)^2 \leq b_R^2 \| z \|^2 + 2 a_R b_R \| z \| + a_R^2$ and $\| z \|\, R(z) \leq b_R \| z \|^2 + a_R \| z \|$. Summing the two error terms and taking $\mathbb{E}_z$ with $\mathbb{E}\| z \|^2 = d\sigma^2$ and $\mathbb{E}\| z \| \leq \sqrt{\mathbb{E}\| z \|^2} = \sigma\sqrt{d}$ (Jensen) bounds the optimization error by $\frac{L}{2}\rho^{2N}\bigl( b_R^2 d\sigma^2 + 2 a_R b_R \sigma\sqrt{d} + a_R^2 \bigr) + \rho^N\bigl( b_R d\sigma^2 + a_R \sigma\sqrt{d} \bigr) = \mathcal{E}_{\mathrm{opt}}(N)$ of \eqref{eq:E-opt}, which is the stated bound.
\end{proof}

\begin{remark}\label{rem:empirical-limit}
As $N \to \infty$ both terms of $\mathcal{E}_{\mathrm{opt}}(N)$ vanish and the bound tends to $d\sigma^2/(2\mu)$, the standard objective-perturbation rate. The bound holds at every finite $N$; here $a_R = \| \nabla F(0; S) \|/\mu$ is the dataset-dependent constant of \Cref{lem:traj-scale} for the fixed $S$, and the correction's coefficients depend on $d, \sigma, L, \mu$.
\end{remark}

\subsection{Excess population risk}\label{subsec:population}

We bound the excess population risk by transferring the empirical bound of \Cref{thm:empirical-risk} to the population through uniform convergence; this avoids the condition-number factor that a curvature-based decomposition through $w^*$ would incur (\Cref{rem:no-kappa}). The vehicle is the fast-rate guarantee of \citet[Theorem~1]{uniform-convergence-fastrates}, imported as \Cref{thm:uniform-convergence} in \Cref{sec:external}; applying it at the (random) output $w_N$ and taking expectations over the perturbation converts the empirical bound into a population bound.

\thmpopulation*

\begin{proof}
\Cref{ass:bounded-features,ass:bounded-loss,ass:lipschitz-third} verify the hypotheses of \Cref{thm:uniform-convergence}, by the derivations recorded there. On its event (probability at least $1 - \delta_S$ over $S$), its bound at $a = 1$ holds simultaneously for all $w$, in particular for $w = w_N$ at every realization of $z$. Taking $\mathbb{E}_z$,
\[
    \mathbb{E}_z\bigl[ \mathcal{R}_\mu(w_N) - \mathcal{R}_\mu(w^*) \bigr] \leq 2\, \mathbb{E}_z\bigl[ F(w_N; S) - F(w_S^*; S) \bigr] + \frac{16\, L_\ell^2 \bigl( 32 + \log(1/\delta_S) \bigr)}{\mu n},
\]
and \Cref{thm:empirical-risk} bounds $\mathbb{E}_z[ F(w_N; S) - F(w_S^*; S) ] \leq \frac{d\sigma^2}{2\mu} + \mathcal{E}_{\mathrm{opt}}(N)$, giving the claim. (The right-hand side depends on $S$ through $a_R = \| \nabla F(0; S) \|/\mu$ in $\mathcal{E}_{\mathrm{opt}}(N)$, and the high-probability statement is over the draw of $S$; this dependence is benign, as the sample-average structure and \Cref{ass:lipschitz} give $\| \nabla F(0; S) \| \leq \frac{1}{n}\sum_{i=1}^{n}\| \nabla \ell(0; d_i) \| + \| \nabla r(0) \| \leq L_\ell + \| \nabla r(0) \|$, so $a_R \leq (L_\ell + \| \nabla r(0) \|)/\mu$ uniformly over $S$.)
\end{proof}

\begin{remark}\label{rem:no-kappa}
As $N\to\infty$, $\mathcal E_{\mathrm{opt}}(N)\to0$, and the leading statistical terms in \Cref{thm:population-uniform} scale as $1/\mu$ without an additional multiplicative factor $L/\mu$. A quadratic model illustrates this scaling. Let $H$ be the empirical Hessian, so that $w^*(z)=w_S^*-H^{-1}z$. If $H=\nabla^2\mathcal R_\mu$, then the leading perturbation contribution is
\[
\frac{\sigma^2}{2}\operatorname{tr}\!\left((\nabla^2\mathcal R_\mu)^{-1}\right)
\leq \frac{d\sigma^2}{2\mu}.
\]
By contrast, bounding $H^{-1}$ and $\nabla^2\mathcal R_\mu$ separately by $\mu^{-1}I$ and $LI$ gives the looser scale $L d/\mu^2=\kappa d/\mu$. The uniform-convergence argument avoids this extra curvature ratio in the leading statistical terms.
\end{remark}

\subsection{Proof of the target-budget calibration}\label{subsec:target-budget-proof}
\begin{proof}[Proof of \Cref{cor:target-budget}]
All quantities used to select $\mu$, $\sigma$, and $N$ are public. Fix these quantities and write
\[
 q(\mu)=\log\!\left(1+\frac{L_2\zeta^2}{n\mu}\right).
\]
By the choice of $\mu_{\mathrm{priv}}$ and $\mu\geq\mu_{\mathrm{priv}}$, we have $q(\mu)\leq\epsilon/4$. Also $u^2=2\log(2/\delta)>1$ and $0<\epsilon\leq1$, so
\[
 \frac{G}{\sigma}u=\frac{\epsilon}{4},
 \qquad
 \frac{G^2}{2\sigma^2}
 =\frac{\epsilon^2}{32u^2}\leq\frac{\epsilon}{32}.
\]
The definition of $c_N$ gives
$\epsilon_0=G^2/(2\sigma^2)+q(\mu)+c_N$.
Applying \Cref{thm:dp} at $\widehat\epsilon=\epsilon_0+(G/\sigma)u$ therefore gives
\[
 \widehat\epsilon
 \leq\frac{\epsilon}{32}+\frac{\epsilon}{4}
       +\frac{\epsilon}{4}+\frac{\epsilon}{4}
 =\frac{25\epsilon}{32}<\epsilon.
\]
For the additive privacy parameter, let $\phi$ denote the standard Gaussian density. Since $\mathcal Q_1(u)\leq\phi(u)$ for $u \geq 0$ and $e^{-u^2/2}=\delta/2$,
\[
 \widehat\delta
 =e^{-u^2/2}+\frac{2G}{\sigma}\mathcal Q_1(u)
 \leq\frac{\delta}{2}+\frac{G}{\sigma}\frac{\delta}{\sqrt{2\pi}}
 <\delta,
\]
where the last inequality uses $G/\sigma=\epsilon/(4u)<1/4$. This proves $(\epsilon,\delta)$-privacy. For fixed public parameters the right-hand side of \eqref{eq:N-dagger} is finite, while $c_N\to0$. Hence we can choose a public integer satisfying both iteration conditions.

On the sample event used in \Cref{thm:population-uniform}, the argument for \Cref{thm:unregularized} applies simultaneously to every comparator. Thus, for every $\|\bar w\|\leq W$,
\[
 \mathbb E_z[\mathcal R_0(w_N)]-\mathcal R_0(\bar w)
 \leq\frac{A}{\mu}+\frac{\mu W^2}{2}
       +2\mathcal E_{\mathrm{opt}}(N).
\]
Since $\mu\geq\mu_{\mathrm{util}}$ and $\mu\leq\mu_{\mathrm{util}}+\mu_{\mathrm{priv}}$,
\[
 \frac{A}{\mu}+\frac{\mu W^2}{2}
 \leq\frac{A}{\mu_{\mathrm{util}}}+\frac{\mu_{\mathrm{util}}W^2}{2}
       +\frac{\mu_{\mathrm{priv}}W^2}{2}
 =W\sqrt{2A}+\frac{\mu_{\mathrm{priv}}W^2}{2}.
\]
If $\mu_{\mathrm{util}}\geq\mu_{\mathrm{priv}}$, then $\mu=\mu_{\mathrm{util}}$ and the first two terms are exactly $W\sqrt{2A}$, so no additional regularization term is needed. Finally, substitute $\sigma=4Gu/\epsilon$, $G=2L_\ell/n$, and $\sqrt{a+b}\leq\sqrt a+\sqrt b$ to obtain the decomposition following the corollary.
\end{proof}

\subsection{Proof of the realized-release utility corollary}\label{subsec:high-probability-utility}
\begin{proof}[Proof of \Cref{cor:high-probability-utility}]
By \Cref{cor:ext-gauss-norm}, $\|Z\|\leq r_\beta$ with probability at least $1-\beta$. Since $\|\nabla F(0;S)\|\leq L_\ell$, the trajectory scale satisfies $R(Z)\leq R_\beta$ on this event, uniformly over $S$. By \eqref{eq:pointwise-excess}, the pointwise estimate in the proof of \Cref{thm:empirical-risk}, 
\[
 F(w_N;S)-F(w_S^*;S)
 \leq\frac{\|Z\|^2}{2\mu}
   +\frac L2\rho^{2N}R(Z)^2+\rho^N\|Z\|R(Z)
 \leq U_N(\beta).
\]
Intersect this event with the uniform sample event of \Cref{thm:uniform-convergence} at $a=1$. A union bound gives probability at least $1-\delta_S-\beta$, and the uniform-convergence inequality gives the population bound at the realized $w_N$. Finally use the regularization comparison in the proof of \Cref{thm:unregularized}, which holds for every comparator on the same event.
\end{proof}

\section{External results}\label{sec:external}

We collect here the external results used in the proofs, stated in the notation needed below. When we use a form that differs from the source statement, we include the short derivation.

\begin{lemma}[Minimizer of a strongly convex function]\label{lem:ext-minimizer}
Let $f: \mathbb{R}^d \to \mathbb{R}$ be continuously differentiable and $\mu$-strongly convex with $\mu > 0$. Then $f$ attains its minimum at a unique point $x^*$, and a point minimizes $f$ if and only if its gradient vanishes.
\end{lemma}

Existence and uniqueness are \citet[Theorem~2.2.10]{strongconvex-minimizer}, stated there for the class $\mathcal{S}^{1}_{\mu}(Q)$ with $Q$ closed and convex, applied here at $Q = \mathbb{R}^d$. The characterization is \citet[Theorem~2.2.9]{strongconvex-minimizer}: $x^*$ minimizes $f$ over $Q$ if and only if $\langle \nabla f(x^*),\, x - x^* \rangle \geq 0$ for all $x \in Q$. At $Q = \mathbb{R}^d$, taking $x = x^* - \nabla f(x^*)$ makes the left side $-\| \nabla f(x^*) \|^2$, so the condition forces $\nabla f(x^*) = 0$; conversely, a zero gradient satisfies it trivially. \Cref{lem:perturbed-minimizer} applies this to the perturbed objective $w \mapsto F(w) + \langle z, w \rangle$.

\begin{lemma}[Quadratic bounds from two-sided curvature]\label{lem:ext-quadratic}
Let $f: \mathbb{R}^d \to \mathbb{R}$ be twice continuously differentiable with $\mu I \preceq \nabla^2 f(w) \preceq L I$ for all $w \in \mathbb{R}^d$, where $0 < \mu \leq L$. Then for all $x, y \in \mathbb{R}^d$,
\[
    \frac{\mu}{2}\| y - x \|^2 \;\leq\; f(y) - f(x) - \langle \nabla f(x),\, y - x \rangle \;\leq\; \frac{L}{2}\| y - x \|^2.
\]
\end{lemma}

The lower bound is the defining inequality of the class $\mathcal{S}^1_\mu$ \citep[Definition~2.1.3, eq.~(2.1.20)]{strongconvex-minimizer}, to which $f$ belongs by the second-order characterization \citep[Theorem~2.1.11]{strongconvex-minimizer}. The upper bound is \citet[Lemma~1.2.3]{strongconvex-minimizer}, valid for functions with $L$-Lipschitz gradient, a property that $-L I \preceq \nabla^2 f \preceq L I$ supplies \citep[eq.~(1.2.10)]{strongconvex-minimizer}.

\begin{theorem}[Lipschitz functions of a Gaussian vector]\label{thm:ext-gauss-lipschitz}
Let $X \sim \mathcal{N}(0, I_d)$ and let $f: \mathbb{R}^d \to \mathbb{R}$ be $1$-Lipschitz with respect to the Euclidean norm. Then for every $t \geq 0$,
\[
    \Pr\bigl( f(X) - \mathbb{E}[f(X)] \geq t \bigr) \leq e^{-t^2/2}.
\]
\end{theorem}

This is the $1$-Lipschitz case of \citet[Theorem~2.26]{gaussian-norm-conc}, which asserts that $f(X) - \mathbb{E}[f(X)]$ is sub-Gaussian with parameter at most the Lipschitz constant; in the case at hand, $\mathbb{E}\bigl[ e^{\lambda(f(X) - \mathbb{E}[f(X)])} \bigr] \leq e^{\lambda^2/2}$ for all $\lambda \in \mathbb{R}$. The one-sided tail follows by the Chernoff bound: for $\lambda > 0$, Markov's inequality gives $\Pr\bigl( f(X) - \mathbb{E}[f(X)] \geq t \bigr) \leq e^{-\lambda t}\, \mathbb{E}\bigl[ e^{\lambda(f(X) - \mathbb{E}[f(X)])} \bigr] \leq e^{-\lambda t + \lambda^2/2}$, and the choice $\lambda = t$ yields $e^{-t^2/2}$.

The privacy analysis uses the theorem through the following norm bound; the same computation appears as Example~2.28 of \citet{gaussian-norm-conc}.

\begin{corollary}[Gaussian norm concentration]\label{cor:ext-gauss-norm}
Let $Z \sim \mathcal{N}(0, \sigma^2 I_d)$. For every $t \geq 0$,
\[
    \Pr\bigl( \| Z \| \geq \sigma(\sqrt{d} + t) \bigr) \leq e^{-t^2/2}.
\]
\end{corollary}

\begin{proof}
Write $Z = \sigma X$ with $X \sim \mathcal{N}(0, I_d)$. The Euclidean norm is $1$-Lipschitz by the triangle inequality, and Jensen's inequality gives $\mathbb{E}[\| X \|] \leq \sqrt{\mathbb{E}[\| X \|^2]} = \sqrt{d}$. By \Cref{thm:ext-gauss-lipschitz} with $f = \| \cdot \|$,
\[
    \Pr\bigl( \| Z \| \geq \sigma(\sqrt{d} + t) \bigr) = \Pr\bigl( \| X \| \geq \sqrt{d} + t \bigr) \leq \Pr\bigl( \| X \| - \mathbb{E}[\| X \|] \geq t \bigr) \leq e^{-t^2/2}. \qedhere
\]
\end{proof}

The proof of \Cref{thm:dp} applies this at $t = u$, where it evaluates the first summand of \eqref{eq:delta-profile}.

The truncation argument of \Cref{sec:dp-guarantee} computes output densities through the change of variables; we state the formula in the form used, allowing the conditioning set to be smaller than the domain of the diffeomorphism.

\begin{lemma}[Densities under a diffeomorphism]\label{lem:ext-pushforward}
Let $\mathcal{B} \subseteq \mathbb{R}^d$ be open, let $G : \mathcal{B} \to \mathbb{R}^d$ be a $C^1$-diffeomorphism onto its image, let $\Omega \subseteq \mathcal{B}$ be Lebesgue measurable, and let $Z$ be a random vector on $\mathbb{R}^d$ with density $p_Z$. Then for every Lebesgue measurable $E \subseteq \mathbb{R}^d$,
\[
    \Pr\bigl[ G(Z) \in E,\; Z \in \Omega \bigr] = \int_{E \cap G(\Omega)} \frac{p_Z(G^{-1}(y))}{\bigl| \det \nabla G(G^{-1}(y)) \bigr|}\, dy,
\]
Thus $E\mapsto\Pr[G(Z)\in E,\,Z\in\Omega]$ defines the restricted measure. On $G(\Omega)$ its density is
\[
    y\mapsto \frac{p_Z(G^{-1}(y))}{|\det\nabla G(G^{-1}(y))|}.
\]

\end{lemma}

\begin{proof}
By the change-of-variables theorem \citep[Theorem~2.47]{folland-real}, for a $C^1$-diffeomorphism $G$ of the open set $\mathcal{B}$ and any measurable $f \geq 0$ on $G(\mathcal{B})$,
\[
    \int_{G(\mathcal{B})} f(y)\, dy = \int_{\mathcal{B}} f(G(x))\, |\det \nabla G(x)|\, dx;
\]
part (b) of the same theorem makes $G(\Omega)$ Lebesgue measurable. Apply the formula to $f(y) := \mathbf{1}_{E \cap G(\Omega)}(y)\, p_Z(G^{-1}(y))/|\det \nabla G(G^{-1}(y))|$, measurable since $G^{-1}$ and $\det \nabla G$ are continuous with $\det \nabla G \neq 0$ on $\mathcal{B}$. Since $G$ is a bijection of $\mathcal{B}$ onto $G(\mathcal{B})$, for $x \in \mathcal{B}$ we have $G(x) \in E \cap G(\Omega)$ if and only if $x \in G^{-1}(E) \cap \Omega$, so $f(G(x))\, |\det \nabla G(x)| = \mathbf{1}_{G^{-1}(E) \cap \Omega}(x)\, p_Z(x)$, and the right side of the formula equals $\Pr[ Z \in G^{-1}(E) \cap \Omega ] = \Pr[ G(Z) \in E,\; Z \in \Omega ]$.
\end{proof}

In \Cref{sec:privacy} we take $G=\mathcal{M}_T$ and use the neighborhood $\mathcal{B}$ furnished by \Cref{lem:feasible-diffeo}. The restricted density $p_Y^S(\cdot\,;\Omega)$ is then the density displayed above.

The remaining four results are matrix analysis. Throughout, $\sigma_1(A) \geq \cdots \geq \sigma_d(A)$ denote the singular values of $A \in \mathbb{R}^{d \times d}$ in nonincreasing order, so that $\| A \| = \sigma_1(A)$ and $\| A \|_* = \sum_{i} \sigma_i(A)$. The first result is the \emph{Neumann series}, the matrix form of the geometric series $(1 - x)^{-1} = \sum_{m \geq 0} x^m$.

\begin{lemma}[Neumann series]\label{lem:ext-neumann}
Let $\Psi \in \mathbb{R}^{d \times d}$ satisfy $\| \Psi \| < 1$. Then $I - \Psi$ is nonsingular and
\[
    (I - \Psi)^{-1} = \sum_{m=0}^{\infty} \Psi^m,
\]
the series converging absolutely, as $\sum_{m} \| \Psi^m \| \leq \sum_{m} \| \Psi \|^m = (1 - \| \Psi \|)^{-1}$.
\end{lemma}

This is \citet[Corollary~5.6.16]{horn-johnson-2013}, which states that a matrix $A$ with $\| I - A \| < 1$ in some matrix norm is nonsingular with $A^{-1} = \sum_{m \geq 0} (I - A)^m$; take $A = I - \Psi$ in the spectral norm.

\begin{lemma}[Trace and nuclear-norm inequalities]\label{lem:ext-trace-nuclear}
For all $A, X, Y \in \mathbb{R}^{d \times d}$,
\[
    |\mathrm{tr}\, A| \leq \| A \|_* \qquad\text{and}\qquad \| X Y \|_* \leq \min\bigl( \| X \|_*\, \| Y \|,\; \| X \|\, \| Y \|_* \bigr).
\]
In particular, $\| X Y Z \|_* \leq \| X \|\,\| Y \|_*\,\| Z \|$, and $\| K^k \|_* \leq \| K \|_*\, \| K \|^{k-1}$ for every integer $k \geq 1$.
\end{lemma}

Both displayed inequalities are consequences of the singular-value results of \citet[\S 7.4.1]{horn-johnson-2013}. Von Neumann's trace theorem \citep[Theorem~7.4.1.1]{horn-johnson-2013} gives $\mathrm{Re}\,\mathrm{tr}(AB^*) \leq \sum_i \sigma_i(A)\,\sigma_i(B)$; at $B = I$ every $\sigma_i(B) = 1$ and the trace of the real matrix $A$ is real, so $\mathrm{tr}\,A \leq \| A \|_*$, and applying this to $-A$ as well gives the absolute value. For the product bound, \citet[Corollary~7.4.1.3(d)]{horn-johnson-2013} gives $\sum_i \sigma_i(XB^*) \leq \sum_i \sigma_i(X)\,\sigma_i(B)$; at $B = Y^\top$ (so that $XB^* = XY$ and $\sigma_i(B) = \sigma_i(Y)$),
\[
    \| X Y \|_* \leq \sum_i \sigma_i(X)\,\sigma_i(Y) \leq \min\Bigl( \sigma_1(Y) \sum_i \sigma_i(X),\; \sigma_1(X) \sum_i \sigma_i(Y) \Bigr),
\]
which is the stated minimum. The three-factor form follows by applying the product bound twice, $\| X Y Z \|_* \leq \| X \|\, \| Y Z \|_* \leq \| X \|\, \| Y \|_*\, \| Z \|$, and the power bound by iterating $\| K^k \|_* \leq \| K^{k-1} \|_*\, \| K \|$ down to $\| K \|_*$.

We use three determinant identities in \Cref{sec:log-det}: a nuclear-norm bound for $\log|\det(I+K)|$, the matrix determinant lemma for a rank-one update, and a trace identity for $\log p(A)$ when $p$ is positive on the relevant spectral interval.

\begin{lemma}[Determinant estimates]\label{lem:ext-logdet}
For every $K \in \mathbb{R}^{d \times d}$, with the convention $\log 0=-\infty$,
\begin{equation}\label{eq:det-nuclear}
    \log |\det(I + K)| \leq \| K \|_*.
\end{equation}
For every symmetric positive definite $H \in \mathbb{R}^{d \times d}$, every $z \in \mathbb{R}^d$, and every $c \in \mathbb{R}$,
\begin{equation}\label{eq:det-rank-one}
    \det\bigl( H + c\, z z^\top \bigr) = \det(H)\,\bigl( 1 + c\, z^\top H^{-1} z \bigr).
\end{equation}
\end{lemma}

For \eqref{eq:det-nuclear} we first record that $\sum_i |\lambda_i(A)| \leq \| A \|_*$ for every $A \in \mathbb{C}^{d \times d}$, the eigenvalues counted with algebraic multiplicity. Schur triangularization \citep[Theorem~2.3.1]{horn-johnson-2013} supplies a unitary $U$ with $U^* A U = T$ upper triangular and $t_{ii} = \lambda_i(A)$. Let $D$ be the diagonal unitary with $d_{ii} := \overline{\lambda_i(A)}/|\lambda_i(A)|$ when $\lambda_i(A) \neq 0$ and $d_{ii} := 1$ otherwise, so that $\mathrm{tr}(D U^* A U) = \sum_i d_{ii} t_{ii} = \sum_i |\lambda_i(A)|$, the product $DT$ being upper triangular with those diagonal entries. Since $|\mathrm{tr}\, B| \leq \| B \|_*$ (\Cref{lem:ext-trace-nuclear}) and singular values are unitarily invariant, $\| D U^* A U \|_* = \| A \|_*$, and the record follows. Applying it to $K$ and using that the eigenvalues of $I + K$ are $1 + \lambda_i(K)$,
\[
    |\det(I + K)| = \prod_i \bigl| 1 + \lambda_i(K) \bigr| \leq \prod_i \bigl( 1 + |\lambda_i(K)| \bigr) \leq \exp\Bigl( \sum_i |\lambda_i(K)| \Bigr) \leq e^{\| K \|_*},
\]
by $1 + t \leq e^{t}$; taking logarithms gives \eqref{eq:det-nuclear}. At $K = (u/d) I$ the left side is $d\log(1 + u/d) \to u = \| K \|_*$ as $d$ grows, so the estimate cannot be improved in terms of $\| K \|_*$ alone. For \eqref{eq:det-rank-one}, write $H + c z z^\top = H\bigl( I + u v^\top \bigr)$ with $u := c H^{-1} z$ and $v := z$. The matrix $u v^\top$ has rank at most one, so at most one of its eigenvalues is nonzero, and that one equals its trace $v^\top u$; the eigenvalues of $I + uv^\top$ are therefore $1 + v^\top u$ together with $1$ repeated $d - 1$ times, and $\det(I + uv^\top) = 1 + v^\top u = 1 + c\, z^\top H^{-1} z$. Multiplying by $\det H$ gives \eqref{eq:det-rank-one}.

The next lemma applies the rank-one determinant identity to $\log p(A)$ when $p$ is a positive polynomial. We use it in \Cref{sec:log-det} with the polynomial that defines the frozen-Hessian Jacobian.

\begin{lemma}[Trace of a log-polynomial along a segment]\label{lem:ext-trace-segment}
Let $p$ be a real polynomial that is positive on $[\mu, L] \subset (0,\infty)$ and set $\phi := \log p$. Let $A_0, A_1$ be symmetric with spectra in $[\mu, L]$, and let $A_t := (1-t)A_0 + t A_1$, whose spectrum lies in $[\mu, L]$ for every $t \in [0,1]$. Then
\begin{equation}\label{eq:trace-segment}
    \mathrm{tr}\,\phi(A_1) - \mathrm{tr}\,\phi(A_0) = \int_0^1 \mathrm{tr}\bigl[ \phi'(A_t)\,(A_1 - A_0) \bigr]\, dt,
\end{equation}
with $\phi'(A_t)$ the matrix function of $A_t$ built from $\phi' = p'/p$. Consequently:
\begin{enumerate}
\item\label{itm:trace-lip} $\bigl| \mathrm{tr}\,\phi(A_1) - \mathrm{tr}\,\phi(A_0) \bigr| \leq \Lambda\, \| A_1 - A_0 \|_*$, where $\Lambda := \sup_{[\mu,L]} | \phi' |$.
\item\label{itm:trace-rankone} If $A_1 = A_0 + u u^\top$ and $\phi'(h) = -g(h)/h$ with $0 \leq g \leq \bar g$ on $[\mu, L]$, then
\begin{equation}\label{eq:trace-rank-one}
    -\bar g\,\log\bigl( 1 + u^\top A_0^{-1} u \bigr) \; \leq \; \mathrm{tr}\,\phi(A_1) - \mathrm{tr}\,\phi(A_0) \; \leq \; 0.
\end{equation}
\end{enumerate}
\end{lemma}

\begin{proof}
Each $A_t$ is a convex combination of $A_0$ and $A_1$, so $\mu I \preceq A_t \preceq L I$ and every eigenvalue of $p(A_t)$ is positive; hence $\det p(A_t) > 0$ and $\mathrm{tr}\,\phi(A_t) = \log\det p(A_t)$, a smooth function of $t$. Write $\dot{A} := A_1 - A_0$ and $p(x) = \sum_k c_k x^k$. Jacobi's formula \citep[\S 0.8.10]{horn-johnson-2013} gives $\frac{d}{dt}\log\det p(A_t) = \mathrm{tr}\bigl[ p(A_t)^{-1}\,\frac{d}{dt}p(A_t) \bigr]$, and differentiating each power,
\[
    \frac{d}{dt}p(A_t) = \sum_k c_k \sum_{j=0}^{k-1} A_t^{j}\, \dot{A}\, A_t^{k-1-j}.
\]
The matrix $p(A_t)^{-1}$ is a function of $A_t$ and so commutes with every power of it, so cyclicity of the trace moves the trailing $A_t^{k-1-j}$ to the front and merges it with $A_t^{j}$: each of the $k$ summands contributes $\mathrm{tr}[p(A_t)^{-1}A_t^{k-1}\dot{A}]$, independently of $j$. Summing over $k$ leaves $\mathrm{tr}[p(A_t)^{-1}p'(A_t)\dot{A}] = \mathrm{tr}[\phi'(A_t)\dot{A}]$, and integrating over $t \in [0,1]$ gives \eqref{eq:trace-segment}.

\emph{Item~\ref{itm:trace-lip}.} By $|\mathrm{tr}\, X| \leq \| X \|_*$ and $\| X Y \|_* \leq \| X \|\,\| Y \|_*$ (\Cref{lem:ext-trace-nuclear}), the integrand is at most $\| \phi'(A_t) \|\,\| \dot{A} \|_*$ in absolute value; and $\phi'(A_t)$ is symmetric with eigenvalues $\phi'(\lambda_i(A_t))$, so $\| \phi'(A_t) \| \leq \Lambda$ for every $t$.

\emph{Item~\ref{itm:trace-rankone}.} Here $\dot{A} = u u^\top$, so the integrand is $u^\top \phi'(A_t) u = -u^\top g(A_t) A_t^{-1} u$. The matrices $g(A_t)$ and $A_t^{-1}$ are functions of the same symmetric matrix, hence simultaneously diagonalizable, and $0 \leq g \leq \bar g$ with $A_t^{-1} \succ 0$ gives $0 \preceq g(A_t)A_t^{-1} \preceq \bar g\, A_t^{-1}$. The integrand therefore lies in $[-\bar g\, u^\top A_t^{-1} u,\, 0]$. Applying \eqref{eq:det-rank-one} to $A_t = A_0 + t u u^\top$ gives $\det A_t = \det(A_0)(1 + t s)$ with $s := u^\top A_0^{-1} u \geq 0$, and Jacobi's formula applied to $\log\det A_t$ evaluates the same derivative two ways, $u^\top A_t^{-1} u = \frac{d}{dt}\log\det A_t = s/(1 + ts)$. Integrating over $t \in [0,1]$ gives $\int_0^1 u^\top A_t^{-1}u\, dt = \log(1+s)$, which is \eqref{eq:trace-rank-one}.
\end{proof}

\begin{theorem}[Uniform convergence for regularized GLMs]\label{thm:uniform-convergence}
In addition to \Cref{ass:regularity,ass:population} and the sample-average structure of \Cref{subsec:util-assumptions}, suppose the per-sample loss has the generalized linear form of \eqref{eq:empirical-risk}, $\ell(w; (x, y)) = \ell(\langle x, w \rangle, y)$, with $t\mapsto\ell(t,y)$ convex, $| \ell'(t, y) | \leq L_1$ for all $t, y$ (as in \Cref{ass:bounded-loss}) and $\| x \| \leq \zeta$ for $\mathcal{D}$-almost every sample (the distributional analogue of \Cref{ass:bounded-features}); then \Cref{ass:lipschitz} holds with $L_\ell = L_1 \zeta$. For any $a > 0$ and $\delta_S \in (0,1)$, with probability at least $1 - \delta_S$ over $S \sim \mathcal{D}^n$ the following holds simultaneously for all $w \in \mathbb{R}^d$:
\[
    \mathcal{R}_\mu(w) - \mathcal{R}_\mu(w^*) \leq (1+a)\bigl[ F(w; S) - F(w_S^*; S) \bigr] + \frac{8(1 + 1/a)\, L_\ell^2 \bigl( 32 + \log(1/\delta_S) \bigr)}{\mu n}.
\]
\end{theorem}

This is the fast-rate guarantee of \citet[Theorem~1]{uniform-convergence-fastrates} for regularized generalized linear objectives, instantiated with scalar Lipschitz constant $L_1$, feature bound $\zeta$, and the $\mu$-strong convexity of the population risk (\Cref{ass:population}); the generalized linear structure is what its Rademacher-contraction argument uses, and we invoke the result only in that form. Its two constants enter as the product $L_1\zeta = L_\ell$, and the guarantee holds for all $w \in \mathbb{R}^d$. The result is stated for each fixed $a > 0$, and its high-probability event depends on $a$; the applications fix $a = 1$ in advance.

\end{document}